\documentclass[]{article}
\PassOptionsToPackage{numbers, sort, square}{natbib}
\usepackage[final]{neurips_2025}
\usepackage[utf8]{inputenc}
\usepackage{graphicx}
\usepackage[dvipsnames]{xcolor}
\usepackage{enumitem,kantlipsum}
\usepackage{bm}
\usepackage{soul} 
\usepackage{algorithm}
\usepackage[noend]{algpseudocode}
\usepackage{mathtools}
\usepackage{graphicx} 
\usepackage{fancyvrb} 
\usepackage{amsmath,amssymb}
\usepackage{hyperref}
\usepackage{hyperref}
\usepackage[english]{babel}
\usepackage{amsthm}
\usepackage{bbm}
\usepackage{subcaption}
\usepackage{float}
\usepackage{tabularx,graphicx}
\usepackage{multirow}
\usepackage{xpatch}
\usepackage{booktabs}
\usepackage{etoc}
\makeatletter

\usepackage{subcaption} 

\DeclareMathOperator{\R}{\mathbb{R}}

\newtheorem{theorem}{Theorem}
\newtheorem{corollary}{Corollary}
\newtheorem{proposition}{Proposition}
\newtheorem{lemma}{Lemma}
\theoremstyle{remark}

\theoremstyle{definition}

\usepackage{pgfplots}
\pgfplotsset{compat=newest}

\let\emptyset\varnothing

\title{ElliCE: Efficient and Provably Robust Algorithmic Recourse via the Rashomon Sets
}
\author{Bohdan Turbal \and Iryna Voitsitska \and Lesia Semenova}
\author{%
  Bohdan Turbal\textsuperscript{\rm 1} \quad Iryna Voitsitska\textsuperscript{\rm 2}\quad Lesia Semenova\textsuperscript{\rm 3$ $\thanks{The majority of this work was conducted while Bohdan was at Taras Shevchenko National University of Kyiv and Lesia at Microsoft Research NYC.}} \\
  \\
  \textsuperscript{1} Princeton University \textsuperscript{2} Ukrainian Catholic University
  \textsuperscript{3} Rutgers University\\
  \texttt{bt4811@princeton.edu}, \texttt{voitsitska.pn@ucu.edu.ua}, 
  \texttt{lesia.semenova@rutgers.edu}
}
\date{ }
\begin{document}

\maketitle
\begin{abstract}

Machine learning models now influence decisions that directly affect people’s lives, making it important to understand not only their predictions, but also how individuals could act to obtain better results. Algorithmic recourse provides actionable input modifications to achieve more favorable outcomes, typically relying on counterfactual explanations to suggest such changes.
However, when the Rashomon set -- the set of near-optimal models -- is large, standard counterfactual explanations can become unreliable, as a recourse action valid for one model may fail under another.
We introduce ElliCE, a novel framework for robust algorithmic recourse that optimizes counterfactuals over an ellipsoidal approximation of the Rashomon set.
The resulting explanations are provably valid over this ellipsoid, with theoretical guarantees 
on uniqueness, stability, and alignment with key feature directions.
Empirically, ElliCE generates counterfactuals that are not only more robust but also more flexible, adapting to user-specified feature constraints  while being substantially faster than existing baselines.
This provides a principled and practical solution for reliable recourse under model uncertainty, ensuring stable recommendations for users even as models evolve.
\end{abstract}

\section{Introduction}

When an algorithmic decision denies someone a loan, a job, or insurance coverage, a natural question follows: \textit{What could I change to obtain a better outcome next time?}
Algorithmic recourse answers this question by providing concrete, actionable changes that could lead to a more favorable decision.
A common way to generate such recommendations is through counterfactual explanations---small modifications to an individual’s features that flip the model’s prediction.
Yet, even when the recommendation looks specific (e.g. “increase your income by \$5000”), one must ask: \textit{``Would that same change still work tomorrow if the institution retrains or replaces its model?''} or \textit{``How stable are these suggestions across equally good models that explain the data in different ways?''}

Most existing counterfactual generation methods \cite{ PoyiadziSokolSantosRodriguezDeBieFlach2019, SunEtAl24, LaugelLesotMarsalaRenardDetyniecki2017, McGrathCostabelloLeVanSweeneyKamiabShenLecue2018, Russell2019efficient, NemirovskyThiebautXuGupta2022a, mohammadi2021scaling, WachterMittelstadtRussell2017} implicitly assume that the underlying model is fixed and perfectly known. In practice, models evolve: banks regularly retrain risk predictors, healthcare institutions update diagnostic classifiers, and regulators may require model re-validation under new privacy or transparency constraints. Small shifts in data or regularization can result in very different-yet-equally-accurate models. This phenomenon, known as the Rashomon Effect \cite{breiman2001statistical, semenova2022existence, rudin2024amazing, d2022underspecification, ganesh2025systemizing}, implies that many distinct predictors achieve nearly optimal performance. In such settings, a counterfactual valid for one model can fail under another, undermining the reliability and consistency of algorithmic recourse.

Recent approaches have attempted to produce robust counterfactuals, meaning counterfactuals that are  valid under small parameter perturbations or across predefined ensembles \cite{upadhyay2021towards, DuttaLongMishraTilliMagazzeni2022, jiang2024argumentative, pmlr-v222-jiang24a, JiangLeofanteRagoToni2024, hamman2023robust, ForelParmentierVidal2022, Kinjo2025robustcounterfactual}. However, these methods either rely on heavy-weight mixed-integer solvers, restrict robustness to local neighborhoods around a single model, or lack formal guarantees of validity across the full space of near-optimal solutions known as the Rashomon set. None of them directly leverages the geometry of this Rashomon set itself.

We introduce ElliCE, an efficient and provably robust framework for algorithmic recourse that works over an ellipsoidal approximation of the Rashomon set.
By modeling the space of near-optimal models as an ellipsoid derived from the curvature (Hessian) of the loss landscape, 
ElliCE reformulates robust counterfactual generation as a tractable convex optimization problem.
The resulting counterfactuals are valid for every model inside the ellipsoid, ensuring that a user’s recommended action remains meaningful even if the deployed model is replaced by another equally accurate one from the approximated Rashomon set.

Our contributions are fourfold:
(1) \textit{Theoretical foundation}. 
We derive a closed-form expression for the worst-case prediction, which allows us to formulate the robust recourse problem as a convex optimization and establish formal guarantees of validity, uniqueness, and stability for ElliCE’s counterfactuals.
(2) \textit{Geometric intuition}. We show that ElliCE’s robustness term connects the counterfactual’s stability with the importance of the features it modifies as the optimization naturally aligns recourse directions with the principal curvature axes of the loss landscape.
(3) \textit{Actionability}. ElliCE supports feature-level constraints, such as sparsity constraints, immutable or range-restricted attributes, allowing users to generate realistic, actionable recourse tailored to specific application or user settings.
(4) \textit{Empirical validation}. Across multiple high-stakes tabular datasets, ElliCE achieves higher robustness and remains one to three orders of magnitude faster than competing baselines, while maintaining proximity and plausibility.

Ultimately, ElliCE looks at algorithmic recourse through the lens of model multiplicity. Instead of relying on a single model’s decision boundary, it offers explanations that stay consistent across many models that fit the data almost equally well. This perspective treats the Rashomon Effect not as a flaw to eliminate, but as an inherent source of uncertainty to account for, leading to stable recourse in the presence of model diversity.

\section{Related works}

\textbf{Rashomon Effect.} 
The Rashomon Effect, a term popularized by \citet{breiman2001statistical} in the context of machine learning, describes the phenomenon where multiple distinct models can achieve near-optimal empirical risk (these models form a Rashomon set). This effect is also referred to as model multiplicity \cite{marx2020predictive, black2022model}. The existence of the Rashomon set has implications for the trustworthiness and reliability of machine learning systems, influencing feature importance  \cite{fisher2019all, muller2023empirical, donnelly2023the, dong2020exploring}, fairness \cite{DaiRavishankarYuanNeillBlack2025, langlade2025fairness, meyer2024perceptions}, the existence of simple yet accurate models \cite{semenova2022existence, semenova2023noise, boner2024using} to name a few.
Significant research has focused on measuring and characterizing the Rashomon set for different model classes \cite{xin2022rashomontrees, hsu2022rashomon, hsu2024dropout, hsu2024rashomongb, zhong2023exploring}. 
Our work leverages insights into the geometry of the Rashomon set, explored by works like \citet{zhong2023exploring, donnelly2025rashomon}, but applies them to the distinct challenge of generating robust algorithmic recourse across this set.

\textbf{Counterfactual Explanations (CEs).} 
Counterfactual Explanations have emerged as a prominent tool for providing  algorithmic recourse. Numerous approaches exist for generating CEs. Proximity-based methods aim for counterfactuals requiring minimal feature space perturbations \citep{WachterMittelstadtRussell2017, UstunSpangherLiu2018, mohammadi2021scaling, brughmans2021nice}. Sparsity techniques prioritize modifying the fewest features possible to enhance actionability \citep{McGrathCostabelloLeVanSweeneyKamiabShenLecue2018, SunEtAl24}, while some methods attempt to balance both objectives \citep{LaugelLesotMarsalaRenardDetyniecki2017}. Another research direction emphasizes plausibility, ensuring generated CEs represent realistic instances by constraining them to the data manifold, for example, using guidance from generative models \citep{PawelczykBroelemannKasneci2020, JoshiKoyejoVijitbenjaronkKimGhosh2019, NemirovskyThiebautXuGupta2022a}, encoding feasibility rules \citep{KarimiBartheBalleValera2020a}, or tracing density-aware paths \citep{PoyiadziSokolSantosRodriguezDeBieFlach2019}. Recent extensions also incorporate temporal reasoning \citep{ceccon2025reinforcement} and fairness objectives \citep{barrainkua2025pays, Yetukuri_Hardy_Vorobeychik_Ustun_Liu_2024, kuratomi2025subgroup}.
A key limitation across these approaches (which ElliCE directly addresses) is the assumption of a fixed, perfectly known predictive model, as counterfactuals constructed near a specific decision boundary can become unstable under model updates or perturbations.

\textbf{Robustness to Local Model Perturbations.}
Building upon the limitation of fixed models, one line of work has focused specifically on achieving robustness against small, local changes or perturbations in the model's parameters.
For instance, ROAR \cite{upadhyay2021towards} optimizes CEs considering local $\Delta$-perturbations of the model. \citet{JiangLeofanteRagoToni2022} introduced $\Delta$-robustness, a formal metric to assess CE validity under bounded parameter perturbations in neural networks, with subsequent works developing provably robust MILP-based methods \cite{pmlr-v222-jiang24a}.  
While these methods offer formal guarantees for $\Delta$-robustness, MILP-based approaches can face scalability challenges, and the focus is generally on local parameter stability rather than the broader implications of the Rashomon Effect. 

\textbf{Robustness under the Rashomon Effect.} 
A growing body of work addresses counterfactual robustness under model multiplicity, aligning closely with the Rashomon Effect. Several approaches evaluate stability across predefined sets or ensembles of models, introducing heuristic stability measures (e.g., T:Rex \citep{hamman2023robust} and RobX \citep{DuttaLongMishraTilliMagazzeni2022}), probabilistic frameworks \citep{ForelParmentierVidal2022, Kinjo2025robustcounterfactual}, or guarantees under specific norms and conditions like distribution shift \citep{Kyaw2025optimal, konig2025performative, garg2025search}. Foundational work by \citet{PawelczykBroelemannKasneci2020} conceptually linked the Rashomon Effect to counterfactuals, though primarily enhancing input perturbation robustness. More recent methods use argumentative ensembling \citep{jiang2024argumentative} or aggregate explanations across AutoML-generated sets \citep{cavus2025beyond} to handle model multiplicity.

Our work takes a distinct approach. Rather than relying on ensemble agreement, heuristic stability metrics, local perturbations, or argumentative aggregation, ElliCE leverages the local geometry of the Rashomon set, approximated by an ellipsoid, to derive theoretically grounded, robust recourse valid across all models within the approximation.

\section{Background and Notation}\label{section:notation}
\textbf{Dataset and hypothesis space}. Consider $n$ i.i.d. samples $\mathcal{S}_n=\{\mathbf{z}_i = (\mathbf{x}_i,y_i)\}_{i=1}^n$, where $\mathbf{x}_i\in\mathcal{X}\subseteq\mathbb{R}^d$ and $y_i\in\mathcal{Y} = \{0, 1\}$ are generated from an unknown distribution $\mathcal{D}$ on $\mathcal{X}\times \mathcal{Y}$. Let $\mathcal{Y}_{pred}$ be an output space, where $\mathcal{Y}_{pred} \subseteq \mathbb{R}$ for scores (logits) or $\mathcal{Y}_{pred}  \subseteq [0,1]$ for probabilities.
Then $\mathcal{F}=\{f_{\bm{\theta}} : \bm{\theta} \in \Theta\}$ is a hypothesis space of functions $f_{\bm{\theta}}: \mathcal{X}\rightarrow \mathcal{Y}_{pred}$, parameterized by a vector $\bm{\theta} \in \Theta  \subseteq \mathbb{R}^p$. For example, $\mathcal{F}$ can represent linear models or multilayer perceptrons. We denote a specific function by $f_{\bm{\theta}}$. As our analysis focuses on the parameter space $\Theta$, we will often refer to the model directly by its parameter vector $\bm{\theta}$.

\textbf{Loss and objective function}. Let $\phi:\mathcal{Y}_{pred}\times \mathcal{Y}\rightarrow \mathbb{R}_+$ be a loss function. In this work, we consider binary cross-entropy (logistic) loss $\phi(f_{\bm{\theta}}(\mathbf{x}), y) = -[y \log(\sigma(f_{\bm{\theta}}(\mathbf{x}))) + (1-y) \log(1-\sigma(f_{\bm{\theta}}(\mathbf{x})))]$, which is applied to the model's raw score (logit), $s=f_{\bm{\theta}}(\mathbf{x})$, where $\sigma(s) = \frac{1}{1 + \exp(-s)}$ is the sigmoid function. However, our results generalize to other convex losses. 
The true risk is the expected loss $J(\bm{\theta}) = \mathbb{E}_\mathbf{z} [\phi(f_{\bm{\theta}}(\mathbf{x}), y)]$ that we approximate with the empirical risk, which is the average loss, $\hat{J}(\bm{\theta}) = \frac{1}{n} \sum_{i=1}^n \phi(f_{\bm{\theta}}(\mathbf{x}_i),y_i)$.
  We also define an $\ell_2$-regularized objective function:
$\hat{L}(\bm{\theta}) = \hat{J}(\bm{\theta}) + \frac{\lambda}{2} ||\bm{\theta}||_2^2,$
where $\lambda \ge 0$ is the regularization strength. 
The empirical risk minimizer (ERM) is $\hat{\bm{\theta}} \in \arg\min_{\bm{\theta}\in \Theta} \hat{L}(\bm{\theta})$.
When $\lambda = 0$, the ERM is $\hat{\bm{\theta}}\in \arg\min_{\bm{\theta}\in \Theta}\hat{J}(\bm{\theta})$.

\textbf{Rashomon set}. Following \cite{fisher2019all, semenova2022existence, xin2022rashomontrees}, we define the $\epsilon$-Rashomon set within the parameter space $\Theta$ as the set of parameter vectors whose corresponding models $f_{\bm{\theta}}$ have objective value close to the minimum:
\[
\mathcal{R}(\epsilon) \coloneqq \{\bm{\theta} \in \Theta: \hat L(\bm{\theta}) \leq\hat L(\hat{\bm{\theta}}) + \epsilon\},
\]
where $\epsilon \geq 0$ is the Rashomon parameter defining the allowable tolerance in objective compared to the ERM, $\hat{L}(\hat{\bm{\theta}})$. It is typically a small value. The existence of the Rashomon set with multiple, distinct parameter vectors $\bm{\theta}$ (corresponding to potentially distinct functions $f_{\bm{\theta}}$) achieving similar performance implies that different underlying logic (how features contribute to predictions) can explain the data equally well. It is important to be aware of this variability among near-optimal models when generating explanations for individual predictions, as different models in $\mathcal{R}(\epsilon)$ might suggest different ways an outcome could be changed.

\textbf{Counterfactual explanations.} 
Let $g: \mathcal{Y}_{pred} \to \{0,1\}$ be the decision function that converts a model's score output $s = f_{\bm{\theta}}(\mathbf{x})$ to a final binary class label by applying a threshold $t$, such that $g(s) = \mathbf{1}[s \ge t]$.
For an ERM $\hat{\bm{\theta}}$ and for an input vector $\mathbf{x}_0$ with prediction $g(f_{\hat{\bm{\theta}}}(\mathbf{x}_0)) = \hat{y}_0$, a counterfactual explanation $\mathbf{x}_c$ is a data point such that its predicted class is the opposite, i.e., $g(f_{\hat{\bm{\theta}}}(\mathbf{x}_c)) = 1 - \hat{y}_0$.
The set of all counterfactual explanations for $\mathbf{x}_0$ under the model $\hat{\bm{\theta}}$ and decision function $g$ is defined as:
\[
\mathcal{C}(\mathbf{x}_0, \hat{\bm{\theta}}) = \left\{\mathbf{x}_c \in \mathcal{X} : g(f_{\hat{\bm{\theta}}}(\mathbf{x}_c)) = 1 - g(f_{\hat{\bm{\theta}}}(\mathbf{x}_0))\right\}.
\]
For instance, in a credit loan application scenario, if an applicant $\mathbf{x}_0$ is denied a loan (e.g., $g(f_{\hat{\bm{\theta}}}(\mathbf{x}_0)) = 0$), a counterfactual explanation $\mathbf{x}_c$ would be a modified version of their application details (e.g., increased income, reduced debt) such that the model predicts approval, $g(f_{\hat{\bm{\theta}}}(\mathbf{x}_c)) = 1$.
While many such $\mathbf{x}_c$ might exist, practical algorithmic recourse aims to find explanations that require minimal change for the user. This means finding the ``closest'' counterfactual:
$\mathbf{x}_c^* = \arg\min_{\mathbf{x}_c \in \mathcal{C}(\mathbf{x}_0, \hat{\bm{\theta}})} \nu(\mathbf{x}_c, \mathbf{x}_0),
$
where $\nu(\cdot, \cdot):\mathbb{R}^d\times \mathbb{R}^d \rightarrow \mathbb{R}$ is a defined distance function or cost metric that  we discuss next.


\textbf{Distance Metrics.} 
In our framework, we primarily focus on the two distance metrics for generating actionable and interpretable counterfactuals: $\ell_2$ or Euclidean distance and mixed distance $\ell_{mix}$. Note that $\ell_2$ is a natural geometric measure of proximity, that penalizes large differences in any feature, \(\nu(\mathbf{x}_c, \mathbf{x}_0) = \ell_2(\mathbf{x}_c, \mathbf{x}_0) = \|\mathbf{x}_c - \mathbf{x}_0\|_2^2 = \sum_{j=1}^d (x_{cj} - x_{0j})^2.\)
For practical applications where features have different natures (continuous and categorical),  one can also consider the mixed distance $\nu(\cdot, \cdot) = \ell_{mix}$, inspired by Gower's distance. Assuming that the data are standardized, it is defined as:
\(\ell_{mix}(\mathbf{x}_c, \mathbf{x}_0) = \sqrt{\sum_{j \in \mathcal{I}_{cont}} (x_{cj} - x_{0j})^2 + \sum_{j \in \mathcal{I}_{cat}} \bar{u}_j \mathbf{1}[x_{cj} \neq x_{0j}]},\)
where $\mathcal{I}_{cont}$ and $\mathcal{I}_{cat}$ denote the sets of continuous and categorical feature indices respectively, $\mathbf{1}[\cdot]$ is the indicator function, and $\bar{u}_j$ are optional weights reflecting the cost of changing feature $j$. We use $\ell_2$ distance for our theoretical analysis and experiments in the subsequent sections.

Next, we describe our approximating framework and outline the optimization process.

\section{A Framework for Robust Recourse over the Rashomon Set}\label{section:framework}

We focus our theoretical analysis on linear predictors of the form $f_{\bm{\theta}}(\mathbf{x})=\bm{\theta}^{\top}\mathbf{x}$. However, 
the same methodology applies in the final embedding space of multilayer perceptrons (MLPs) by writing the model as
$f_{\boldsymbol{\theta}}(\mathbf{x}) = \boldsymbol{\theta}^{\top} h(\mathbf{x})$,
where $h(\mathbf{x})$ is the penultimate-layer embedding and $\boldsymbol{\theta}$ are the last-layer parameters. We freeze $h(\cdot)$ and apply the same ellipsoidal procedure to $\boldsymbol{\theta}$ as in the linear case (equivalently, replace $\mathbf{x}$ by $h(\mathbf{x})$ in the formulas below).



\textbf{Approximated Rashomon set.} 
For certain objectives, such as $\ell_2$-regularized mean-squared error on linear models, the Rashomon set is exactly an ellipsoid in the parameter space \citep{semenova2022existence}: $\mathcal{R}(\epsilon) = \{\bm{\theta}: (\bm{\theta} - \hat{\bm{\theta}})^\top (X^\top X + \lambda I_p)(\bm{\theta} - \hat{\bm{\theta}}) \leq \epsilon\}$, where $X \in \mathbb{R}^{n\times d}$ is the data matrix, whose $i$-th row is the feature vector $\mathbf{x}_i^\top$, $I_p$ is an identity matrix of size $p\times p$, and $\lambda \in \mathbb{R}_+$ is the regularization parameter. 
Because mean-squared error provides a local quadratic approximation to other convex losses, the exact ellipsoidal form of its Rashomon set serves as strong motivation for the Rashomon set approximation. Building on this and on similar geometric intuition \citep{zhong2023exploring},
we approximate the $\epsilon$-Rashomon set with an ellipsoid defined by the local geometry of the loss landscape:
\[
\hat{\mathcal{R}}(\epsilon) =\{\bm{\theta}: \frac{1}{2}(\bm{\theta} - \hat{\bm{\theta}})^\top H(\bm{\theta} - \hat{\bm{\theta}}) \leq \epsilon\},
\]
where $H = X^\top W X + \lambda I_p$ is the Hessian of the $\ell_2$-regularized loss function, evaluated at $\hat{\bm{\theta}}$. For logistic loss, $W$ is an $n \times n$ diagonal matrix of weights where $w_{ii} = \sigma(f_{\hat{\bm{\theta}}}(\mathbf{x}_i))(1 - \sigma(f_{\hat{\bm{\theta}}}(\mathbf{x}_i)))$. Recall from Section~\ref{section:notation} that $\sigma(\cdot)$ is the sigmoid function.

The Hessian matrix $H$ of the regularized objective $\hat{L}(\bm{\theta})$ is strictly positive definite. This is because it is the sum of the positive semidefinite (PSD) Hessian from the convex logistic loss and the positive definite (PD) Hessian from the $\ell_2$ regularization term ($\lambda I_p$), assuming $\lambda > 0$. A positive definite Hessian is important for our framework, as it guarantees the approximating ellipsoid is well-defined and bounded, and ensures that $H$ is invertible for our closed-form solution.

In cases where the unregularized risk $\hat{J}(\bm{\theta})$ is minimized (e.g., for neural networks), the resulting Hessian is only guaranteed to be PSD and may be singular. For these models, we ensure positive definiteness in practice by adding a small stabilization term, $\alpha I_p$, $\alpha > 0$, to the computed Hessian, which is a standard technique to guarantee invertibility.

 \textbf{Optimization.} To find a robust counterfactual explanation, we want to compute an explanation $\mathbf{x}_c$ that is closest to the original point $\mathbf{x}_0$ while ensuring that its predicted outcome is above a target threshold $t$ for all models within the approximated Rashomon set.
In other words, for a given $\mathbf{x}_0$, we look for a minimally modified (measured in some distance; we will use $\ell_2$ here) $\mathbf{x}_c$, such that its predicted outcome achieves $t$ even when evaluated by the least favorable model $\bm{\theta}$ within the approximated Rashomon set $\hat{\mathcal{R}}(\epsilon)$.
Formally, this requirement leads to the following optimization problem:
\begin{equation}\label{eq:original_optimization}
    \min_{\mathbf{x}_c}  \quad  \|\mathbf{x}_c - \mathbf{x}_0\|_2^2 \quad \quad \text{s.t.} \quad \min_{\bm{\theta} \in \hat{\mathcal{R}}(\epsilon)} \bm{\theta}^\top \mathbf{x}_c \ge t. 
\end{equation}
The inner minimization problem admits a closed-form solution, as we show next in Theorem~\ref{th:closed_form_solution}. By reformulating the problem in this way, we get a tractable optimization framework that supports more efficient computation and analytical analysis of solution properties.

\newcommand{\TheoremClosedForm}
{
For positive-definite Hessian $H$, the inner minimization problem over the ellipsoid-approximated Rashomon set $\hat{\mathcal{R}}(\epsilon)$
has the closed-form solution
$\min_{\bm{\theta} \in \hat{\mathcal{R}}(\epsilon)} \bm{\theta}^{\top} \mathbf{x}_c =
        \hat{\bm{\theta}}^{\top} \mathbf{x}_c -
        \sqrt{2\epsilon\,\mathbf{x}_c^{\top} H^{-1} \mathbf{x}_c}.$
Moreover, for a given $\mathbf{x}_c$, the worst-case model $\bm{\theta}_{worst}(\mathbf{x}_c)$ that achieves this minimum is: 
$\bm{\theta}_{worst}(\mathbf{x}_c) =
        \hat{\bm{\theta}} -
        \sqrt{2\epsilon}\,
        \frac{H^{-1}\mathbf{x}_c}{\sqrt{\mathbf{x}_c^{\top}H^{-1}\mathbf{x}_c}}.$
}

\begin{theorem}[Closed-form solution]\label{th:closed_form_solution}
\TheoremClosedForm
\end{theorem}

We prove Theorem~\ref{th:closed_form_solution} in Appendix~\ref{app:proof1}. 
As a direct consequence of Theorem~\ref{th:closed_form_solution}, we obtain a practical criterion for verifying the robustness of a potential counterfactual. Specifically, since the theorem provides an explicit characterization of the output generated by the least favorable model $\bm{\theta} \in \hat{\mathcal{R}}(\epsilon)$ for a given $\mathbf{x}_c$, we can immediately determine if this $\mathbf{x}_c$ achieves the target $t$ across the entire set as we show in the following corollary.

\begin{corollary}\label{corollary:check}
A given counterfactual explanation $\mathbf{x}_c$ is robust with respect to all models in the ellipsoid-approximated Rashomon set $\hat{\mathcal{R}}(\epsilon)$ against a target score $t$ if and only if: 
$\hat{\bm{\theta}}^{\top} \mathbf{x}_c - \sqrt{2\epsilon\,\mathbf{x}_c^{\top} H^{-1} \mathbf{x}_c} \ge t.$
\end{corollary}

By substituting the closed-form solution from Theorem~\ref{th:closed_form_solution} into the original optimization problem \eqref{eq:original_optimization}, the robust counterfactual optimization problem becomes:
\begin{equation} \label{eq:final_optimization}
    \min_{\mathbf{x}_c}  \quad \|\mathbf{x}_c - \mathbf{x}_0\|_2^2 \quad \quad
    \text{s.t.}  \quad \hat{\bm{\theta}}^{\top} \mathbf{x}_c - \sqrt{2\epsilon\, \mathbf{x}_c^\top H^{-1} \mathbf{x}_c} \ge t. 
\end{equation}



The resulting problem is a quadratically constrained quadratic program (QCQP), which is a class of tractable convex optimization problems. We solve it efficiently using a gradient-based method. 
Leveraging the formulation \eqref{eq:final_optimization}, we implement two approaches for generating counterfactuals: a search-based method for generating data-supported counterfactuals lying on the data manifold, and a continuous optimization method for exploring potentially novel non-data supported solutions.

\textbf{Continuous CE generation.} 
For non-data supported counterfactuals, we solve the convex optimization problem in Equation~\eqref{eq:final_optimization} using a gradient-based approach for both linear models and multilayer perceptrons. This method directly optimizes for a counterfactual $\mathbf{x}_c$ in the input space. For neural networks, the process is guided by the worst-case model $\bm{\theta}_{worst}(\mathbf{x}_c)$  identified in the final layer's embedding space using  Theorem~\ref{th:closed_form_solution}, with the resulting gradients mapped back to the input features. The full details of this procedure are available in Appendix~\ref{appendix:optimization}.

\textbf{Data-supported CE generation}. 
For practical applications where counterfactuals should remain on the data manifold, we generate data-supported explanations based on the training set. Specifically, we evaluate the robust logit $\hat{\bm{\theta}}^{\top} \mathbf{x}_i - \sqrt{2\epsilon \mathbf{x}_i^\top H^{-1} \mathbf{x}_i}$ for each training data point $\mathbf{x}_i$ using Theorem~\ref{th:closed_form_solution}. Then, we compute the subset $S_{stable}$ by filtering out points where this robust prediction exceeds the target threshold $t$. Finally, we use k-d tree nearest neighbor search within $S_{stable}$ to identify the points closest to the input point $\mathbf{x}_0$ in terms of defined distance (for example, $\ell_2$), which gives us a counterfactual that is both robust and lies on the data manifold.


The continuous approach offers flexibility by exploring the entire feature space for new solutions, while the data-supported approach guarantees plausibility by restricting solutions to observed examples. We evaluate the performance of both approaches in Section~\ref{section:experiments} and focus on theoretical guarantees of our framework next.

\section{Theoretical Guarantees of ElliCE Counterfactuals}\label{section:properties}


In this section, we explore key theoretical properties of the counterfactual explanations generated under our framework. Note that we use $\ell_2$ distance as target distance between $\mathbf{x}_0$ and $\mathbf{x}_c$.
We show that the counterfactual explanations generated by our method are valid, unique, stable, and align with important directions in the feature space. We focus on each of these properties separately and proofs of theorems provided in this section are in Appendix \ref{app:proof2}.

\textbf{Validity.} By explicitly optimizing for the worst-case model  $\bm{\theta}_{worst}$ within the defined ellipsoid, any counterfactual $\mathbf{x}_c$ generated by ElliCE is, by construction, valid for all models in the approximated Rashomon set. This inherent validity ensures that the provided recourse is faithful, regardless of which model from the approximated Rashomon set was selected.

\textbf{Uniqueness.} By Theorem~\ref{thm:uniqueness}, that we state next, any solution $\mathbf{x}_c$ to the optimization problem \eqref{eq:final_optimization} is unique. Because our objective is strictly convex and the approximated Rashomon set is characterized as an ellipsoid, for a given $\mathbf{x}_0$, there can never be two distinct counterfactuals at the same $\ell_2$ distance from the original $\mathbf{x}_0$. In practical terms, this uniqueness guarantees that ElliCE provides a single solution for a given input and desired robustness level. This directly addresses and resolves ``explanation multiplicity'' \cite{Gunasekaran2024MAMCR}, where multiple, distinct explanation paths might exist for a single input (at least for $\ell_2$ distance).

\newcommand{\TheoremUniqueness
}{
If a solution $\mathbf{x}_c$ to the optimization problem \eqref{eq:final_optimization} exists, then $\mathbf{x}_c$ is unique.
}

\begin{theorem}[Uniqueness]
    \label{thm:uniqueness}
\TheoremUniqueness
\end{theorem}

\textbf{Stability.} The input data $\mathbf{x}_0$ is often subject to noise or minor variations. A desirable property is that such small changes in the input do not lead to drastically different counterfactuals. Our framework ensures this stability. Theorem~\ref{thm:stability} formally states that the process of generating robust counterfactuals is Lipschitz continuous with a constant of $1$. This means that if the original input $\mathbf{x}_0$ is perturbed by a small amount $\bm{\delta}$ to become $\mathbf{x}_0'$, the resulting robust counterfactual $\mathbf{x}_c'$ will not deviate from the original counterfactual $\mathbf{x}_c$ by more than the magnitude of the initial perturbation $\|\bm{\delta}\|_2$. This property guarantees the reliability of the explanations.

\newcommand{\TheoremStability
}{
Given an input $\mathbf{x}_0$, let $\mathbf{x}_c$ be the robust counterfactual solution for $\mathbf{x}_0$. If the input is perturbed to $\mathbf{x}_0' = \mathbf{x}_0 + \bm{\delta}$, where $\bm{\delta} \in \mathbb{R}^d$, and $\mathbf{x}_c'$ is the robust counterfactual solution for $\mathbf{x}_0'$, then $\|\mathbf{x}_c - \mathbf{x}_c'\|_2 \leq \|\bm{\delta}\|_2$.

}

\begin{theorem}[Stability]
    \label{thm:stability}
\TheoremStability
\end{theorem}

\textbf{Alignment with Important Feature Directions.} An insightful explanation should not only provide a path to a different outcome but also highlight which features are most critical in achieving that change, particularly under model uncertainty. The robustness penalty term, $C_{rob}(\epsilon, \mathbf{x}_c) = \sqrt{2\epsilon \mathbf{x}_c^\top H^{-1} \mathbf{x}_c}$, plays a key role in this alignment. Theorem~\ref{thm:VI} formalizes the intuition that to minimize this penalty (and thus find an efficient robust counterfactual), the recourse direction $\mathbf{x}_c$ should align with directions in the feature space that are most sensitive or influential, as captured by the eigenvectors of the Hessian matrix $H$. Specifically, under reasonable conditions, the penalty is minimized when the counterfactual aligns with the leading eigenvector of $H$, which often corresponds to the direction of greatest sensitivity. This encourages the counterfactual to suggest changes along features that have a significant impact, making the explanation more informative.

\newcommand{\TheoremAlignment
}{
Define the robustness penalty as $C_{rob}(\epsilon, \mathbf{x}_c) = \sqrt{2\epsilon \mathbf{x}_c^\top H^{-1} \mathbf{x}_c}$ for a symmetric positive definite Hessian $H$. Let $\lambda_1$ be the largest eigenvalue of $H$ with corresponding eigenvector $\mathbf{q}_1$, and assume that $\lambda_1$ is unique.  Then, for a fixed non-zero norm $\|\mathbf{x}_c\|_2$, the robustness penalty term $C_{rob}(\epsilon, \mathbf{x}_c)$ is minimized when the counterfactual vector $\mathbf{x}_c$ is aligned (i.e., collinear) with the eigenvector $\mathbf{q}_1$.
}

\begin{theorem}[Alignment with Important Feature Directions]
    \label{thm:VI}
\TheoremAlignment
\end{theorem}

\textbf{Price of robustness.} Previous literature has observed the trade-off between robustness and proximity \cite{ForelParmentierVidal2022}. Indeed, intuitively, increasing robustness and ensuring validity across a larger set of potential models may require more changes to the input features, effectively increasing the proximity. 
This implies a ``cost'' for greater robustness that Theorem~\ref{thm:tradeoff} formalizes.

\newcommand{\TheoremTradeoff
}{
For an input $\mathbf{x}_0$ such that $\hat{\bm{\theta}}^\top \mathbf{x}_0 \leq t$, where $\hat{\bm{\theta}}$ is ERM, let $\mathbf{x}_c^*(\epsilon)$ be the optimal robust counterfactual for a given robustness level $\epsilon > 0$, and let $\nu(\epsilon) = \|\mathbf{x}_c^*(\epsilon) - \mathbf{x}_0\|_2^2$ be its $\ell_2$ distance from $\mathbf{x}_0$. If  $\nu(\epsilon_1) > 0$ and $\mathbf{x}_c^*(\epsilon_1) \neq \mathbf{0}$, then for any two robustness levels $0 < \epsilon_1 < \epsilon_2$, $\nu(\epsilon_1) < \nu(\epsilon_2)$. 
}

\begin{theorem}[Robustness-Proximity Trade-off]
    \label{thm:tradeoff}
\TheoremTradeoff
\end{theorem}

The practical impact of this trade-off is significant. Overly robust counterfactuals may become distant and unactionable, while insufficient robustness compromises recourse reliability under model shifts. This underscores the need for methods that efficiently explore this trade-off by achieving substantial robustness with reasonable proximity---a goal that ElliCE effectively meets (Figure \ref{fig:side_by_side_minipage}). 




When applying our theoretical results to MLPs, \textit{the validity guarantee is fully preserved in the input space}, which is a key result. The formal guarantees for uniqueness (Theorem~\ref{thm:uniqueness}), stability (Theorem~\ref{thm:stability}), and the robustness-proximity trade-off (Theorem~\ref{thm:tradeoff}), however, depend on the convexity of the feasible set (see proof of Theorem  \ref{thm:uniqueness}). While this convexity is guaranteed in the embedding space $h(\mathbf{x})$, the nonlinear mapping from the input space ($\mathbf{x} \mapsto h(\mathbf{x})$) means it is not guaranteed to hold there. This distinction highlights a fundamental challenge for robust recourse in deep models and underscores that extending these formal guarantees to the input space is a promising direction for future work. Nonetheless, these theorems provide a principled geometric foundation for our approach and hold for linear models and embedding spaces. Next, we present empirical results showing that ElliCE's performance is consistent with its theoretical guarantees.

\section{Evaluation Pipeline and Experimental Results}\label{section:experiments}


In our evaluation pipeline, we work with the hypothesis space of linear models and multi-layer perceptrons (MLPs). However, our results can be extended to other hypothesis spaces that can be optimized with gradient descent, such as neural additive models \cite{agarwal2021neural}. 
In this section, we empirically show that ElliCE is faster and more robust as compared to other methods that produce robust counterfactuals.
Please see Appendix \ref{appendix:experiments} for additional details and results.


\textbf{Datasets.} We consider nine datasets from high-stakes decision domains such as lending (Australian Credit \cite{statlog_143}, FICO \cite{fico2018heloc}, German Credit \cite{statlog_144}, Banknote  \cite{banknote_authentication_267}), healthcare (Parkinson’s \cite{parkinsons_telemonitoring_189}, Diabetes \cite{smith1988pima}), and recidivism (COMPAS \cite{angwin2016machine}), as well as benchmark datasets (Wine Quality \cite{wine_quality_186}, Extended Iris \cite{samy_baladram_2023}). Please see Table \ref{tab:datasets} for detailed dataset descriptions and preprocessing notes. We used datasets with predominantly categorical features (FICO, Australian Credit, COMPAS, German Credit, Diabetes) for data-supported CE generation, and datasets with continuous features (Diabetes, Parkinson’s, Banknote, Iris, and Wine Quality) for continuous methods.
We balanced the datasets, standardized continuous features, and, for some datasets, dropped rows with missing values. 



\textbf{Baselines.} We compare ElliCE to other methods that are designed to generate robust counterfactual explanations, such as  T:Rex, Interval Abstractions (we refer to it as  Delta-robustness \cite{JiangLeofanteRagoToni2022}), PROPLACE, and ROAR. 
\textit{T:Rex} \cite{hamman2023robust} generates robust counterfactuals for neural networks using a Stability measure that depends on variance. It quantifies robustness to naturally occurring model changes, providing probabilistic validity guarantees. It is a successor of  RobX \cite{DuttaLongMishraTilliMagazzeni2022}, which targets tree-based ensembles.
\textit{Interval Abstractions} \cite{JiangLeofanteRagoToni2024}  ensures that  counterfactuals are robust to bounded changes in model parameters (weights and biases). It uses interval neural networks and mixed-integer linear programming.
\textit{PROPLACE} \cite{pmlr-v222-jiang24a} 
formulates counterfactual generation as a bi-level robust optimization problem: it enforces plausibility by restricting solutions to the convex hull of realistic samples and uses interval bounds on neural networks to ensure robustness.
\textit{ROAR} \cite{upadhyay2021towards} optimizes counterfactual validity under bounded model parameter perturbations using a robustness-constrained loss formulation.
 Most implementations of our baselines follow \citet{jiang2025robustx}.

\textbf{Evaluators.} Precisely computing the entire Rashomon set for the hypothesis spaces that we consider is intractable. Therefore, to evaluate the robustness and validity of counterfactual explanations generated by ElliCE and the baselines, we rely on established techniques that approximate or characterize this set. These approaches generate diverse collections of near-optimal models, each serving as a proxy for the actual Rashomon set. 
Our evaluators include:
\textit{Random Retrain}, which retrains models multiple times with different random seeds to capture procedural variability.
\textit{Rashomon Dropout} \cite{hsu2024dropout}, which applies Gaussian dropout to a single trained neural network’s weights during inference, creating an ensemble of stochastically perturbed models.
\textit{Adversarial Weight Perturbation (AWP)} \cite{hsu2022rashomon}, which generates diverse models from an initially trained model by applying small perturbations to its weights. We define the objective tolerance (Rashomon parameter) for the evaluators as $\varepsilon_{\text{target}}$, which is distinct from $\epsilon$. This separation ensures that the Rashomon set used for evaluation is controlled independently from the robustness tolerance $\epsilon$ used by ElliCE.

\textbf{Metrics.} We evaluate the generated counterfactual explanations based on four metrics: validity, proximity, robustness, and plausibility. 
 \textit{Validity} measures whether a generated counterfactual $\mathbf{x}_{c}$ for a given input $\mathbf{x}_{0}$ successfully achieves the desired outcome $c$ when evaluated on the original model $f_{\text{baseline}}$ for which it was generated, $\text{Validity} = \frac{1}{n} \sum_{i=1}^{n} \mathbf{1}[f_{\text{baseline}}(\mathbf{x}_{ci})=c].$
\textit{Proximity} measures the closeness of a counterfactual $\mathbf{x}_{c}$ to the original instance $\mathbf{x}_{0}$. We primarily report the $\ell_2$ distance: $\|\mathbf{x}_{c} - \mathbf{x}_{0}\|_2$. Lower values indicate less change required and are thus better.
\textit{Plausibility} checks whether the generated counterfactuals lie in realistic regions of the feature space. 
Our data-supported counterfactuals are inherently plausible, as they lie on the data manifold. For continuous approach, because ElliCE enforces robustness by pushing counterfactuals away from the decision boundary, the resulting counterfactuals might shift toward higher-density regions of the target class. Nevertheless, we evaluate plausibility using the Local Outlier Factor (LOF) \cite{pmlr-v222-jiang24a}, 
a standard outlier-detection metric.
LOF values close to~1 indicate high plausibility, whereas larger values suggest 
the counterfactual is in a low-density region.
\textit{Robustness} computes whether the generated counterfactual $\mathbf{x}_{c}$ remains valid (i.e., still achieves the desired outcome $c$) for all models within an evaluator ensemble $\tilde{\mathcal{R}}(\varepsilon_{\text{target}})$. Total is calculated as the average across all $n$ counterfactual points:
$\text{Robustness} = \frac{1}{n} \sum_{i=1}^{n} \mathbf{1}\left[ \forall f_{\bm{\theta}} \in \tilde{\mathcal{R}}(\varepsilon_{\text{target}}), f_{\bm{\theta}}(\mathbf{x}_{{c_i}}) = c \right].$
A higher robustness score (closer to 1) is better, indicating that more counterfactual explanations are robust to model changes.

\textbf{Experimental Setup.} For evaluators, we define a target multiplicity tolerance globally in range \(\varepsilon_\text{target} \in [0, 0.1]\).  
We provide discussion on how to choose ElliCE's \(\epsilon\) in Appendix \ref{appendix:experiments}.
For every dataset, we performed 4-fold stratified cross-validation. Within each fold, the training data are further split into 80\% for training and 20\% for validation.
The procedure within each inner fold is as follows: (1) We train a base model $f_\text{baseline}$, which serves as a reference model for all counterfactual generation methods.
(2) Using $f_\text{baseline}$ as a reference (if required by the evaluation method), we generate $\varepsilon_\text{target}$-Rashomon set.  
(3) Multiplicity parameters for each baseline  ($\epsilon$ for ElliCE, $\delta$ for Delta Robustness, ROAR and PROPLACE, or $\tau$ for T:Rex)  are tuned via grid search on the validation set with a goal of maximizing validity and robustness. We allocate approximately the same amount of time for each method to tune its parameters with a hard maximum of 8 hours per method per data fold (as a result, we could not run PROPLACE for Parkinsons dataset).
(4) Final performance metrics are reported on the held-out split of the outer fold. Note that due to our tuning procedure, we expect high validity metric for ElliCE and baselines. Indeed, for data-supported methods validity is consistently 100\% across datasets, so we do not report it.

We conducted experiments on logistic regression and multilayer perceptrons. Consistent with prior work \cite{upadhyay2021towards, pmlr-v222-jiang24a}, we focus on generating counterfactuals that change predicted labels from 0 to 1. Linear models are trained using Scikit-learn's LBFGS solver with an $\ell_2$ penalty (regularization parameter 0.001). MLPs are trained with the Adam optimizer (learning rate 0.001), early stopping, and $\ell_2$ regularization parameter 0.001.
For evaluation, we generate one counterfactual per method for each data point in the held-out set. Each counterfactual is then evaluated against the three evaluators (Random Retrain, Rashomon Dropout, AWP). The exact construction algorithms for these evaluators are described in Appendix~\ref{appendix:emp_rash_c}. Reported metrics are averaged across data points and folds, with plots displaying the mean and standard error.

\begin{figure}[h]
    \centering
\includegraphics[width=0.95\textwidth]{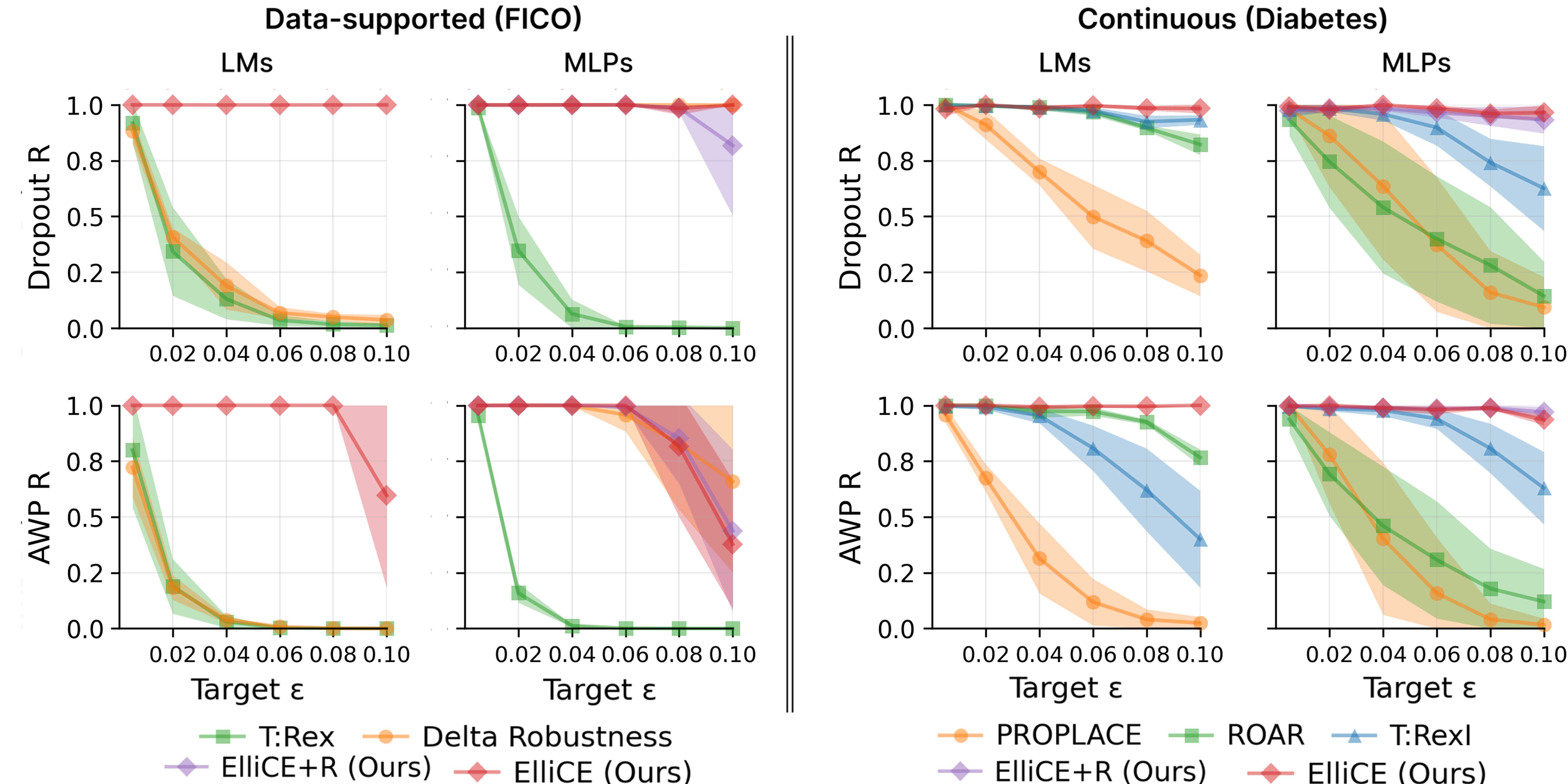}
   \caption{Robustness evaluation of ElliCE against baselines. The plot displays the robustness metric (y-axis) as a function of the target robustness level $\varepsilon_{\text{target}}$ for the evaluators (x-axis). ElliCE consistently outperforms all baselines across all robustness levels. See Appendix \ref{appendix:experiments} for more figures. 
   For ElliCE+R for MLPs, we apply additional regularization to the Hessian, using $\lambda = 0.1$ instead of $0.001$. 
}
    \label{fig:target_eps_rob}
\end{figure}

\begin{table}[h]
\centering
\footnotesize
\caption{Performance of counterfactual methods on MLPs. For evaluators, we set $\varepsilon_{\text{target}}$ to 10\% of the  training objective ($\varepsilon_{\text{target}}=0.1\times \hat{L}(f_{\text{baseline}})$). \(\mathbf{R}\) here stands for Robustness, $\mathbf{L2}$ for proximity, and PROP stands for PROPLACE. See Appendix \ref{appendix:experiments} for results on other datasets.}
\label{tab:performance_comparison_sub}
\setlength{\tabcolsep}{4pt}
\begin{tabular}{l|l|cc|cc|cc}
\hline
\multirow{3}{*}{\textbf{Data}} & \multirow{3}{*}{\textbf{Method}} & \multicolumn{6}{c}{\textbf{Evaluation Metric}} \\
\cline{3-8}
& & \multicolumn{2}{c|}{Retrain} & \multicolumn{2}{c|}{Dropout Rashomon} & \multicolumn{2}{c}{AWP} \\
& & \textbf{R$\uparrow$} & \textbf{L2$\downarrow$} & \textbf{R$\uparrow$} & \textbf{L2$\downarrow$} & \textbf{R$\uparrow$} & \textbf{L2$\downarrow$} \\
\hline
\multicolumn{8}{c}{\textbf{Data-supported (DS)}} \\
\hline

\multirow{3}{*}{FICO} 
& ElliCE   
& \textbf{1.00 $\pm$ 0.00} & 3.53 $\pm$ 0.17 & \textbf{1.00 $\pm$ 0.00} & 4.91 $\pm$ 0.22 & \textbf{1.00 $\pm$ 0.00} & 5.06 $\pm$ 0.29 \\
& DeltaRob 
& \textbf{1.00 $\pm$ 0.00} & 4.00 $\pm$ 0.10 & \textbf{1.00 $\pm$ 0.00} & 5.67 $\pm$ 0.58 & 0.96 $\pm$ 0.07 & 5.70 $\pm$ 0.72 \\
& T:Rex    
& 0.83 $\pm$ 0.08 & 3.12 $\pm$ 0.07 & 0.01 $\pm$ 0.00 & 3.07 $\pm$ 0.11 & 0.00 $\pm$ 0.00 & 2.77 $\pm$ 0.19 \\
\hline

\multirow{3}{*}{German} 
& ElliCE 
& \textbf{1.00 $\pm$ 0.00} & 3.48 $\pm$ 0.10 & \textbf{1.00 $\pm$ 0.00} & 4.32 $\pm$ 0.31 & \textbf{1.00 $\pm$ 0.00} & 4.00 $\pm$ 0.24 \\
& DeltaRob 
& 0.98 $\pm$ 0.01 & 3.45 $\pm$ 0.06 & 0.99 $\pm$ 0.02 & 4.00 $\pm$ 0.15 & \textbf{1.00 $\pm$ 0.00} & 3.99 $\pm$ 0.22 \\
& T:Rex    
& 0.99 $\pm$ 0.01 & 3.47 $\pm$ 0.04 & 0.97 $\pm$ 0.02 & 4.03 $\pm$ 0.20 & 0.99 $\pm$ 0.01 & 4.23 $\pm$ 0.24 \\
\hline
\hline
\multicolumn{8}{c}{\textbf{Continuous (CNT)}} \\
\hline
\multirow{4}{*}{Diabetes}
& ElliCE   
& \textbf{0.98 $\pm$ 0.01} & 2.15 $\pm$ 0.39 & \textbf{0.99 $\pm$ 0.02} & 3.05 $\pm$ 0.34 & \textbf{0.98 $\pm$ 0.02} & 3.22 $\pm$ 0.40 \\
& PROP
& 0.48 $\pm$ 0.48 & 2.01 $\pm$ 0.05 & 0.19 $\pm$ 0.28 & 2.01 $\pm$ 0.05 & 0.08 $\pm$ 0.19 & 2.01 $\pm$ 0.05 \\
& ROAR     
& 0.86 $\pm$ 0.11 & 1.86 $\pm$ 0.24 & 0.40 $\pm$ 0.28 & 1.86 $\pm$ 0.24 & 0.31 $\pm$ 0.26 & 1.86 $\pm$ 0.24 \\
& T:Rex    
& 0.94 $\pm$ 0.03 & 2.47 $\pm$ 0.86 & 0.90 $\pm$ 0.08 & 4.18 $\pm$ 0.36 & 0.94 $\pm$ 0.04 & 4.18 $\pm$ 0.36 \\
\hline
\end{tabular}
\end{table}

\subsection{ElliCE Generates  Robust Counterfactuals}

Figure~\ref{fig:target_eps_rob} illustrates the relationship between the evaluators' multiplicity level $\varepsilon_{\text{target}}$ and the achieved robustness for the baselines. We report results for both linear models and  MLPs for data-supported and continuous methods. Across different  settings, we observe that ElliCE consistently produces more robust counterfactuals than baselines. Notably, ElliCE's counterfactuals generally do not exhibit a decrease in robustness as $\varepsilon_{\text{target}}$ increases, demonstrating stability under different levels of target multiplicity. This  robustness, however, can sometimes come with a greater distance from the original instance (i.e., longer CEs), a trade-off that we saw in Section~\ref{section:properties} and report in Table~\ref{tab:performance_comparison_sub}.
For the MLP setting, 
our empirical results in Figure~\ref{fig:target_eps_rob} and 
Table~\ref{tab:performance_comparison_sub} suggest that ElliCE's ellipsoidal approximation offers good flexibility, allowing it to adapt to the underlying loss function's shape.




\subsection{ElliCE is Efficient}

Tables~\ref{tab:runtime_discrete_sp}, ~\ref{tab:runtime_dats_cm} and ~\ref{tab:runtime_continuous_cm} clearly demonstrate ElliCE's advantage in computational efficiency. Our method is consistently faster than baselines with speedups of up to three orders of magnitude. 
The runtimes of both T:Rex and Delta Robustness tend to grow substantially with the dataset size. In contrast, ElliCE remains lightweight and exhibits better scalability. Across all datasets tested, ElliCE's absolute runtimes for generating a counterfactual remain under  two seconds. This efficiency comes from a closed-form solution for the inner optimization problem (Theorem~1). 
The primary preprocessing cost involves computing and inverting the Hessian matrix $H$, requiring $O(np^2)$ for computation and $O(p^3)$ for inversion, performed once per model (where $n$ is the training set size and $p$ is the parameter dimension). Per-instance counterfactual generation then requires only $O(p^2)$ operations.

\begin{table}[t]
\centering
\footnotesize
\caption{Runtime performance and speedups for data-supported CE for MLPs. }
\setlength{\tabcolsep}{4pt}
\begin{tabular}{l|ccc||cc}
\hline
 & \multicolumn{3}{c||}{Absolute (seconds)} & \multicolumn{2}{c}{Relative (speedup)} \\
\textbf{Dataset} & \textbf{ElliCE} & \textbf{T:Rex} & \textbf{Delta Rob} & \textbf{Over T:Rex} & \textbf{Over Delta Rob} \\
\hline
FICO & $1.792 \pm 0.123$ & $7.006 \pm 0.058$ & $242.035 \pm 1.161$ & $3.91\times$ & $135.04\times$ \\
COMPAS & $0.526 \pm 0.011$ & $3.534 \pm 0.128$ & $360.480 \pm 6.701$ & $6.72\times$ & $685.34\times$ \\
Australian & $0.057 \pm 0.011$ & $0.281 \pm 0.006$ & $2.783 \pm 0.032$ & $4.92\times$ & $48.64\times$ \\
Diabetes & $0.053 \pm 0.001$ & $0.296 \pm 0.006$ & $1.922 \pm 0.032$ & $5.60\times$ & $36.33\times$ \\
German & $0.101 \pm 0.001$ & $0.432 \pm 0.013$ & $9.905 \pm 0.068$ & $4.27\times$ & $97.88\times$ \\
\hline
\end{tabular}
\label{tab:runtime_discrete_sp}
\end{table}


\begin{figure}[t]
    \centering
    \includegraphics[width =\textwidth]{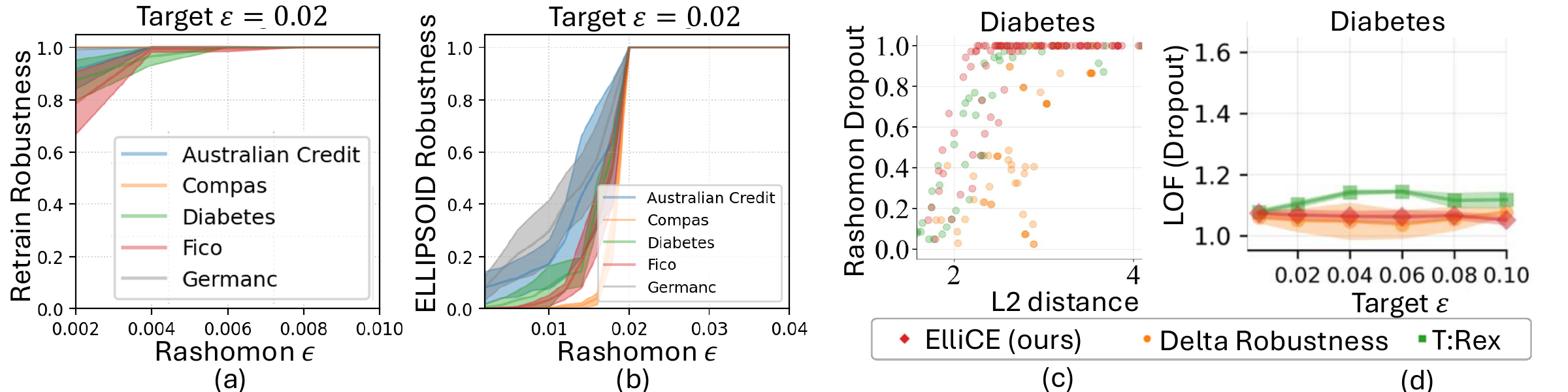}
    \caption{(a,b) Sensitivity of ElliCE's robustness (y-axis) to its internal $\epsilon$ hyperparameter (x-axis). Robustness is evaluated against Random Retrain (a) and an Ellipsoidal Rashomon set approximation defined with a fixed $\varepsilon_{\text{target}}$ (b).
(c, d) Robustness  vs. $\ell_2$ proximity trade-off (c) and  plausibility (d) of counterfactuals generated by ElliCE and baselines on Diabetes dataset. }
\label{fig:side_by_side_minipage}
\end{figure}

\subsection{Sensitivity Analysis}
Figure~\ref{fig:side_by_side_minipage} (a,b) shows an empirical sensitivity analysis of ElliCE's robustness with respect to its internal Rashomon parameter $\epsilon$. The plots show how the achieved robustness (evaluated against the Random Retrain and Ellipsoidal Rashomon set evaluators, respectively) varies as ElliCE's internal $\epsilon$ changes. These results illustrate that ElliCE can achieve high levels of robustness even for relatively small values of its internal $\epsilon$ when evaluated against the Retrain ensemble. For the middle plot (Ellipsoidal evaluator), while initial robustness may be lower for smaller internal $\epsilon$ values, the performance increases sharply, as $\epsilon$ approaches the targeted robustness level.

\subsection{Robustness-Proximity Trade-off and Plausibility}

Figure~\ref{fig:side_by_side_minipage}(c) illustrates the inherent trade-off between robustness and proximity for CEs generated by ElliCE, supporting our discussion in Section~\ref{section:properties}. 
While the trade-off occurs for all baselines, ElliCE achieves the highest robustness at a given length level.
Understanding this trade-off is key to selecting counterfactuals that balance reliability under model shifts with practical user actionability. ElliCE provides a mechanism to navigate this by allowing control over its Rashomon parameter.
We also observed good plausibility across all baselines and datasets, as supported by  Figure~\ref{fig:side_by_side_minipage}(d)  and \ref{fig:lof_ds_lin}. 
All LOF values tend to be close to~1, thus the generated counterfactuals lie on the data manifold.

\subsection{Actionability}

To ensure that generated recourse remains realistic and feasible, we incorporate actionability constraints that specify which features can change and within what ranges. ElliCE supports restrictions on features, including immutable features (e.g., age, citizenship) as well as range and direction constraints such as income or loan duration. It also allows for sparse counterfactuals by adding an optional penalty on the number of modified features. For instance, before applying actionability, one robust counterfactual on the German Credit dataset suggested changing the applicant’s age, an immutable feature. After enforcing immutability and sparsity constraints, ElliCE instead adjusted the credit amount and credit length, reducing both and thus lowering the predicted credit risk, which is reasonable in the lending context. Further details  are provided in Appendix \ref{appendix:actionability}.

\section{Conclusions, Implications and Limitations}
\label{section:conclusion_limitations}

Standard algorithmic recourse is fragile. A recommendation given to a user today may become invalid tomorrow if the underlying model is retrained or replaced---a common scenario under the Rashomon Effect. This paper addressed this reliability gap by introducing ElliCE, a framework that provides recourse with provable robustness guarantees. ElliCE approximates the set of near-optimal models with an ellipsoid and computes counterfactuals that remain valid across this approximated Rashomon set.
A strength of ElliCE is its support for actionability. Users can specify immutable features, range or direction constraints, and optional sparsity penalties, ensuring that the resulting recourse is both robust and realistic. This flexibility might help prevent impractical or unethical recommendations and gives users greater control over actions. While robustness alone does not ensure fairness, user-specified actionability constraints can help to ensure that counterfactuals remain feasible and ethically sound. A comprehensive fairness analysis remains an important direction for future work.
The ellipsoidal approximation, while efficient, is a simplification of the true Rashomon set, and for neural networks our analysis currently captures local rather than global model multiplicity. Despite these limitations, ElliCE offers a practical and theoretically grounded tool for robust and actionable recourse, providing stable and trustworthy advice.

\section*{Acknowledgments}
We thank the \href{https://r-ai.co/ukraine}{RAI for Ukraine} program, led by the Center for Responsible AI at New York University in collaboration with Ukrainian Catholic University in Lviv, for supporting Bohdan's and Iryna's participation in this research.

\section*{Code Availability}

Implementations of ElliCE are available at \url{https://github.com/BogdanTurbal/ElliCE_EXPERIMENTS}.

\bibliographystyle{plainnat}
\bibliography{literature}

@article{WachterMittelstadtRussell2017,
title={Counterfactual explanations without opening the black box: Automated decisions and the {GDPR}},
  author={Wachter, Sandra and Mittelstadt, Brent and Russell, Chris},
  journal={Harv. JL \& Tech.},
  volume={31},
  pages={841},
  year={2017},
  publisher={HeinOnline}
}

@inproceedings{cavus2025beyond,
  title={Beyond the Single-Best Model: {R}ashomon Partial Dependence Profile for Trustworthy Explanations in {A}uto{ML}},
  author={Cavus, Mustafa and van Rijn, Jan N and Biecek, Przemys{\l}aw},
  booktitle={International Conference on Discovery Science},
  pages={445--459},
  year={2025},
  organization={Springer}
}

@article{kuratomi2025subgroup,
  title={Subgroup fairness based on shared counterfactuals},
  author={Kuratomi, Alejandro and Lee, Zed and Tsaparas, Panayiotis and Pitoura, Evaggelia and Lindgren, Tony and Dinis Junior, Guilherme and Papapetrou, Panagiotis},
  journal={Knowledge and Information Systems},
  pages={1--39},
  year={2025},
  publisher={Springer}
}

@article{garg2025search,
  title={From Search To Sampling: Generative Models For Robust Algorithmic Recourse},
  author={Garg, Prateek and Nagalapatti, Lokesh and Sarawagi, Sunita},
  journal={arXiv preprint arXiv:2505.07351},
  year={2025}
}

@article{barrainkua2025pays,
  title={Who Pays for Fairness? {R}ethinking Recourse under Social Burden},
  author={Barrainkua, Ainhize and De Toni, Giovanni and Lozano, Jose Antonio and Quadrianto, Novi},
  journal={arXiv preprint arXiv:2509.04128},
  year={2025}
}

@article{konig2025performative,
  title={Performative validity of recourse explanations},
  author={K{\"o}nig, Gunnar and Fokkema, Hidde and Freiesleben, Timo and Mendler-D{\"u}nner, Celestine and von Luxburg, Ulrike},
  journal={arXiv preprint arXiv:2506.15366},
  year={2025}
}

@article{ceccon2025reinforcement,
  title={Reinforcement Learning for Durable Algorithmic Recourse},
  author={Ceccon, Marina and Fabris, Alessandro and Radanovi{\'c}, Goran and Biega, Asia J and Susto, Gian Antonio},
  journal={arXiv preprint arXiv:2509.22102},
  year={2025}
}

@article{kyaw2025optimal,
  title={Optimal Robust Recourse with $L^{p}$ - Bounded Model Change},
  author={Kyaw, Phone and Kayastha, Kshitij and Jabbari, Shahin},
  journal={arXiv preprint arXiv:2509.21293},
  year={2025}
}

@article{Yetukuri_Hardy_Vorobeychik_Ustun_Liu_2024,
  title   = {Providing Fair Recourse over Plausible Groups},
  author  = {Yetukuri, Jayanth and Hardy, Ian and Vorobeychik, Yevgeniy and Ustun, Berk and Liu, Yang},
  journal = {Proceedings of the AAAI Conference on Artificial Intelligence},
  volume  = {38},
  number  = {19},
  pages   = {21753--21760},
  year    = {2024},
  month   = {Mar},
  url     = {}
}

@InProceedings{pmlr-v222-jiang24a,
  title = 	 {Provably Robust and Plausible Counterfactual Explanations for Neural Networks via Robust Optimisation},
  author =       {Jiang, Junqi and Lan, Jianglin and Leofante, Francesco and Rago, Antonio and Toni, Francesca},
  booktitle = 	 {Proceedings of the 15th Asian Conference on Machine Learning},
  pages = 	 {582--597},
  year = 	 {2024},
  editor = 	 {Yanıkoğlu, Berrin and Buntine, Wray},
  volume = 	 {222},
  series = 	 {Proceedings of Machine Learning Research},
  month = 	 {11--14 Nov},
  publisher =    {PMLR},
}

@article{jiang2025robustx,
  title={Robust{X}: Robust Counterfactual Explanations Made Easy},
  author={Jiang, Junqi and Marzari, Luca and Purohit, Aaryan and Leofante, Francesco},
  journal={arXiv preprint arXiv:2502.13751},
  year={2025}
}

@inproceedings{McGrathCostabelloLeVanSweeneyKamiabShenLecue2018,
  title={Interpretable Credit Application Predictions With Counterfactual Explanations},
  author={Mc Grath, Rory and Costabello, Luca and Le Van, Chan and Sweeney, Paul and Kamiab, Farbod and Shen, Zhao and Lecue, Freddy},
  booktitle={NIPS 2018-Workshop on Challenges and Opportunities for AI in Financial Services: the Impact of Fairness, Explainability, Accuracy, and Privacy},
  year={2018}
}

@article{LaugelLesotMarsalaRenardDetyniecki2017,
  title={Inverse classification for comparison-based interpretability in machine learning},
  author={Laugel, Thibault and Lesot, Marie-Jeanne and Marsala, Christophe and Renard, Xavier and Detyniecki, Marcin},
  journal={arXiv preprint arXiv:1712.08443},
  year={2017}
}

@inproceedings{Russell2019efficient,
  title={Efficient search for diverse coherent explanations},
  author={Russell, Chris},
  booktitle={Proceedings of the conference on fairness, accountability, and transparency},
  pages={20--28},
  year={2019}
}

@inproceedings{UstunSpangherLiu2018,
  title={Actionable recourse in linear classification},
  author={Ustun, Berk and Spangher, Alexander and Liu, Yang},
  booktitle={Proceedings of the conference on fairness, accountability, and transparency},
  pages={10--19},
  year={2019}
}

@inproceedings{KarimiBartheBalleValera2020a,
  author       = {Amir-Hossein Karimi and Gilles Barthe and Borja Balle and Isabel Valera},
  title        = {Model-Agnostic Counterfactual Explanations for Consequential Decisions},
  booktitle    = {Proceedings of the 23rd International Conference on Artificial Intelligence and Statistics (AISTATS)},
  series       = {Proceedings of Machine Learning Research},
  volume       = {108},
  pages        = {895--905},
  year         = {2020},
  editor       = {Silvia Chiappa and Roberto Calandra},
  publisher    = {PMLR},
  month        = {Aug},
}

@inproceedings{PawelczykBroelemannKasneci2020,
  author       = {Martin Pawelczyk and Klaus Broelemann and Gjergji Kasneci},
  title        = {Learning Model-Agnostic Counterfactual Explanations for Tabular Data},
  booktitle    = {Proceedings of The Web Conference 2020 (WWW '20)},
  pages        = {3126--3132},
  year         = {2020},
  location     = {Taipei, Taiwan},
  publisher    = {ACM / IW3C2},
}

@article{JoshiKoyejoVijitbenjaronkKimGhosh2019,
  title={Towards realistic individual recourse and actionable explanations in black-box decision making systems},
  author={Joshi, Shalmali and Koyejo, Oluwasanmi and Vijitbenjaronk, Warut and Kim, Been and Ghosh, Joydeep},
  journal={arXiv preprint arXiv:1907.09615},
  year={2019}
}

@inproceedings{NemirovskyThiebautXuGupta2022a,
  author    = {Daniel Nemirovsky and Nicolas Thiebaut and Ye Xu and Abhishek Gupta},
  title     = {Counte{RGAN}: Generating Counterfactuals for Real-Time Recourse and Interpretability Using Residual {GAN}s},
  booktitle = {Proceedings of the Thirty-Eighth Conference on Uncertainty in Artificial Intelligence},
  series    = {Proceedings of Machine Learning Research},
  volume    = {180},
  pages     = {1488--1497},
  year      = {2022},
  editor    = {James Cussens and Kun Zhang},
  publisher = {PMLR},
  month     = {Aug}
}

@inproceedings{PoyiadziSokolSantosRodriguezDeBieFlach2019,
  title={FACE: {F}easible and actionable counterfactual explanations},
  author={Poyiadzi, Rafael and Sokol, Kacper and Santos-Rodriguez, Raul and De Bie, Tijl and Flach, Peter},
  booktitle={Proceedings of the AAAI/ACM Conference on AI, Ethics, and Society},
  pages={344--350},
  year={2020}
}

@article{brughmans2021nice,
  title={Nice: {A}n algorithm for nearest instance counterfactual explanations},
  author={Brughmans, Dieter and Leyman, Pieter and Martens, David},
  journal={Data mining and knowledge discovery},
  volume={38},
  number={5},
  pages={2665--2703},
  year={2024},
  publisher={Springer}
}

@inproceedings{DuttaLongMishraTilliMagazzeni2022,
  author    = {Sanghamitra Dutta and Jason Long and Saumitra Mishra and Cecilia Tilli and Daniele Magazzeni},
  title     = {Robust Counterfactual Explanations for Tree-Based Ensembles},
  booktitle = {Proceedings of the 39th International Conference on Machine Learning},
  series    = {Proceedings of Machine Learning Research},
  volume    = {162},
  pages     = {5742--5756},
  year      = {2022},
  editor    = {Kamalika Chaudhuri and Stefanie Jegelka and Le Song and Csaba Szepesvari and Gang Niu and Sivan Sabato},
  publisher = {PMLR},
  month     = {Jul},
}

@inproceedings{ForelParmentierVidal2022,
  title={Don’t explain noise: Robust counterfactuals for randomized ensembles},
  author={Forel, Alexandre and Parmentier, Axel and Vidal, Thibaut},
  booktitle={International Conference on the Integration of Constraint Programming, Artificial Intelligence, and Operations Research},
  pages={293--309},
  year={2024},
  organization={Springer}
}

@article{Kinjo2025robustcounterfactual,
  title={Robust counterfactual explanations under model multiplicity using multi-objective optimization},
  author={Kinjo, Keita},
  journal={arXiv preprint arXiv:2501.05795},
  year={2025}
}

@inproceedings{JiangLeofanteRagoToni2022,
  title={Formalising the robustness of counterfactual explanations for neural networks},
  author={Jiang, Junqi and Leofante, Francesco and Rago, Antonio and Toni, Francesca},
  booktitle={Proceedings of the AAAI conference on artificial intelligence},
  volume={37},
  number={12},
  pages={14901--14909},
  year={2023}
}

@article{JiangLeofanteRagoToni2024,
  title={Interval abstractions for robust counterfactual explanations},
  author={Jiang, Junqi and Leofante, Francesco and Rago, Antonio and Toni, Francesca},
  journal={Artificial Intelligence},
  volume={336},
  pages={104218},
  year={2024},
  publisher={Elsevier}
}

@article{upadhyay2021towards,
  title={Towards robust and reliable algorithmic recourse},
  author={Upadhyay, Sohini and Joshi, Shalmali and Lakkaraju, Himabindu},
  journal={Advances in Neural Information Processing Systems ({NeurIPS})},
  volume={34},
  pages={16926--16937},
  year={2021}
}

@inproceedings{jiang2024argumentative,
  title={Recourse under Model Multiplicity via Argumentative Ensembling},
  author={Jiang, Junqi and Leofante, Francesco and Rago, Antonio and Toni, Francesca},
  booktitle={Proceedings of the 23rd International Conference on Autonomous Agents and Multiagent Systems},
  pages={954--964},
  year={2024}
}

@inproceedings{hamman2023robust,
  title={Robust Counterfactual Explanations for Neural Networks with Probabilistic Guarantees},
  author={Hamman, Faisal and Noorani, Erfaun and Mishra, Saumitra and Magazzeni, Daniele and Dutta, Sanghamitra},
  booktitle={Proceedings of the 40th International Conference on Machine Learning},
  year={2023}
}

@inproceedings{xin2022rashomontrees,
  title={Exploring the Whole {R}ashomon Set of Sparse Decision Trees},
  author={Xin, Rui and Zhong, Chudi and Chen, Zhi and Takagi, Takuya and Seltzer, Margo and Rudin, Cynthia},
  booktitle={Advances in Neural Information Processing Systems ({NeurIPS})},
  volume={35},
  year={2022}
}

@inproceedings{semenova2023noise,
  title={A Path to Simpler Models Starts With Noise},
  author={Semenova, Lesia and Chen, Harry and Parr, Ronald and Rudin, Cynthia},
  booktitle={Advances in Neural Information Processing Systems ({NeurIPS})},
  year={2023}
}

@inproceedings{mohammadi2021scaling,
  title={Scaling Guarantees for Nearest Counterfactual Explanations},
  author={Mohammadi, Kiarash and Karimi, Amir-Hossein and Barthe, Gilles and Valera, Isabel},
  booktitle={Proceedings of the Thirty-Fifth AAAI Conference on Artificial Intelligence},
  year={2021}
}

@inproceedings{black2022model,
  title={Model multiplicity: Opportunities, concerns, and solutions},
  author={Black, Emily and Raghavan, Manish and Barocas, Solon},
  booktitle={Proceedings of the 2022 ACM Conference on Fairness, Accountability, and Transparency},
  pages={850--863},
  year={2022}
}

@inproceedings{semenova2022existence,
  title={On the existence of simpler machine learning models},
  author={Semenova, Lesia and Rudin, Cynthia and Parr, Ronald},
  booktitle={Proceedings of the 2022 ACM Conference on Fairness, Accountability, and Transparency},
  pages={1827--1858},
  year={2022}
}

@misc{samy_baladram_2023,
	title={Iris Dataset Extended},
	url={ },
	publisher={Kaggle},
	author={Samy Baladram},
	year={2023}
}

@misc{smith1988pima,
  author       = {Smith, John W. and Everhart, William A. and Dickson, William C. and Knowler, William C. and Johannes, Richard S.},
  title        = {{Pima Indians Diabetes Database}},
  year         = {1988},

}

@misc{fico2018heloc,
  author       = {{Fair Isaac Corporation (FICO)}},
  title        = {{FICO} Explainable Machine Learning Challenge: Home Equity Line of Credit (HELOC) dataset},
  year         = {2018}
}

@misc{statlog_144,
  author       = {Hofmann, Hans},
  title        = {{Statlog (German Credit Data)}},
  year         = {1994},
  howpublished = {UCI Machine Learning Repository},
}

@misc{statlog_143,
  author       = {Quinlan, Ross},
  title        = {{Statlog (Australian Credit Approval)}},
  year         = {1987},
  howpublished = {UCI Machine Learning Repository},
}

@article{breiman2001statistical,
	title        = {Statistical modeling: The two cultures (with comments and a rejoinder by the author)},
	author       = {Breiman, Leo},
	year         = 2001,
	journal      = {Statistical Science},
	publisher    = {Institute of Mathematical Statistics},
	volume       = 16,
	number       = 3,
	pages        = {199--231}
}

@inproceedings{SunEtAl24,
	title        = {Sparse and Faithful Explanations without Sparse Models},
	author       = {Yiyang Sun and Zhi Chen and Vittorio Orlandi and Tong Wang and Cynthia Rudin},
	year         = 2024,
	booktitle    = {Proc. Artificial Intelligence and Statistics {(AISTATS)}}
}

@inproceedings{marx2020predictive,
	title        = {Predictive multiplicity in classification},
	author       = {Marx, Charles and Calmon, Flavio and Ustun, Berk},
	year         = 2020,
	booktitle    = {Proceedings of the International Conference on Machine Learning (ICML)},
	pages        = {6765--6774}
}

@inproceedings{hsu2022rashomon,
	title        = {Rashomon Capacity: A Metric for Predictive Multiplicity in Classification},
	author       = {Hsu, Hsiang and Calmon, Flavio},
	year         = 2022,
	booktitle    = {Neural Information Processing Systems ({NeurIPS})},
	volume       = 35,
	pages        = {28988--29000}
}

@article{fisher2019all,
	title        = {All Models are Wrong, but Many are Useful: Learning a Variable's Importance by Studying an Entire Class of Prediction Models Simultaneously.},
	author       = {Fisher, Aaron and Rudin, Cynthia and Dominici, Francesca},
	year         = 2019,
	journal      = {Journal of Machine Learning Research},
	volume       = 20,
	number       = 177,
	pages        = {1--81}
}

@article{dong2020exploring,
	title        = {Exploring the cloud of variable importance for the set of all good models},
	author       = {Dong, Jiayun and Rudin, Cynthia},
	year         = 2020,
	journal      = {Nature Machine Intelligence},
	publisher    = {Nature Publishing Group},
	volume       = 2,
	number       = 12,
	pages        = {810--824}
}

@article{agarwal2021neural,
	title        = {Neural additive models: Interpretable machine learning with neural nets},
	author       = {Agarwal, Rishabh and Melnick, Levi and Frosst, Nicholas and Zhang, Xuezhou and Lengerich, Ben and Caruana, Rich and Hinton, Geoffrey E},
	year         = 2021,
	journal      = {Advances in Neural Information Processing Systems ({NeurIPS})},
	volume       = 34,
	pages        = {4699--4711}
}

@inproceedings{zhong2023exploring,
	title        = {Exploring and Interacting with the Set of Good Sparse Generalized Additive Models},
	author       = {Zhong, Chudi and Chen, Zhi and Liu, Jiachang and Seltzer, Margo and Rudin, Cynthia},
	year         = 2023,
	booktitle    = {Neural Information Processing Systems ({NeurIPS})}
}

@inproceedings{donnelly2023the,
	title        = {The {R}ashomon Importance Distribution: Getting {RID} of Unstable, Single Model-based Variable Importance},
	author       = {Jon Donnelly and Srikar Katta and Cynthia Rudin and Edward P Browne},
	year         = 2023,
	booktitle    = {Neural Information Processing Systems ({NeurIPS})}
}

@inproceedings{rudin2024amazing,
	title        = {Amazing Things Come From Having Many Good Models},
	author       = {Cynthia Rudin and Chudi Zhong and Lesia Semenova and Margo Seltzer and Ronald Parr and Jiachang Liu and Srikar Katta and Jon Donnelly and Harry Chen and Zachery Boner},
	year         = 2024,
	booktitle    = {Proceedings of the International Conference on Machine Learning (ICML)}
}

@inproceedings{meyer2024perceptions,
  title={Perceptions of the Fairness Impacts of Multiplicity in Machine Learning},
  author={Meyer, Anna P and Kim, Yea-Seul and D'Antoni, Loris and Albarghouthi, Aws},
  booktitle={Proceedings of the 2025 CHI Conference on Human Factors in Computing Systems},
  pages={1--15},
  year={2025}
}

@inproceedings{boner2024using,
	title        = {Using Noise to Infer Aspects of Simplicity Without Learning},
	author       = {Zachery Boner and Harry Chen and Lesia Semenova and Ronald Parr and Cynthia Rudin},
	year         = 2024,
	booktitle    = {The Thirty-eighth Annual Conference on Neural Information Processing Systems},
}

@inproceedings{
hsu2024dropout,
title={Dropout-Based {R}ashomon Set Exploration for Efficient Predictive Multiplicity Estimation},
author={Hsiang Hsu and Guihong Li and Shaohan Hu and Chun-Fu Chen},
booktitle={The Twelfth International Conference on Learning Representations},
year={2024},
url={ }
}

@article{hsu2024rashomongb,
	title        = {Rashomon{GB}: {A}nalyzing the {R}ashomon {E}ffect and mitigating predictive multiplicity in gradient boosting},
	author       = {Hsu, Hsiang and Brugere, Ivan and Sharma, Shubham and Lecue, Freddy and Chen, Richard},
	year         = 2024,
	journal      = {Advances in Neural Information Processing Systems ({NeurIPS})},
	volume       = 37,
	pages        = {121265--121303}
}

@inproceedings{donnelly2025rashomon,
  title={Rashomon sets for prototypical-part networks: {E}diting interpretable models in real-time},
  author={Donnelly, Jon and Guo, Zhicheng and Barnett, Alina Jade and McTavish, Hayden and Chen, Chaofan and Rudin, Cynthia},
  booktitle={Proceedings of the Computer Vision and Pattern Recognition Conference},
  pages={4528--4538},
  year={2025}
}

@inproceedings{muller2023empirical,
	title        = {An empirical evaluation of the {R}ashomon {E}ffect in explainable machine learning},
	author       = {M{\"u}ller, Sebastian and Toborek, Vanessa and Beckh, Katharina and Jakobs, Matthias and Bauckhage, Christian and Welke, Pascal},
	year         = 2023,
	booktitle    = {Joint European Conference on Machine Learning and Knowledge Discovery in Databases},
	pages        = {462--478},
	organization = {Springer}
}

@article{langlade2025fairness,
	title        = {Fairness and Sparsity within {R}ashomon sets: Enumeration-Free Exploration and Characterization},
	author       = {Langlade, Lucas and Ferry, Julien and Laberge, Gabriel and Vidal, Thibaut},
	year         = 2025,
	journal      = {arXiv preprint arXiv:2502.05286}
}

@inproceedings{DaiRavishankarYuanNeillBlack2025,
  title={Be intentional about fairness!: Fairness, size, and multiplicity in the {R}ashomon set},
  author={Dai, Gordon and Ravishankar, Pavan and Yuan, Rachel and Black, Emily and Neill, Daniel B},
  booktitle={Proceedings of the 5th ACM Conference on Equity and Access in Algorithms, Mechanisms, and Optimization},
  pages={42--73},
  year={2025}
}

@article{d2022underspecification,
  title={Underspecification presents challenges for credibility in modern machine learning},
  author={D'Amour, Alexander and Heller, Katherine and Moldovan, Dan and Adlam, Ben and Alipanahi, Babak and Beutel, Alex and Chen, Christina and Deaton, Jonathan and Eisenstein, Jacob and Hoffman, Matthew D and others},
  journal={Journal of Machine Learning Research},
  volume={23},
  number={226},
  pages={1--61},
  year={2022}
}

@article{angwin2016machine,
  author   = {Angwin, Julia and Larson, Jeff and Mattu, Surya and Kirchner, Lauren},
  title    = {Machine Bias: There’s Software Used Across the Country to Predict Future Criminals. {A}nd It’s Biased Against Blacks},
  journal  = {ProPublica},
  year     = {2016},
  month    = may,
  note     = {}
}

@article{Gunasekaran2024MAMCR,
  author  = {Gunasekaran, Abirami and Mistry, Pritesh and Chen, Minsi},
  title   = {Which Explanation Should be Selected: A Method Agnostic Model Class Reliance Explanation for Model and Explanation Multiplicity},
  journal = {SN Computer Science},
  volume  = {5},
  pages   = {503},
  year    = {2024},
}

@book{boyd2004convex,
  title={Convex optimization},
  author={Boyd, Stephen P and Vandenberghe, Lieven},
  year={2004},
  publisher={Cambridge university press}
}

@misc{parkinsons_telemonitoring_189,
  author       = {Tsanas, Athanasios and Little, Max},
  title        = {{Parkinsons Telemonitoring}},
  year         = {2009},
  howpublished = {UCI Machine Learning Repository},
}

@misc{wine_quality_186,
  author       = {Cortez, Paulo and Cerdeira, A. and Almeida, F. and Matos, T. and Reis, J.},
  title        = {{Wine Quality}},
  year         = {2009},
  howpublished = {UCI Machine Learning Repository},
}

@misc{banknote_authentication_267,
  author       = {Lohweg, Volker},
  title        = {{Banknote Authentication}},
  year         = {2012},
  howpublished = {UCI Machine Learning Repository},
}

@inproceedings{ganesh2025systemizing,
  title={Systemizing Multiplicity: The Curious Case of Arbitrariness in Machine Learning},
  author={Ganesh, Prakhar and Taik, Afaf and Farnadi, Golnoosh},
  booktitle={Proceedings of the AAAI/ACM Conference on AI, Ethics, and Society},
  volume={8},
  number={2},
  pages={1032--1048},
  year={2025}
}

@article{bao2021s,
  title={It's compaslicated: The messy relationship between rai datasets and algorithmic fairness benchmarks},
  author={Bao, Michelle and Zhou, Angela and Zottola, Samantha and Brubach, Brian and Desmarais, Sarah and Horowitz, Aaron and Lum, Kristian and Venkatasubramanian, Suresh},
  journal={arXiv preprint arXiv:2106.05498},
  year={2021}
}

\newpage
\appendix
\part*{Appendix}
\addcontentsline{toc}{part}{Appendix}

\begingroup
\etocsetnexttocdepth{subsection} 
\localtableofcontents            
\endgroup

\newpage
\label{app:proofs}
\section{Proofs for Theoretical Results in Sections \ref{section:framework} and \ref{section:properties}}
 In this appendix we provide the proof of theoretical results provided in Sections \ref{section:framework} and \ref{section:properties}.

\subsection{Proof for Theorem \ref{th:closed_form_solution}}
\label{app:proof1}

We state and prove Theorem \ref{th:closed_form_solution} below. 

\begingroup
\def\thetheorem{\ref{th:closed_form_solution}}
\begin{theorem}[Closed-form solution]
\TheoremClosedForm
\end{theorem}
\addtocounter{theorem}{-1}
\endgroup

\begin{proof}
Recall that $H$ is positive definite, therefore $H^{1/2}$ is well-defined, symmetric, and invertible.
Let $\bm{v} = H^{1/2}(\bm{\theta} - \hat{\bm{\theta}})$ and $\bm{\xi} = H^{-1/2}\mathbf{x}_c$. Then, we can reformulate our inner optimization problem using $\bm{v}$ and $\bm{\xi}$ as variables.

First, denote $\bm{\theta} - \hat{\bm{\theta}} = H^{-1/2}\bm{v}$, then we can substitute $\bm{\theta} - \hat{\bm{\theta}} = H^{-1/2}\bm{v}$ into the constraint $\frac{1}{2}(\bm{\theta} - \hat{\bm{\theta}})^{\top} H (\bm{\theta} - \hat{\bm{\theta}}) \le \epsilon
$ to get:
\begin{align*}
(H^{-1/2}\bm{v})^{\top} H (H^{-1/2}\bm{v}) &= \bm{v}^{\top}(H^{-1/2})^{\top} H H^{-1/2}\bm{v} \\
&= \bm{v}^{\top}H^{-1/2} H H^{-1/2}\bm{v} \\
&= \bm{v}^{\top}I \bm{v} = \bm{v}^{\top}\bm{v} = \lVert \bm{v} \rVert_2^2.
\end{align*}

Second, recall that $\bm{\xi} = H^{-1/2}\mathbf{x}_c$, then we have that:
\[\bm{\theta}^\top \mathbf{x}_c = (\hat{\bm{\theta}} + H^{-1/2}\bm{v})^\top\mathbf{x}_c = \hat{\bm{\theta}}^\top \mathbf{x}_c + \bm{\xi}^\top \bm{v}.\]
The original optimization problem is now entirely in terms of $\bm{v}$. Since $\hat{\bm{\theta}}^{\top}\mathbf{x}_c$ is a constant with respect to $\bm{v}$, the optimization problem becomes:
\[
\hat{\bm{\theta}}^{\top}\mathbf{x}_c + \min_{\lVert \bm{v} \rVert_2 \le \sqrt{2\epsilon}} \bm{v}^{\top}\bm{\xi}.
\]




Note that we are looking for a minimum of a linear function over a Euclidean ball in terms of $\bm{v}$.
This problem of minimizing a linear function $\bm{a}^{\top}\mathbf{x}$ subject to an $\ell_2$-norm constraint $\lVert \mathbf{x} \rVert_2 \le B$ has a well-known closed-form solution \cite{boyd2004convex}. The optimal value is $-B \lVert \bm{a} \rVert_2$, achieved at $\hat{x} = -B \frac{\bm{a}}{\lVert \bm{a} \rVert_2}$ (for $\bm{a} \neq \bm{0}$). Translating to $\bm{v}$ and $\bm{\xi}$, we get that the optimal value of the $\bm{v}$-minimization is $-\sqrt{2\epsilon} \lVert \bm{\xi} \rVert_2$, achieved at $\hat{\bm{v}} = -\sqrt{2\epsilon} \frac{\bm{\xi}}{\lVert \bm{\xi} \rVert_2}$ (for $\bm{\xi} \neq \bm{0}$).

And translating to the original problem formulation, we get that the overall optimal value is $\hat{\bm{\theta}}^{\top}\mathbf{x}_c - \sqrt{2\epsilon} \lVert H^{-1/2}\mathbf{x}_c \rVert_2$. The optimal $\bm{\theta}_{worst}$ is achieved at:
\[ \bm{\theta}_{worst} = \hat{\bm{\theta}} - \sqrt{2\epsilon}  \left( \frac{H^{-1}\mathbf{x}_c}{\lVert H^{-1/2}\mathbf{x}_c \rVert_2}\right). \]
Note that if $\mathbf{x}_c = \bm{0}$, then the solution is $\bm{\theta}_{worst} = \hat{\bm{\theta}}$.


\end{proof}

The corollary \ref{corollary:check} follows immediately from the above theorem. It allows us to easily check if the counterfactual explanation is robust given the ellipsoid-based Rashomon set, even if this explanation was generated by another method.

\subsection{Proof for Theorem \ref{thm:uniqueness}}\label{app:proof2}

We state and prove Theorem \ref{thm:uniqueness} below. 

\begingroup
\def\thetheorem{\ref{thm:uniqueness}}
\begin{theorem}[Uniqueness]
\TheoremUniqueness
\end{theorem}
\addtocounter{theorem}{-1}
\endgroup

\begin{proof}
The objective function $\|\mathbf{x}_c - \mathbf{x}_0\|_2^2$ is strictly convex as a squared Euclidean norm.

Consider the constraint function $g_{c}(\mathbf{x}_c) = \hat{\bm{\theta}}^{\top} \mathbf{x}_c - \sqrt{2\epsilon\, \mathbf{x}_c^\top H^{-1} \mathbf{x}_c}$.
The first term, $\hat{\bm{\theta}}^{\top} \mathbf{x}_c$, is linear and thus concave.
The second term is convex, since $\sqrt{\mathbf{x}_c^\top H^{-1} \mathbf{x}_c}=\|H^{-1/2}\mathbf{x}_c\|_2$ is a norm, which is convex given that $H$ and $H^{-1}$ is  positive definite. Given that $\epsilon \ge 0$, $g_{c}(\mathbf{x}_c)$ is the sum of a concave function ($\hat{\bm{\theta}}^{\top} \mathbf{x}_c$) and the concave function ($-\sqrt{2\epsilon}\|H^{-1/2}\mathbf{x}_c\|_2$), which means $g_{c}(\mathbf{x}_c)$ is concave. 

Therefore, we get that the feasible set $S = \{ \mathbf{x}_c \mid g_{c}(\mathbf{x}_c) \ge t \}$ is a convex set for a threshold $t$.
Minimizing a strictly convex function over a convex set guarantees that if a solution exists, it is unique.


\end{proof}

\subsection{Proof for Theorem \ref{thm:stability}}
\label{app:proof4}

We prove Theorem \ref{thm:stability} after first proving a helping Lemma about the set of feasible solutions $S$ defined in the proof of Theorem \ref{thm:uniqueness} above. 

\begin{lemma}\label{lemma:feasible_set}
For $H \succ 0$ and $\epsilon \ge 0$, the feasible set $S = \{ \mathbf{x}_c \mid \hat{\bm{\theta}}^{\top} \mathbf{x}_c - \sqrt{2\epsilon\,\mathbf{x}_c^{\top} H^{-1} \mathbf{x}_c} \ge t \}$ is closed and convex. Furthermore, if a solution to the optimization problem~\eqref{eq:final_optimization} exists, $S$ is non-empty.

\end{lemma}
\begin{proof}
 Let $g_{c}(\mathbf{x}_c) = \hat{\bm{\theta}}^{\top} \mathbf{x}_c - \sqrt{2\epsilon\,\mathbf{x}_c^{\top} H^{-1} \mathbf{x}_c}$, then the feasible set is $S = \{ \mathbf{x}_c \mid g_{c}(\mathbf{x}_c) \ge t \}$.
We already showed in the proof of Theorem \ref{thm:uniqueness} that $S$ is a convex set. It must be non-empty, so that a solution to the optimization problem~\eqref{eq:final_optimization} exists.
Therefore, we next focus on showing that $S$ is closed.

Notice that $g_{c}(\mathbf{x}_c)$ is continuous, as it is being composed of differences and compositions of continuous functions. Since $g_{c}(\mathbf{x}_c)$ is continuous, the set $S = \{ \mathbf{x}_c \mid g_{c}(\mathbf{x}_c) \ge t \}$ is closed.   
\end{proof}

Now, we state and prove Theorem \ref{thm:stability}

\begingroup
\def\thetheorem{\ref{thm:stability}}
\begin{theorem}[Stability]
\TheoremStability
\end{theorem}
\addtocounter{theorem}{-1}
\endgroup

\begin{proof}

Let $S = \{ \mathbf{x} \mid \hat{\bm{\theta}}^{\top} \mathbf{x} - \sqrt{2\epsilon\, \mathbf{x}^\top H^{-1} \mathbf{x}} \ge t \}$ denote the feasible set for the robust counterfactual solutions. According to Lemma \ref{lemma:feasible_set}, $S$ is a non-empty, closed, and convex set.

The robust counterfactual solution $\mathbf{x}_c$ corresponding to an input $\mathbf{x}_0$ minimizes $\|\mathbf{x} - \mathbf{x}_0\|_2^2$ for $\mathbf{x} \in S$. Thus, $\mathbf{x}_c$ is the Euclidean projection of $\mathbf{x}_0$ onto $S$. Let $P_S(\cdot)$ denote this projection operator. Then, we have:
\[ \mathbf{x}_c = P_S(\mathbf{x}_0), \quad \mathbf{x}_c' = P_S(\mathbf{x}_0').
\]
The Euclidean projection $P_S$ onto a non-empty, closed, convex set $S$ is 1-Lipschitz continuous. Given that $\|\mathbf{x}_0 - \mathbf{x}_0'\|_2 = \|\bm{\delta}\|_2$ since $\mathbf{x}_0' = \mathbf{x}_0 + \bm{\delta}$, we get:
\begin{align*}
 \|\mathbf{x}_c - \mathbf{x}_c'\|_2 & = \|P_S(\mathbf{x}_0) - P_S(\mathbf{x}_0')\|_2 \\
 &\leq \|\mathbf{x}_0 - \mathbf{x}_0'\|_2 = \|\bm{\delta}\|_2.
\end{align*}

Therefore, we obtain: $\|\mathbf{x}_c - \mathbf{x}_c'\|_2 \le \|\bm{\delta}\|_2.$











\end{proof}

Next we focus on proving that our counterfactual solutions are aligned with directions of the important features.


\subsection{Proof for Theorem \ref{thm:VI}}\label{app:proof5}

We state and prove Theorem \ref{thm:VI} below.

\begingroup
\def\thetheorem{\ref{thm:VI}}
\begin{theorem}[Alignment with Important Feature Directions.]
\TheoremAlignment
\end{theorem}
\addtocounter{theorem}{-1}
\endgroup

\begin{proof}

When $\mathbf{x}_c = \bm{0}$, the conclusion of the theorem are satisfied in the trivial sense. Assume that $\mathbf{x}_c \neq \bm{0}$.
Minimizing $C_{rob}(\epsilon, \mathbf{x}_c)$ for $\epsilon > 0$ is equivalent to minimizing $\mathbf{x}_c^\top H^{-1} \mathbf{x}_c$.

Since $H$ is symmetric positive definite, its  decomposition is $H = U \Lambda U^{\top}$, where $U$ is an orthogonal matrix, such that its columns ($\mathbf{q}_1, \dots, \mathbf{q}_d$) are the eigenvectors of $H$, and $\Lambda = \operatorname{diag}(\lambda_1, \lambda_2, \dots, \lambda_d)$ is the matrix with eigenvalues on the diagonal. 
Then $H^{-1} = U \Lambda^{-1} U^{\top}$, where $\Lambda^{-1} = \operatorname{diag}(1/\lambda_1, 1/\lambda_2, \dots, 1/\lambda_d)$. The smallest eigenvalue of $H^{-1}$ is $1/\lambda_1$, since by assumption of the theorem $\lambda_1 > \lambda_2 \ge \cdots \ge \lambda_d > 0$.

Let $\bm{\xi} = U^{\top}\mathbf{x}_c$. Then $\mathbf{x}_c = U\bm{\xi}$, and $\lVert \bm{\xi} \rVert_2 = \lVert U^{\top}\mathbf{x}_c \rVert_2 = \lVert \mathbf{x}_c \rVert_2$ since $U$ is orthogonal. We are minimizing the following quadratic form:
\[ \mathbf{x}_c^{\top}H^{-1}\mathbf{x}_c = (U\bm{\xi})^{\top}(U\Lambda^{-1}U^{\top})(U\bm{\xi}) = \bm{\xi}^{\top}\Lambda^{-1}\bm{\xi} = \sum_{j=1}^d \frac{\xi_j^2}{\lambda_j}. \]
To find the optimal for $\bm{\xi}$, we analyze which direction minimizes this sum for any given positive norm. Thus, we add constraints that $\lVert \bm{\xi} \rVert_2^2$ is fixed and positive, $\lVert \bm{\xi} \rVert_2^2= a^2 > 0$, and search for minimizing direction among all vectors with the fixed norm $a$.
Since $1/\lambda_1 < 1/\lambda_j$ for $j \neq 1$, the sum $\sum_{j=1}^d \xi_j^2/\lambda_j$ is minimized when the whole mass of the squared norm $\lVert \bm{\xi} \rVert_2^2$ is placed on the component that corresponds to the smallest coefficient $1/\lambda_1$. Therefore,  the minimizing $\bm{\xi}$ is $(\pm \lVert \bm{\xi} \rVert_2, 0, \dots, 0)^{\top}$ in the basis of $U$.
Transforming back to $\mathbf{x}_c$:
\[ \mathbf{x}_c = U\bm{\xi} = \xi_1\mathbf{q}_1 + \sum_{i=2}^d \xi_i\mathbf{q}_i = (\pm \lVert \bm{\xi} \rVert_2)\mathbf{q}_1 = (\pm \lVert \mathbf{x}_c \rVert_2) \mathbf{q}_1.\]
Thus, for any fixed non-zero norm, $\mathbf{x}_c^\top H^{-1}\mathbf{x}_c$ is minimized when $\mathbf{x}_c$ is aligned with the eigenvector $\mathbf{q}_1$.


\end{proof}

\subsection{Proof for Theorem \ref{thm:tradeoff}}\label{app:proof3}

The last property we consider is that the counterfactuals generated by ElliCE have a robustness-proximity trade-off. This property has also been observed by other works \cite{ForelParmentierVidal2022} and is natural for counterfactual generations. ElliCE allows us to find the smallest length counterfactuals that are robust to all models in the Rashomon set.
We next state Theorem \ref{thm:tradeoff} and then prove it.

\begingroup
\def\thetheorem{\ref{thm:tradeoff}}
\begin{theorem}[Robustness-Proximity Trade-off.]
\TheoremTradeoff
\end{theorem}
\addtocounter{theorem}{-1}
\endgroup

\begin{proof}
For every robustness level $\epsilon>0$ define the feasible region  
\[
   \mathcal C(\epsilon)\;:=\;\Bigl\{\mathbf{x}\in\R^{d}\;\Bigm|\;
         \hat{\bm{\theta}}^{\top}\mathbf{x}-\sqrt{2\epsilon\,\mathbf{x}^{\top}H^{-1}\mathbf{x}}\;\ge\;t\Bigr\}.
\]
Recall that $\mathbf{x}_c^*(\epsilon)$ is the {\em unique} minimiser of the strictly–convex
objective $Obj(\mathbf{x}) = \min_{\mathbf{x}}\|\mathbf{x}-\mathbf{x}_0\|_2^{2}$ over~$\mathcal C(\epsilon)$
as we proved in  Theorem~\ref{thm:uniqueness}.  
Based on the theorem notations, $\nu(\epsilon)=\|\mathbf{x}_c^*(\epsilon)-\mathbf{x}_0\|_2^{2}$.

If $0<\epsilon_1<\epsilon_2$ then for every $\mathbf{x}\in\R^{d}$:
$ -\sqrt{2\epsilon_1\,\mathbf{x}^{\top}H^{-1}\mathbf{x}}\;\;\ge\;\;
   -\sqrt{2\epsilon_2\,\mathbf{x}^{\top}H^{-1}\mathbf{x}}
$.
Therefore for any $\mathbf{x} \in \mathcal C(\epsilon_2)$:
\[\hat{\bm{\theta}}^{\top}\mathbf{x}-\sqrt{2\epsilon_1\,\mathbf{x}^{\top}H^{-1}\mathbf{x}} \geq\hat{\bm{\theta}}^{\top}\mathbf{x}-\sqrt{2\epsilon_2\,\mathbf{x}^{\top}H^{-1}\mathbf{x}}\ge t.\]
This means that sets $\mathcal C(\epsilon_2)$ and $\mathcal C(\epsilon_1)$ are enclosed:
\[
   \mathcal C(\epsilon_2)\subseteq\mathcal C(\epsilon_1).
\]

Then for $\mathbf{x}_c^*(\epsilon_2)\in\mathcal C(\epsilon_1)$ and $\mathbf{x}_c^*(\epsilon_1)$ in $\mathcal C(\epsilon_1)$ we get that $\nu$ is non-decreasing:
\[
   \nu(\epsilon_1)\;=\;\|\mathbf{x}_0-\mathbf{x}_c^*(\epsilon_1)\|_2^{2}
   \;\le\;\|\mathbf{x}_0-\mathbf{x}_c^*(\epsilon_2)\|_2^{2}\;=\;\nu(\epsilon_2).
\]

To show strict increase we will assume  contradiction. More specifically, we assume  that $\nu(\epsilon_1)=\nu(\epsilon_2)$.
Then both $\mathbf{x}_c^*(\epsilon_1)$ and $\mathbf{x}_c^*(\epsilon_2)$ minimize $Obj(\mathbf{x})$
over $\mathcal C(\epsilon_1)$.  Uniqueness 
 of the solution means that it is possible only when:
\[
   \mathbf{x}_c^*(\epsilon_1)=\mathbf{x}_c^*(\epsilon_2)=:\mathbf{x}^{\star}\neq\mathbf0.
\]
Because $\nu(\epsilon_1)>0$, we have $\mathbf{x}_0\notin \mathcal C(\epsilon_1)$.  Therefore the minimizer $\mathbf{x}^\star$ of
$\min_{\mathbf{x}\in\mathcal C(\epsilon_1)} \|\mathbf{x}-\mathbf{x}_0\|_2$ cannot be $\mathbf{x}_0$ and must lie on the boundary of $\mathcal C(\epsilon_1)$, i.e.,
\[
\hat{\bm{\theta}}^{\top}\mathbf{x}^\star-\sqrt{2\epsilon_1}\,\|H^{-1/2}\mathbf{x}^\star\|_2 = t.
\]

Now suppose for contradiction that $\nu(\epsilon_2)=\nu(\epsilon_1)$ with $\epsilon_2>\epsilon_1$.
Since $\mathcal C(\epsilon_2)\subseteq \mathcal C(\epsilon_1)$, the point $\mathbf{x}^\star$ is feasible for $\mathcal C(\epsilon_1)$, and
$\nu(\epsilon_2)=\nu(\epsilon_1)$ implies that $\mathbf{x}^\star$ also attains the minimum distance over $\mathcal C(\epsilon_2)$.
By uniqueness of the projection (Theorem~\ref{thm:uniqueness}), it follows that the minimizers coincide, i.e.,
$\mathbf{x}^\star(\epsilon_2)=\mathbf{x}^\star(\epsilon_1)=\mathbf{x}^\star$.
In particular, $\mathbf{x}^\star$ must be feasible for $\mathcal C(\epsilon_2)$, which gives the inequality
\[
\hat{\theta}^{\top}\mathbf{x}^\star-\sqrt{2\epsilon_2}\,\|H^{-1/2}\mathbf{x}^\star\|_2 \ge t.
\]
Combining this with the boundary equality for $\epsilon_1$ forces $\epsilon_1=\epsilon_2$, which is a contradiction.

\end{proof}

\begin{table}[t]
\centering
\small
\caption{Datasets description and pre-processing notes. In Comments, we provide original dimension of the dataset before any edits (\#samples $\times$ \#features).}\label{tab:datasets}
\label{tab:preprocessing_notes}
\begin{tabular}{p{0.1\textwidth}p{0.05\textwidth}p{0.05\textwidth}p{0.5\textwidth}p{0.05\textwidth}p{0.1\textwidth}}
\toprule
\textbf{Dataset Name} 
  & \textbf{\# Samples} 
  & \textbf{\# Features} 
  & \textbf{Preprocessing Notes} 
  & \textbf{Source}
  & \textbf{Comments}\\
\midrule

Parkinson’s 
  & 5872 
  & 18 
  & 
      \textit{motor\_UPDRS} is considered as the target, that was binarized by cutting at the median value.
      Dropped subject\#, test\_time, raw UPDRS columns.
      Under-sampled to balance the dataset.
      Shuffled data.
  & \cite{parkinsons_telemonitoring_189}
  & Orig. $5 875 \times 21$ \\
\midrule

Wine Quality 
  & 2554 
  & 11 
  & 
      Binarized target \textit{quality} at the median value.
      Dropped \textit{type} and raw \textit{quality}.
      Under-sampled to balance dataset,
      Shuffled data.
  & \cite{wine_quality_186}
  & Orig.$ 6 497 \times 12$ \\
\midrule
German Credit 
  & 600 
  & 61 
  & 
      Performed one‐hot encoding for all categorical features. Under-sampled to balance the dataset.
      Shuffled data.
  & \cite{statlog_144}
  & Orig. $ 1000 \times 20$ \\
\midrule
Extended Iris 
  & 800 
  & 4 
  & 
      Converted to binary classification:  $1$, if Iris-setosa, $0$ otherwise.  Under-sampled to balance the dataset.
      Shuffled data. Only kept features present in the original Iris dataset.
  & \cite{samy_baladram_2023}
  & Orig.$ 1 200 \times 67$ \\
\midrule

Banknote 
  & 1220 
  & 4 
  & 
      Selected variance, skewness, curtosis, entropy as features. Under-sampled to balance the dataset.
      Shuffled data.
  & \cite{banknote_authentication_267}
  & Orig. $1 372 \times 4$ \\
\midrule
Australian Credit 
  & 530 
  & 18 
  & Dropped the first column (unique identifier). Under-sampled to balance the dataset.
      Shuffled data.
  & \cite{statlog_143}
  & Orig. $640 \times 19$ \\
\midrule
FICO 
  & 9470 
  & 20 
  & 
  Dropped 3 features with the high amount of missing values (MSinceMostRecentDelq, NetFractionInstallBurden, MSinceMostRecentInqexcl7days) and removed rows with $>$40\% missingness among remaining features. Filled remaining missing values with -1. Under-sampled to balance the dataset. Shuffled data.
  & \cite{fico2018heloc}
  & Orig. $10 459 \times 23$ \\
\midrule
COMPAS 
  & 6392
  & 7 
  &  Under-sampled to balance the dataset.
      Shuffled data.
  & \cite{bao2021s}
  & Orig. $6 907 \times 7$ \\
\midrule
Diabetes 
  & 536 
  & 8 
  &  
  Under-sampled to balance the dataset.
      Shuffled data.
  & \cite{smith1988pima}
  & Orig. $768 \times 8$ \\

\bottomrule
\end{tabular}
\end{table}

\section{Additional Experiments}\label{appendix:experiments}

In this Appendix, we provide more experimental results and show how ElliCE performs under different optimization schemes and compared to different baselines.

\subsection{Datasets}

Please see Table \ref{tab:datasets} for the summary of all datasets used in the paper and pre-processing notes, if any. For each dataset and each fold, we fit a StandardScaler on the training set, computing the per-feature mean and variance, and then apply those parameters to both the validation and test splits. This ensures that each continuous feature is centered at zero mean and scaled to unit variance based on the training data. One-hot encoded and binary features are left untouched.

\subsection{Empirical Rashomon Set Construction}\label{appendix:emp_rash_c}

Each evaluator set of models was constructed as follows. 
\textbf{Retrain models} were obtained by independently retraining the base model from random initialization using different random seeds on the same training and validation data. 
Models whose training loss remained within the Rashomon bound were retained to form the Rashomon set. 
\textbf{Dropout models} were generated by applying Gaussian dropout noise directly to the weights of the trained base model. 
The dropout variance hyperparameter was first tuned to approximate the target Rashomon bound by finding the maximum dropout value such that at least $5\%$ of the models sampled using it remained within the Rashomon set (which in practice, due to the discreteness of the dropout search space, typically ranged between $10\%$ and $15\%$). 
During model sampling, if a specific perturbed model exceeded the Rashomon bound, its dropout variance was reduced by a factor of $2$; if the resulting model satisfied the Rashomon constraint, it was retained. 
\textbf{AWP models} were produced using Adversarial Weight Perturbation (AWP), where the base model’s parameters were iteratively updated with adversarial perturbations computed on the training data. 
We stopped the perturbations once the training loss exceeded the Rashomon threshold and retained the model weights from the previous step, ensuring that all generated models remained valid members of the Rashomon set.

\subsection{Computation Resources}

 Our experiments were conducted on Apple M2/M2 Max MacBooks (16-32 GB RAM) for development and Google Cloud Platform C4 VMs (4-16 vCPUs, 16-64 GB RAM) for production runs.
We performed parallelization using multiprocessing for cross-validation folds.
 Our complete pipeline (5 datasets, 2 model types) required $\sim 64$ hours sequential execution on a single CPU.
We used fixed random seeds for reproducibility. 



\begin{algorithm}[t]
\caption{Continuous ElliCE (Preparation \& CE generation)}\label{alg1}
\begin{small}
\begin{algorithmic}[1]
  \State \textbf{Input:} $\{(\textbf{x}_i,y_i)\}_{i=1}^n$, baseline model $f_{\bm{\theta}}$, Rashomon budget $\epsilon$, regularization $\lambda$, initial \ $\mathbf{x}_0$
  \Statex\textbf{Preparation:}
  \State $E_i\gets\phi(\textbf{x}_i)$,\quad $\tilde H\gets[E,\mathbf1]$
  \State Optionally refit last layer $\hat{\bm{\theta}}\gets\mathrm{LR}(\tilde H,y)$
  \State $\omega_c\gets\hat{\bm{\theta}}$, $s\gets\tilde H\,\omega_c$, $p\gets\sigma(s)$
  \State $W\gets\mathrm{diag}(p\odot(1-p))$
  \State $H\gets\tfrac1n\,\tilde H^\top W\,\tilde H + \lambda I$
  \Statex\textbf{Generation:}
  \State Initialize: $\mathbf{x}_c \leftarrow \mathbf{x}_0$, \textit{steps} $\leftarrow 0$
  \While{$\textit{steps}<T$}
    \State $h\gets\phi(\mathbf{x}_c)$,\quad $\tilde h\gets[h^\top,1]^\top$
    \State \emph{Worst‐case model (Theorem \ref{th:closed_form_solution}):}
    \State \quad
      $\displaystyle
         \bm{\theta}_{w}
         = \omega_c
           - \sqrt{2\epsilon}\,
             \frac{H^{-1}\tilde h}
                  {\sqrt{\tilde h^{\top}H^{-1}\tilde h}}$
    \State Prediction loss:
      $\ell_{\mathrm{pred}}
       = \mathrm{BCE}(\tilde h^\top\omega_c,1)$
    \State Robustness loss:
      $\ell_{\mathrm{rob}}
       = \mathrm{BCE}(\tilde h^\top\bm{\theta}_w,1)$
    \State Proximity loss:
      $\ell_{\mathrm{prox}}
       = \|\mathbf{x}_c - \mathbf{x}_0\|_2^2$
    \State Sparsity loss:
      $\ell_{\mathrm{sparse}}
       = \|\mathbf{x}_c - \mathbf{x}_0\|_1$
    \State Total objective:
      $\displaystyle
        \mathcal{L}
        = \alpha\,\ell_{\mathrm{pred}}
        + \beta\,\ell_{\mathrm{rob}}
        + \lambda\,\ell_{\mathrm{prox}}
        + \gamma\,\ell_{\mathrm{sparse}}
      $
    \State Gradient update:
      $\mathbf{x}_c \gets \mathbf{x}_c - \eta\,\nabla_{\mathbf{x}_c}\mathcal{L}$
    \If{$\tilde h^\top\bm{\theta}_w\ge0$}
      \State \textbf{break}
    \EndIf
    \State \textit{steps} $\leftarrow$ \textit{steps}+1
  \EndWhile
  \State \textbf{Output:} robust counterfactual $\mathbf{x}_c$
\end{algorithmic}
\end{small}
\end{algorithm}

\begin{algorithm}[t]
\caption{Continuous ElliCE (CE generation)}\label{alg2}
\begin{small}
\begin{algorithmic}[1]
  \State \textbf{Input:} factual $\mathbf{x}_0$, central weights $\boldsymbol\omega_c$, Hessian $H$, Rashomon radius $\epsilon$, learning-rate $\eta$, max\_steps $T$, coefficients $(\alpha,\beta,\lambda,\gamma)$
  \State \textbf{Initialize:} $\mathbf{x}_c \leftarrow \mathbf{x}_0$, \textit{steps} $\leftarrow 0$
  \While{$\textit{steps}<T$}
    \State Compute penultimate features $h \leftarrow \phi(\mathbf{x}_c)$, augmented $\tilde h \leftarrow [h^{\top},1]^{\top}$
    \State \emph{Worst‐case model (Thm.\!1):}
    \State \quad
      $\displaystyle
         \bm{\theta}_{w}
         = \boldsymbol\omega_c
           - \sqrt{2\epsilon}\,
             \frac{H^{-1}\tilde h}
                  {\sqrt{\tilde h^{\top}H^{-1}\tilde h}}$
    \State Prediction loss 
      $\ell_{\mathrm{pred}} = \mathrm{BCE}(\tilde h^{\top}\boldsymbol\omega_c,1)$
    \State Robustness loss 
      $\ell_{\mathrm{rob}} = \mathrm{BCE}(\tilde h^{\top}\bm{\theta}_w,1)$
    \State Proximity loss 
      $\ell_{\mathrm{prox}} = \|\mathbf{x}_c - \mathbf{x}_0\|_2^2$
    \State Sparsity loss 
      $\ell_{\mathrm{sparse}} = \|\mathbf{x}_c - \mathbf{x}_0\|_1$
    \State Total objective 
      $\mathcal{L} = \alpha\,\ell_{\mathrm{pred}}
                       + \beta\,\ell_{\mathrm{rob}}
                       + \lambda\,\ell_{\mathrm{prox}}
                       + \gamma\,\ell_{\mathrm{sparse}}$
    \State Gradient update 
      $\mathbf{x}_c \leftarrow \mathbf{x}_c - \eta\,\nabla_{\mathbf{x}_c}\mathcal{L}$
    \If{$\tilde h^{\top}\bm{\theta}_w \ge 0$}
      \State \textbf{break}
    \EndIf
    \State \textit{steps} $\leftarrow$ \textit{steps}+1
  \EndWhile
  \State \textbf{Output:} robust counterfactual $\mathbf{x}_c$
\end{algorithmic}
\end{small}
\end{algorithm}

\subsection{Optimization}\label{appendix:optimization}

For continuous scenarios where counterfactuals are not required to be on the data manifold, we implemented a continuous optimization approach described in Algorithm \ref{alg1} and \ref{alg2}. This gradient-based method directly optimizes in the input space by iteratively computing the worst-case model using Theorem~\ref{th:closed_form_solution} and updating the counterfactual through gradient descent on a multi-objective loss combining prediction, robustness, proximity, and sparsity terms.


\begin{figure}[h]
    \centering
\includegraphics[width=1\textwidth]{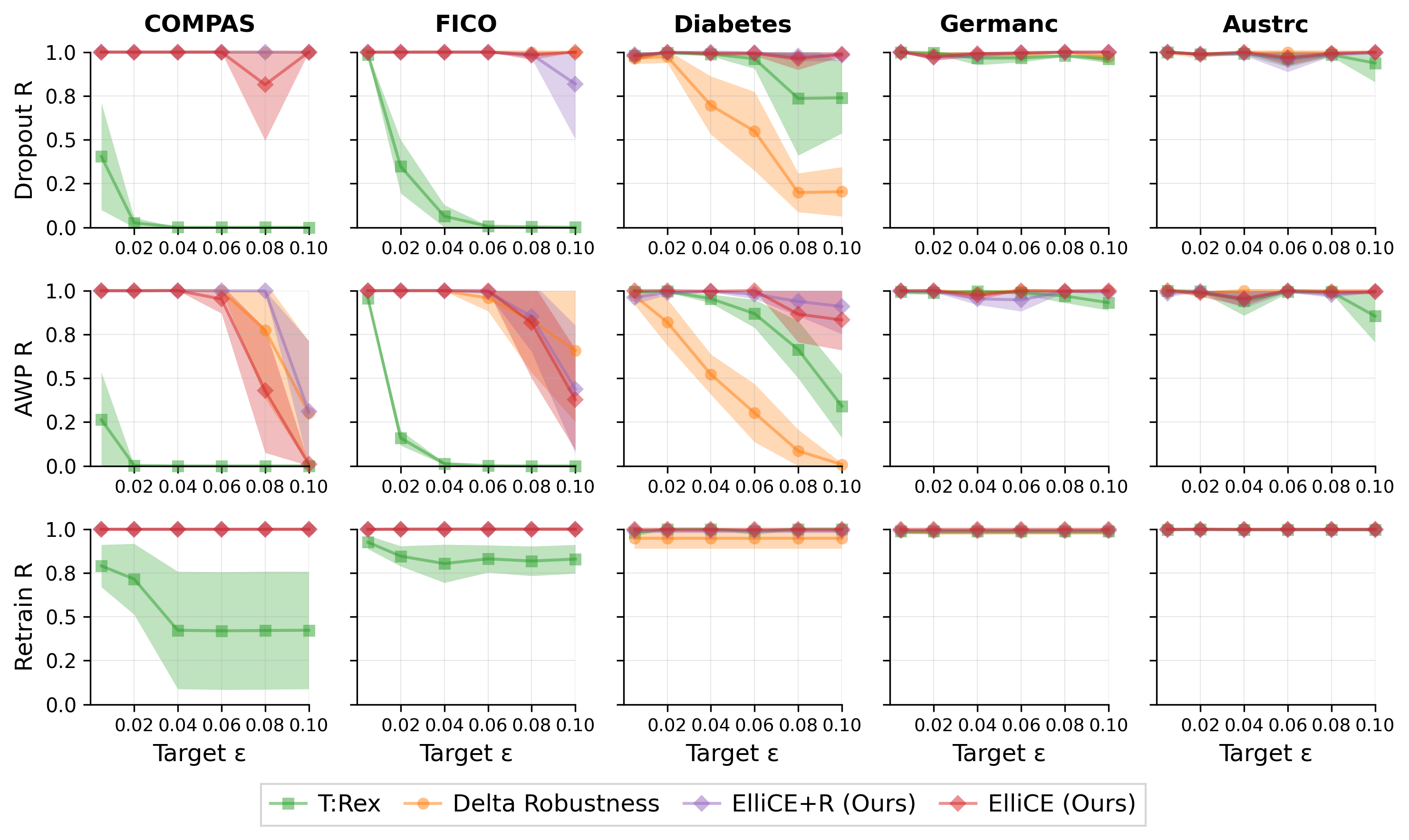}
    \caption{Robustness evaluation of ElliCE against baselines on MLPs using data-supported generation across all datasets.
}
    \label{fig:extension_rob_mlp}
\end{figure}

\begin{figure}[t]
    \centering
\includegraphics[width=1\textwidth]{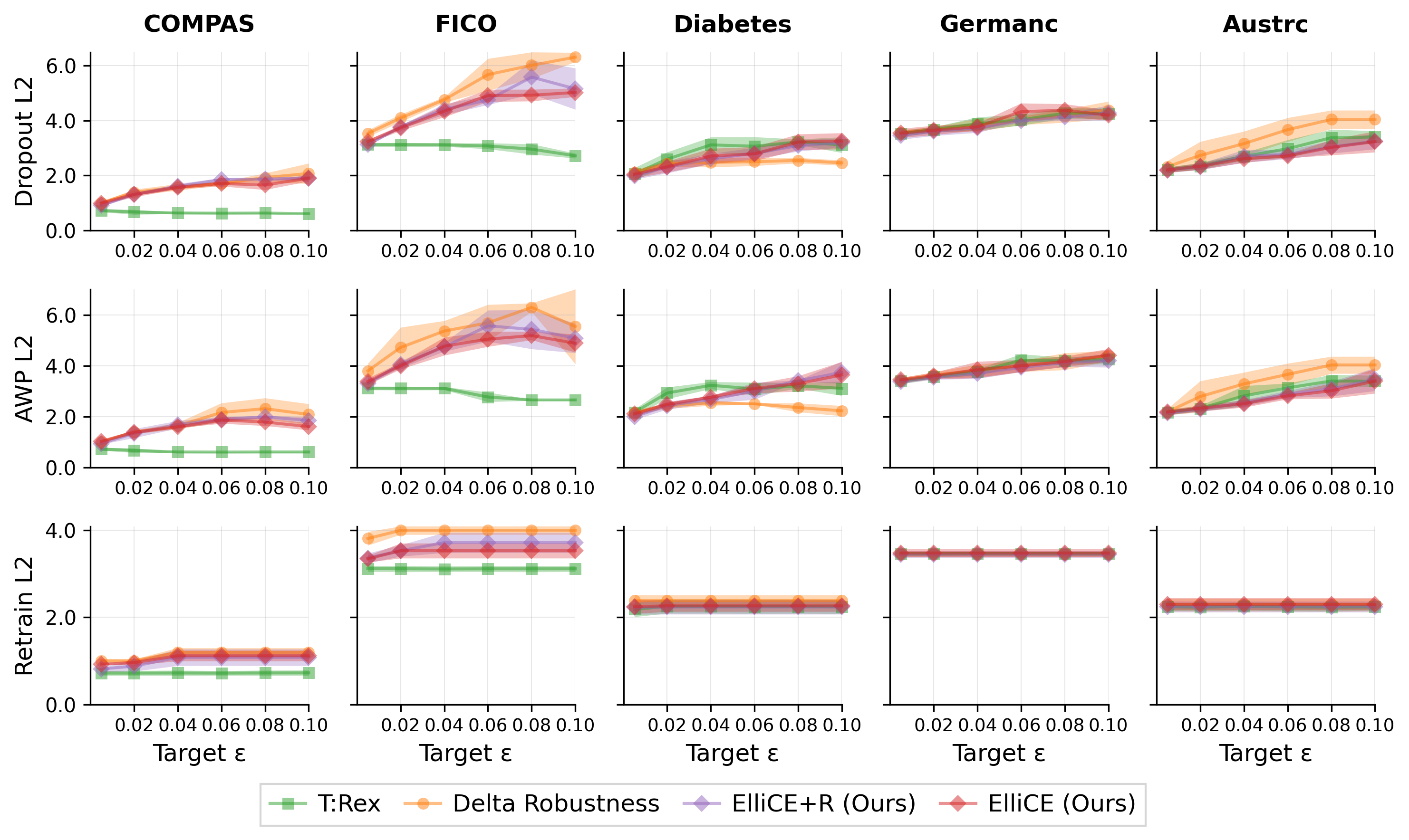}
    \caption{Length evaluation of ElliCE against baselines on MLPs using data-supported generation across all datasets.
}
    \label{fig:extension_length_mlp}
\end{figure}

\begin{figure}[t]
    \centering
\includegraphics[width=1\textwidth]{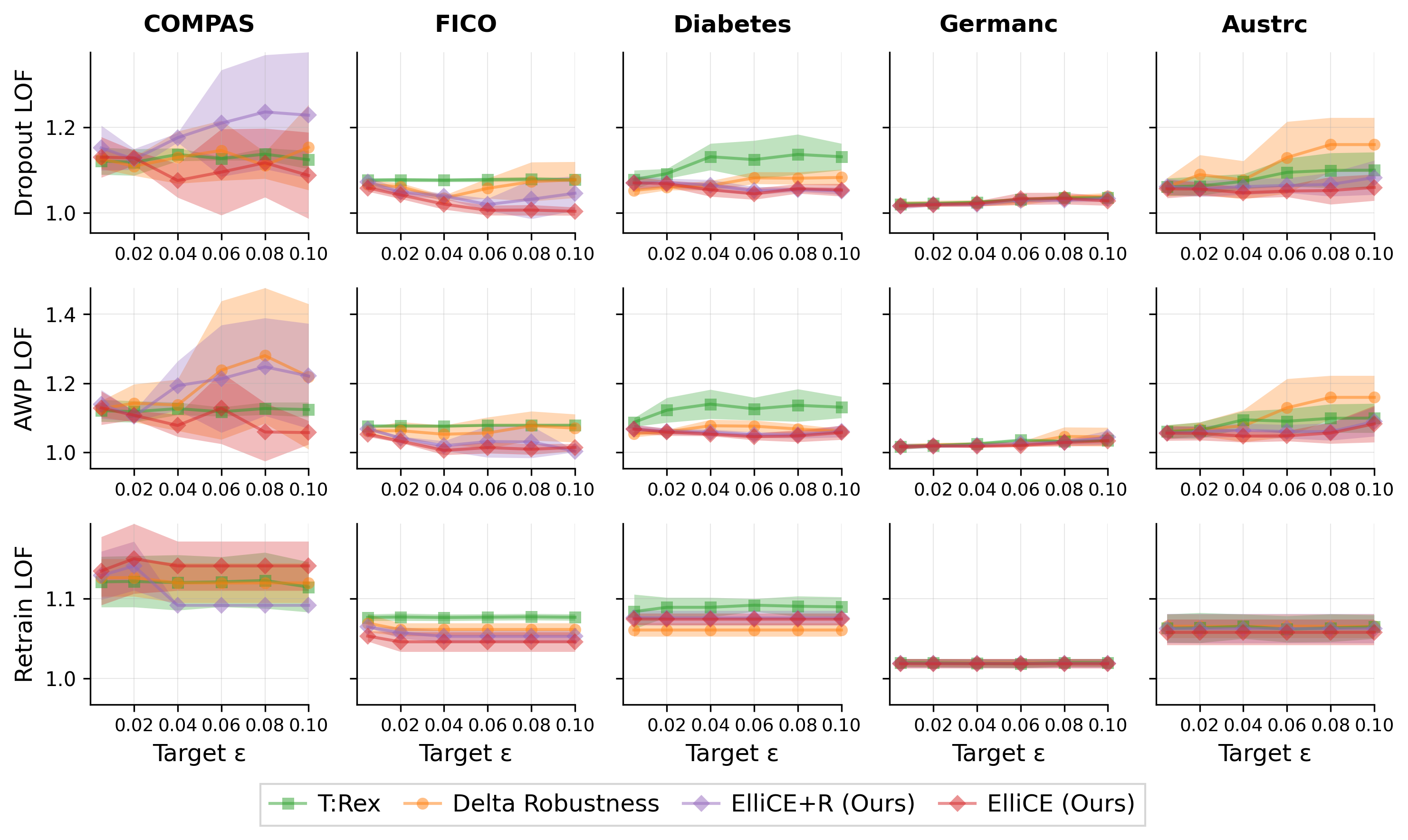}
    \caption{LOF evaluation of ElliCE against baselines on MLPs using data-supported generation across all datasets.
}
    \label{fig:lof_ds_mlp}
\end{figure}

\begin{figure}[t]
    \centering
\includegraphics[width=1\textwidth]{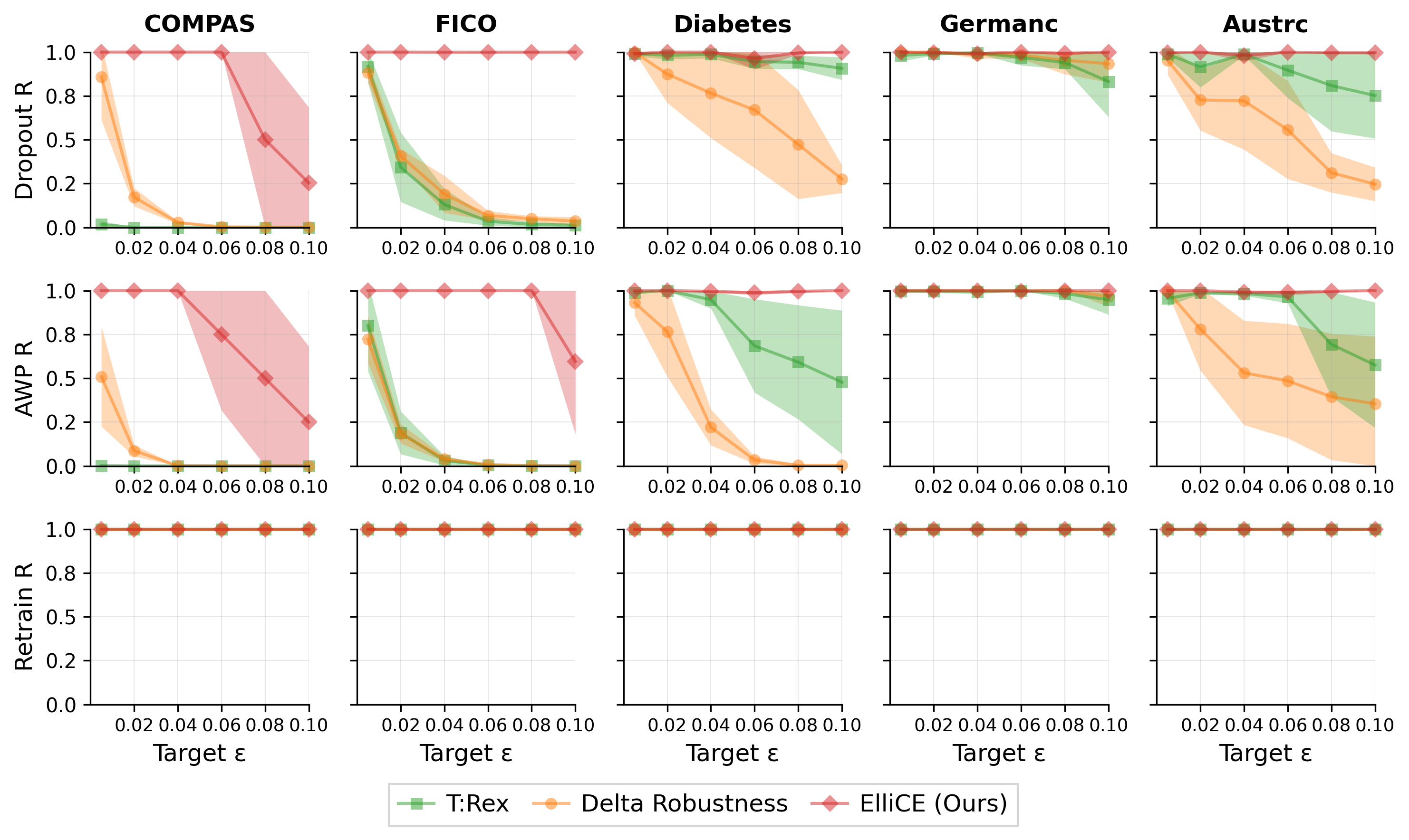}
    \caption{Robustness evaluation of ElliCE against baselines on linear models using data-supported generation across all datasets.
}
    \label{fig:extension_rob_linear}
\end{figure}

\begin{figure}[t]
    \centering
\includegraphics[width=1\textwidth]{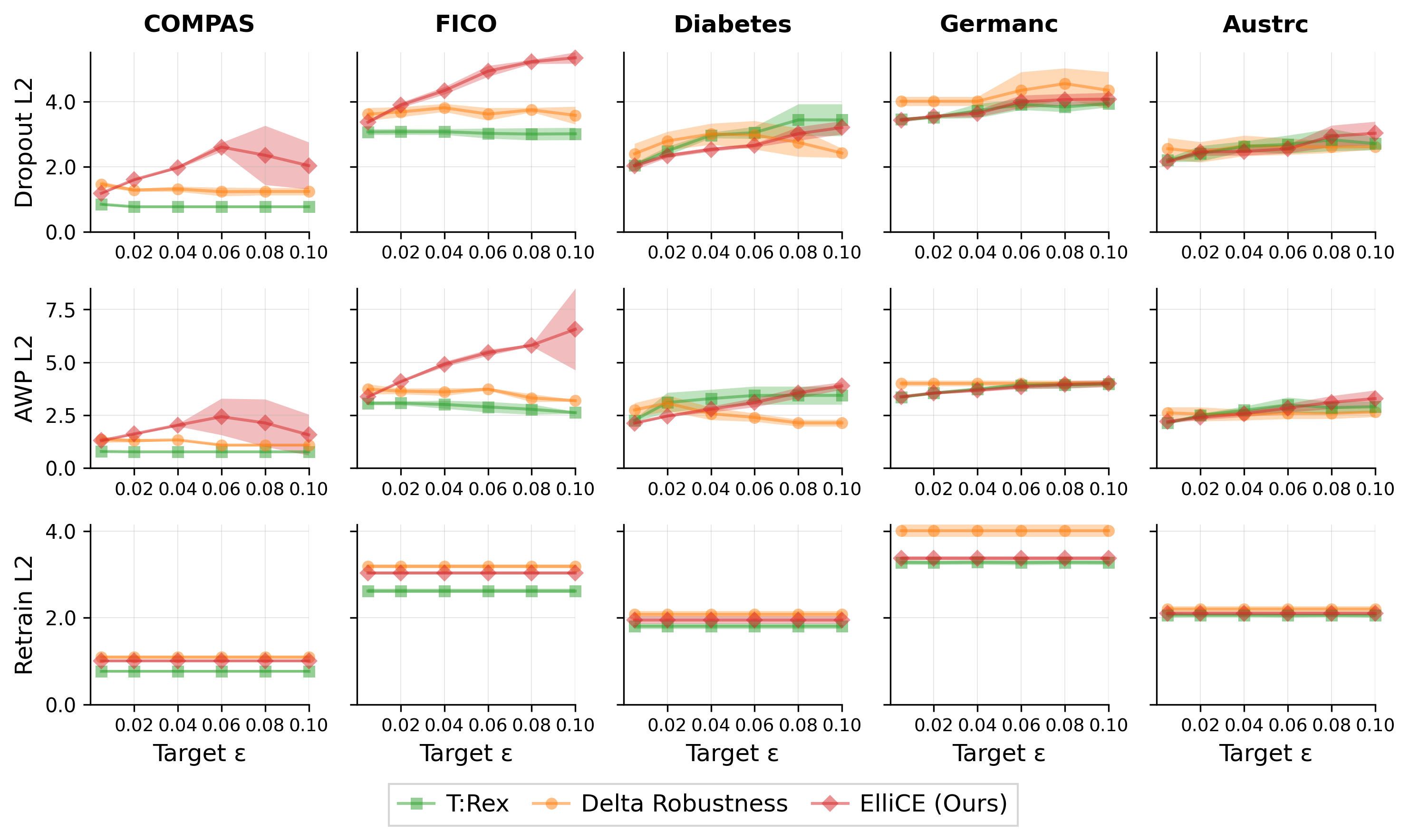}
    \caption{Length evaluation of ElliCE against baselines on linear models using data-supported generation across all datasets.
}
    \label{fig:extension_length_linear}
\end{figure}

\begin{figure}[t]
    \centering
\includegraphics[width=1\textwidth]{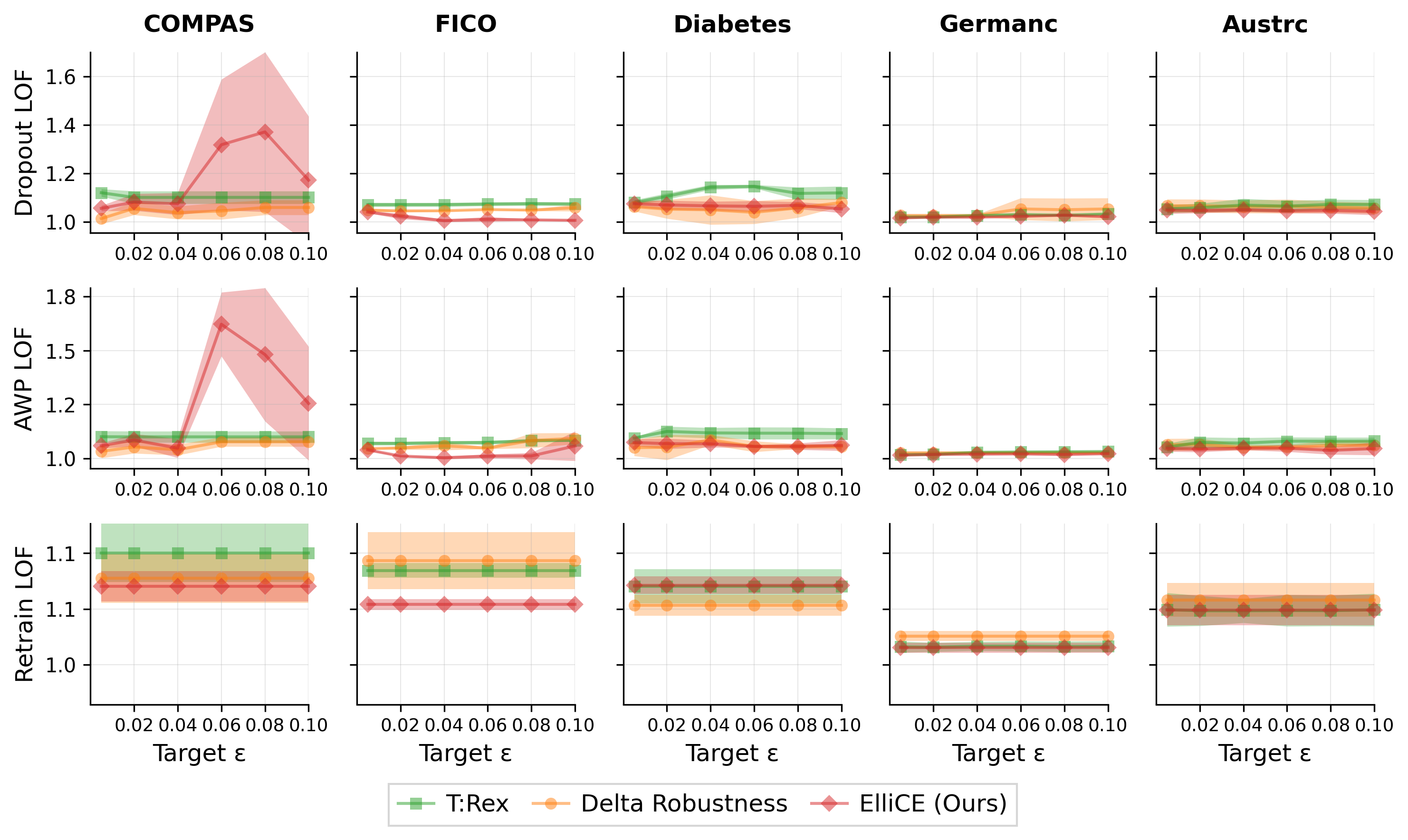}
    \caption{LOF evaluation of ElliCE against baselines on linear models using data-supported generation across all datasets.
}
    \label{fig:lof_ds_lin}
\end{figure}

\subsection{Hyperparameter Tuning}

For hyperparameter tuning, we chose the best hyperparameters for each method, targeting validity first and then robustness. We observed that ElliCE, TRex, and ROAR behaved more consistently across different model types and datasets, not requiring significant adjustments to their hyperparameter search ranges. In contrast, Delta Robustness and PROPLACE required disproportionately different search ranges depending on the model, primarily because different model architectures (linear vs. neural networks) operate with vastly different parameter scales and amplitudes, requiring delta-based perturbation methods to adjust their sensitivity accordingly.

For linear models, the search range for Delta Robustness was $[0.2, 3.0]$ with a step size of $0.1$ to $0.2$ (dataset-dependent), and for PROPLACE, the range was $[0.2, 6.5]$ with a step size of $0.2$. MLPs required much smaller amplitudes: Delta Robustness used a range of $[0.002, 0.1]$ with a step of $0.004$, and PROPLACE used $[0.02, 0.5]$ with a step of $0.02$. Interestingly, for an MLP with hidden layers $[32, 32]$ on the Parkinsons dataset, PROPLACE began to lose validity and failed to find counterfactuals (CEs) for $\delta > 0.001$. Consequently, its search range was reduced to $[0.0001, 0.001]$ with a step of $0.0004$. Even within this adjusted range, PROPLACE sometimes failed to converge in a reasonable time when generating CEs for certain inputs $\mathbf{x}$. Thus, we imposed a hard time limit of $30$ seconds per CE generation.

For both TRex and TRexI, we tuned the threshold $\tau$ in the range $[0.375, 1.0]$ with a step of $0.05$ to $0.1$. For ElliCE and Continuous ElliCE, the respective search ranges were $[0.00125 \text{--} 0.0025, 0.15 \text{--} 0.2]$ with step sizes of $0.0025 \text{--} 0.01$, and $[0.0025, 0.2]$ with step sizes of $0.0025 \text{--} 0.01$. For ROAR, we used ranges of $[0.01 \text{--} 0.6]$ with a step of $0.06$.

Additionally, we applied early stopping for all methods upon achieving a robustness score of $1$, which significantly reduced running times for some methods.


\subsection{The Choice of $\epsilon$ Parameter for ElliCE}
The choice of the robustness parameter $\epsilon$ is a key consideration when generating recourse. ElliCE is designed to produce explanations robust to the Rashomon set, offering a basis for selecting $\epsilon$.

Empirical evidence (as seen in Figures \ref{fig:extension_rob_linear}, \ref{fig:extension_rob_mlp}, \ref{fig:cont_rob_linear}, and \ref{fig:cont_rob_nn}) indicates that models produced by random retraining often represent a smaller, less diverse subset of the  Rashomon set, making them easier to be robust against. ElliCE frequently achieves perfect robustness to such retrained models with relatively small $\epsilon$ values.

Ideally, if $\varepsilon_\text{target}$ (the Rashomon parameter of the evaluator) is known, we recommend choosing $\epsilon \ge \varepsilon_\text{target}$. Selecting a value slightly higher, such as $\epsilon = \varepsilon_\text{target} + 0.01$, can provide additional confidence by accounting for broader model multiplicity beyond the local estimates ElliCE focuses on. This helps prevent under-exploration of the Rashomon set. 

If a precise $\varepsilon_\text{target}$ value is unavailable from an evaluator or prior analysis, we recommend using model performance metrics as a practical way to gauge an appropriate $\epsilon$. For example, one might set a target based on acceptable performance deviations. Empirically, rules of thumb such as choosing $\epsilon$ that corresponds to a 10\% increase above the original model's loss, or an $\epsilon$ that captures models within a 5\% accuracy deviation, often work well as starting points. The goal is to select an $\epsilon$ that reflects a meaningful degree of model variability relevant to the specific application.

\subsection{Experiments for Data-Supported Counterfactual Generation}

This section extends results from Section~\ref{section:experiments}. Here, we provide plots for all datasets and hypothesis spaces.
More specifically, robustness and length results for linear models are shown in Figure \ref{fig:extension_rob_linear} and Figure \ref{fig:extension_length_linear} and for multi-layer perceptrons in Figure \ref{fig:extension_rob_mlp} and Figure \ref{fig:extension_length_mlp}.  Table \ref{tab:performance_comparison} contains a snapshot of figures when target epsilon is 10\% of the optimal loss on the train data. The robustness-proximity tradeoff plots are in Figure \ref{fig:rob_len_tradeoff_lin} for linear models and Figure \ref{fig:rob_len_tradeoff_nn} for MLPs. 

\begin{table}[t]
\centering
\small
\caption{Counterfactual explanations generated by ElliCE based on German Credit dataset. We provide the original input, closest counterfactual, and actionable counterfactual respecting immutable features (denoted by $\dagger$). $\varepsilon_\text{target}=0.01 \hat{L}_{train}(f_{\text{baseline}})$}
\label{tab:actionability_example}
\begin{tabular}{l|c|c|c}
\hline
\textbf{Feature} & \textbf{Original} & \textbf{Robust CE} & \textbf{Actionable and Robust CE} \\
\hline
Checking Account Status 
& No account 
& No account 
& No account \\
Duration (months)
& 48 
& \textbf{37.09} 
& \textbf{34.7} \\
Credit History$^\dagger$
& Critical account 
& Critical account  
& Critical account  \\
Purpose$^\dagger$
& Car (new) 
& Car (new) 
& Car (new) \\
Credit Amount (DM)
& 10 127 
& \textbf{7 584} 
& \textbf{7 073} \\
Savings Account Status 
& 500–1000 DM 
& 500–1000 DM 
& 500–1000 DM \\
Employment Duration 
& 1–4 yrs 
& 1–4 yrs 
& 1–4 yrs \\
Installment Rate (\%) 
& 2 
& 2 
& 2 \\
Personal Status \& Sex$^\dagger$
& Male, single 
& Male, single 
& Male, single \\
Other Debtors/Guarantors 
& None & None & None \\
Residence Since 
& 2 & 2 & 2 \\
Property$^\dagger$
& No property 
& No property 
& No property \\
Age $^\dagger$
& 44 
& \textcolor{red}{\textbf{54.06}} 
& 44 \\
Other Installment Plans 
& Bank 
& Bank 
& Bank \\
Housing$^\dagger$
& Free 
& Free 
& Free \\
Number of Credits$^\dagger$
& 0 
& \textcolor{red}{$\bm{0.29}$}
& 0 \\
Job Level$^\dagger$
& Skilled employee 
& Skilled employee 
& Skilled employee \\
Number of Dependents$^\dagger$
& 0 & 0 & 0 \\
Telephone 
& None 
& None 
& \textbf{Registered} \\
Foreign Worker$^\dagger$
& Yes & Yes & Yes \\
\hline
\end{tabular}
\end{table}

\subsubsection{Examples of Counterfactuals Generated by ElliCE}

Based on the German Credit dataset, we demonstrate ElliCE's ability to generate meaningful and actionable recourse.  
More specifically, when dealing with immutable features (those that cannot be changed), ElliCE can impose immutable feature constraints by restricting candidate selection or gradient flow, thereby generating realistic alternatives.

In Table~\ref{tab:actionability_example}, we present the original input, the closest output that is robust to model changes (over the Rashomon set), and then the actionable output. Features such as ``Age,'' ``Personal status $\&$ Sex,'' and whether someone is a ``Foreign Worker'' were held fixed and marked with $\dagger$.

As input (second column), consider an applicant who was denied credit. 
The closest robust counterfactual requires waiting 10 years, which is not feasible for the applicant. We allow changes in other features such as ``Credit amount'' or ``Duration'' but fix ``Age'' and other features of choice, which can be chosen by user depending on what is actionable for them at the time.

\begin{figure}[t]
    \centering
\includegraphics[width=1\textwidth]{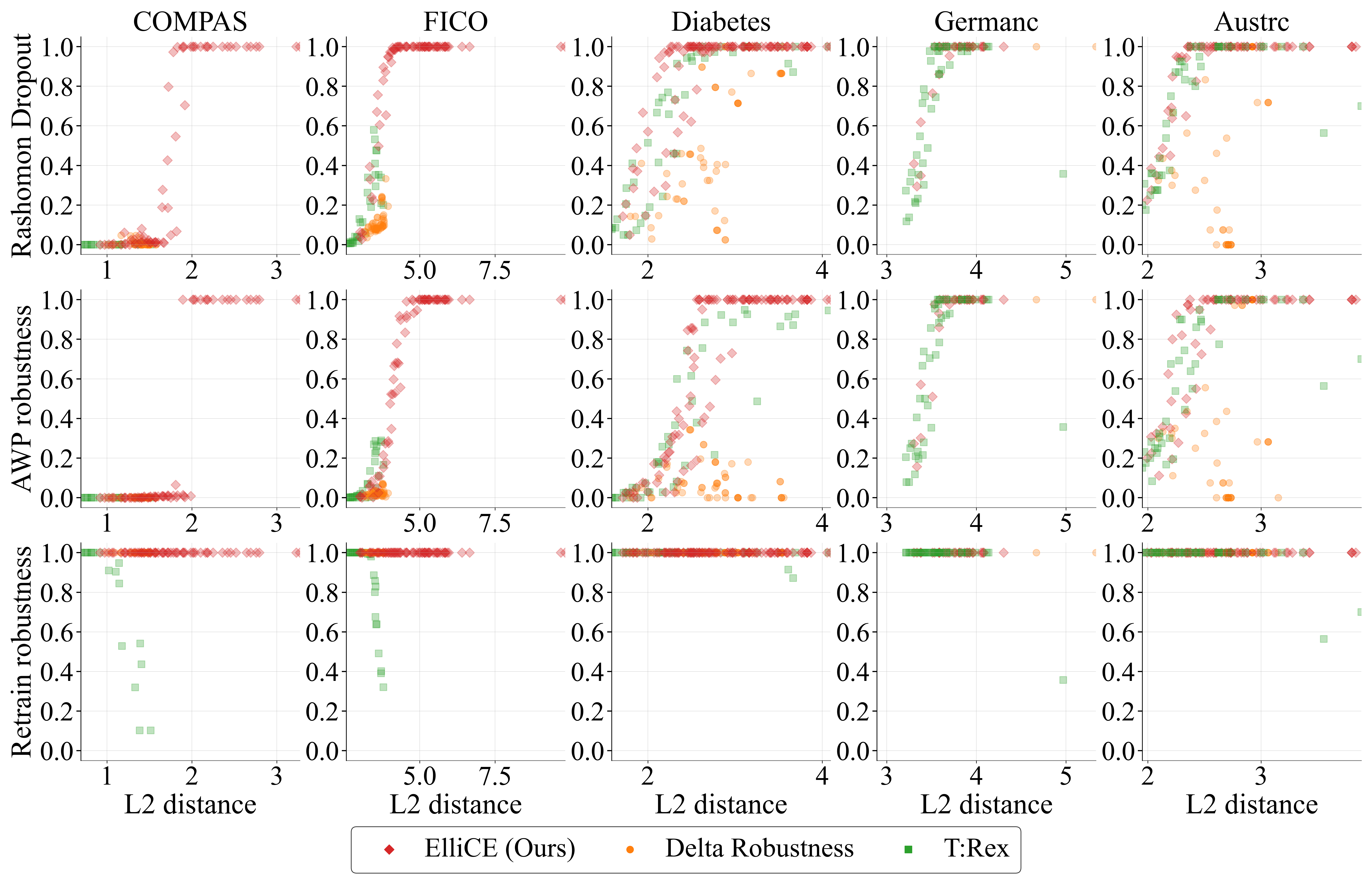}
        \caption{Robustness vs length tradeoff on linear models for all datasets and methods. Data-supported generation,  $\varepsilon_{\text{target}}=0.04$.
}
    \label{fig:rob_len_tradeoff_lin}
\end{figure}

\begin{figure}[t]
    \centering
\includegraphics[width=1\textwidth]{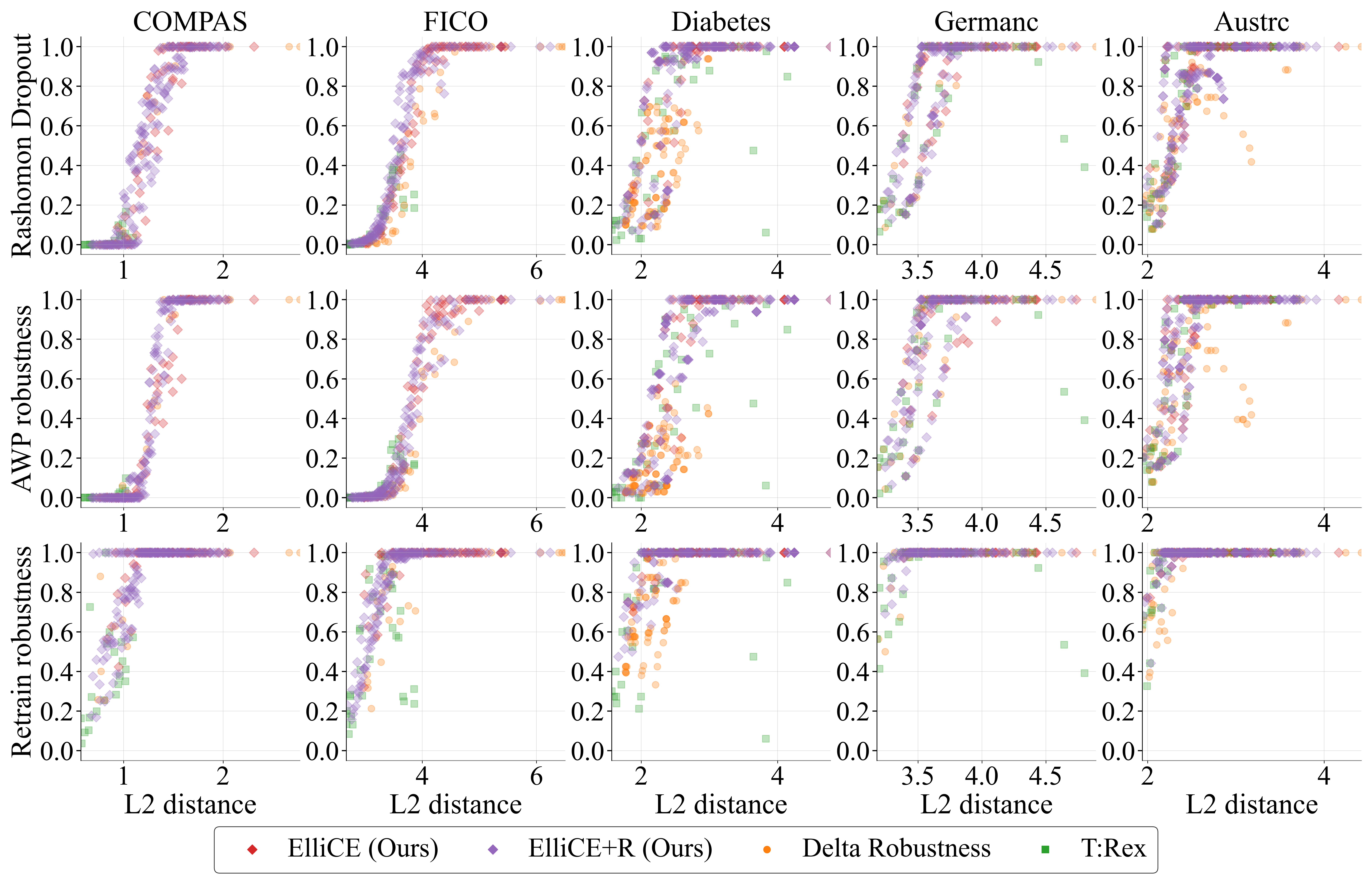}
    \caption{Robustness vs length tradeoff for all datasets and methods on MLPs. Data-supported generation, $\varepsilon_{\text{target}}=0.04$.
}
    \label{fig:rob_len_tradeoff_nn}
\end{figure}

\begin{figure}[t]
    \centering
\includegraphics[width=1\textwidth]{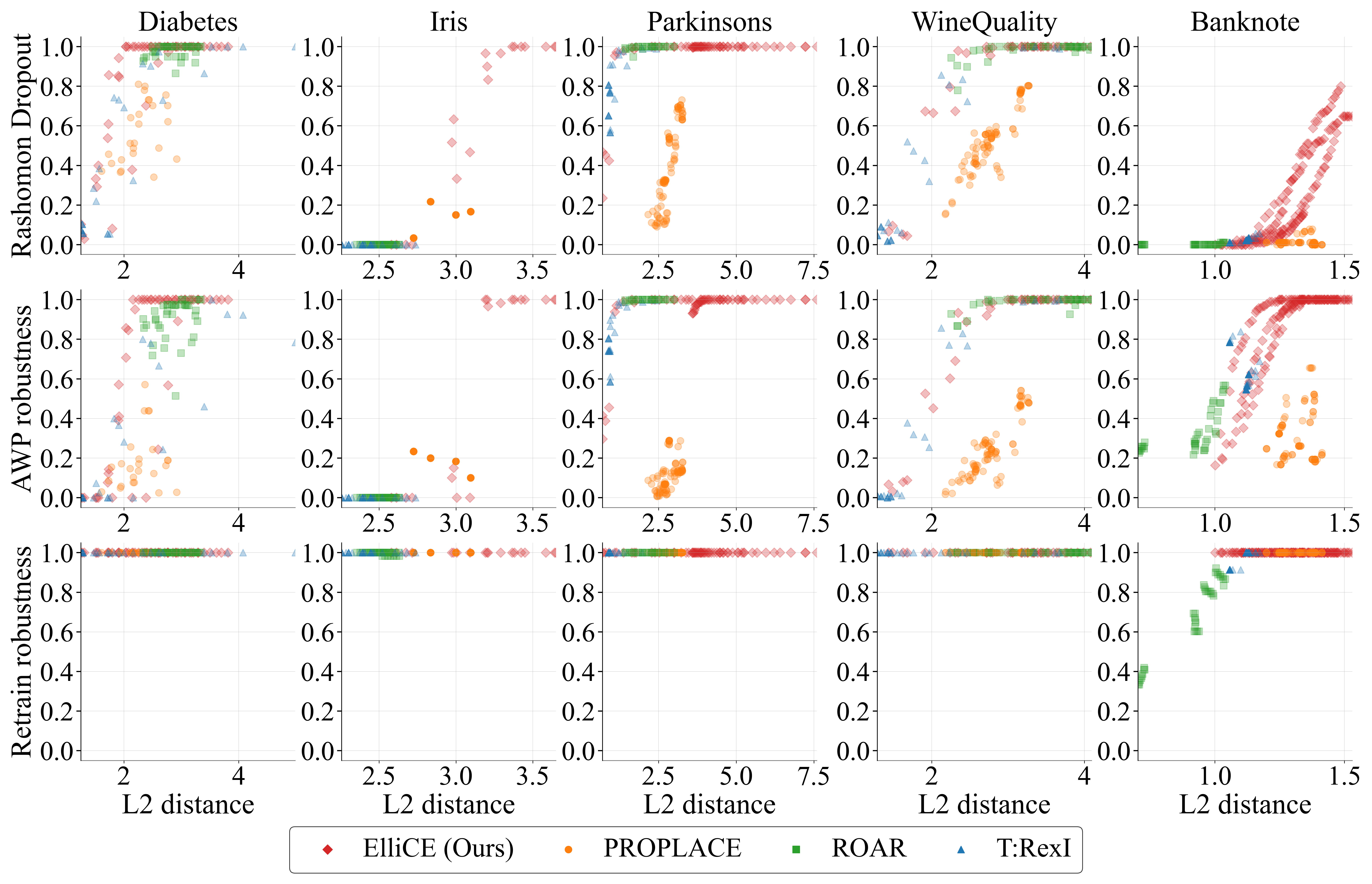}
    \caption{Robustness vs length tradeoff for all datasets and methods on linear models. Continuous generation,  $\varepsilon_{\text{target}}=0.04$.
}
    \label{fig:rob_len_tradeoff_lin_cnt}
\end{figure}

\begin{figure}[t]
    \centering
\includegraphics[width=1\textwidth]{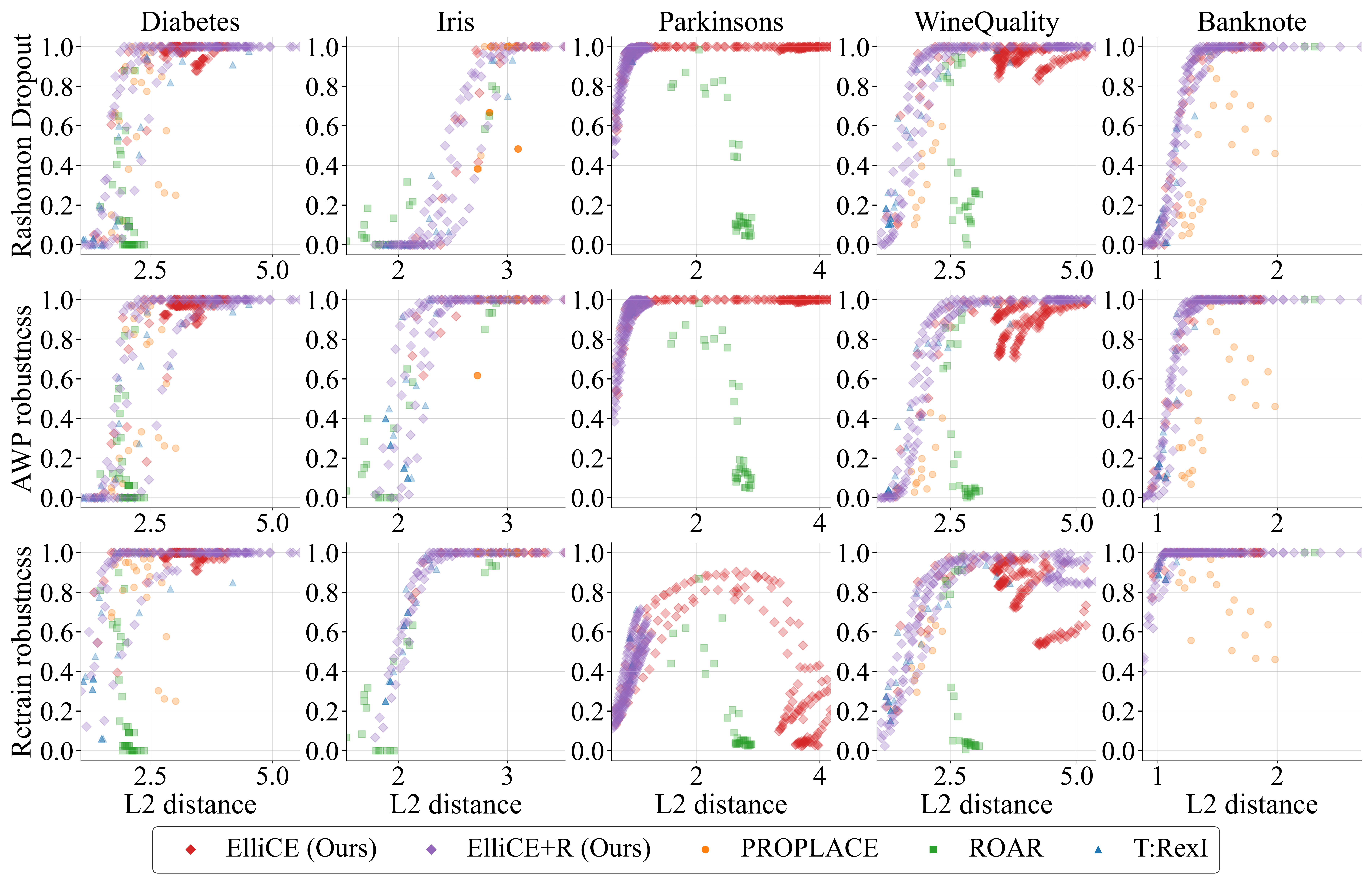}
    \caption{Robustness vs length tradeoff for all datasets and methods. MLPs. Continuous generation. $\varepsilon_{\text{target}}=0.04$.
}
    \label{fig:rob_len_tradeoff_nn_cnt}
\end{figure}


\subsubsection{Analysis of Robustness }

As shown in Figure \ref{fig:extension_rob_linear}, across all five datasets (COMPAS, FICO, Diabetes, German and Australian) as well as all metrics, ElliCE keeps high robustness score almost everywhere, indicating method's ability to adapt to different multiplicity levels. All three methods are robust to Retrain Robustness. Note, that Retrain Robustness for linear models is not influenced by \(\varepsilon_\text{target}\) in any way, due to a determinism of linear solvers, thus the plots stay constant. Delta Robustness starts strong, but quickly drops with increase of \(\varepsilon_\text{target}\). This shows that Delta Robustness is able to capture slight model perturbations, but might not able to extend to more uncertainty. T:Rex performs on par or better than Delta Robustness on all datasets except for COMPAS, showing overall similar behavior.




For MLPs, Figure~\ref{fig:extension_rob_mlp} shows the robustness of ElliCE, Delta Robustness, and T:Rex across different metrics (Retrain, AWP, and Rashomon dropout) and multiplicity levels ($\varepsilon_\text{target}$).

A method's ability to maintain robustness for large $\varepsilon_\text{target}$ values often depends on its capacity to explore more distant counterfactuals as its internal hyperparameters are varied. Delta Robustness exhibits this property of exploring further points. T:Rex, however, not always does. If T:Rex considers a point robust for a given threshold $\tau$, increasing $\tau$ will not necessary change that specific counterfactual, if it remains robust under new $\tau$. ElliCE is designed to navigate the trade-off between counterfactual length and robustness. It increases counterfactual length strategically, only if doing so leads to an improvement in robustness.

\subsubsection{Analysis of Length }

Figure \ref{fig:extension_length_linear} depicts, for each dataset, how the average \(\ell_2\) distance of counterfactuals increases with increasing \(\epsilon\). For random retrain, all methods give approximately the same length. ElliCE tends to be in the middle, but has the longest counterfactuals for FICO and COMPAS datasets. For FICO dataset we observe  a clear trade-off between robustness and length: while other methods struggle to find robust counterfactuals on this dataset, ElliCE does it by increasing the length.

Figure~\ref{fig:extension_length_mlp} shows that, on average, both the counterfactual lengths themselves and their increasing trends with $\varepsilon_\text{target}$ are similar across methods for MLPs.

\subsubsection{Robustness-Length Trade-off}\label{appendix:tradeoff}

Figure~\ref{fig:rob_len_tradeoff_lin} presents the robustness–length trade-off for linear models, computed as the average achieved robustness on the validation dataset compared to the used length. It is noteworthy that for the Retrain Robustness metric, all methods achieve high robustness scores, which can be attributed to the determinism typically observed in linear models under retraining. ElliCE demonstrates a consistent trade-off, producing more robust examples as the length budget is increased. For Delta Robustness, on the Diabetes and Australian datasets, increased counterfactual length does not always translate to proportionally higher robustness compared to other methods. Additionally, for T: Rex for the Retrain Robustness, we observe a decrease in robustness as the distance grows, partially attributed to the fact that it fails to efficiently explore higher distances and returns ``no result'' counterfactual, which by definition is not robust nor valid.



Figure~\ref{fig:rob_len_tradeoff_nn} illustrates, for five benchmark datasets (with $\varepsilon_{\text{target}}=0.04$), how counterfactual length ($\ell_2$ distance) trades off against three robustness metrics: Dropout Robustness (top row), AWP robustness (middle row), and Retrain robustness (bottom row). 


All methods demonstrate a clear trade-off. On some datasets T:Rex explores a range of counterfactual lengths that are, in some instances, shorter than those achieved by ElliCE or Delta Robustness when ElliCE achieves comparable or higher robustness. 

Additional figures for Continuous CE methods are presented in Figures~\ref{fig:rob_len_tradeoff_lin_cnt} and ~\ref{fig:rob_len_tradeoff_nn_cnt}.

\begin{figure}[h]
    \centering
\includegraphics[width=1\textwidth]{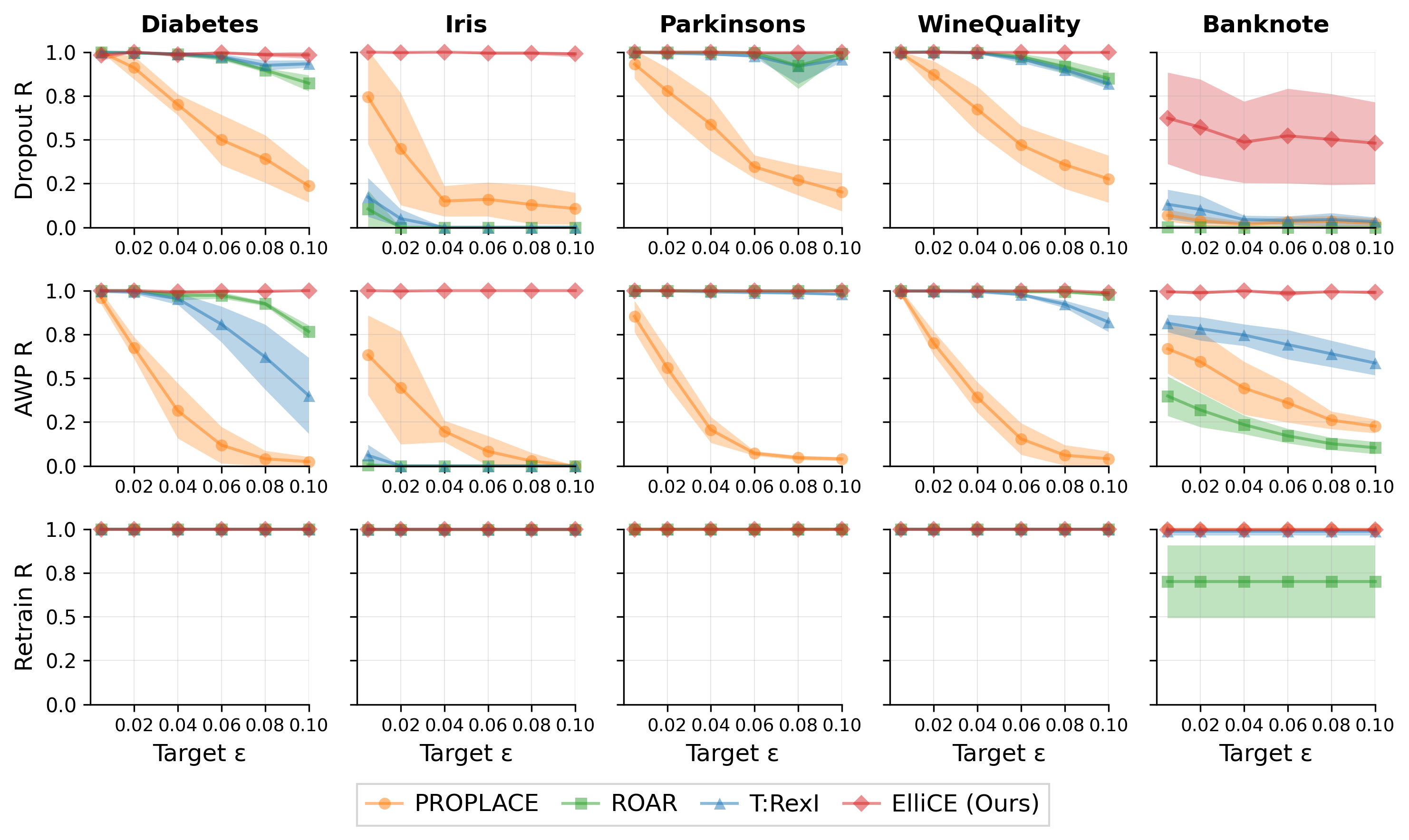}
    \caption{Robustness evaluation of Continuous ElliCE against baselines for linear models. On x-axis we have the target robustness level $\varepsilon_{\text{target}}$ and on y-axis the achieved robustness score. 
}
    \label{fig:cont_rob_linear}
\end{figure}

\begin{figure}[h]
    \centering
\includegraphics[width=1\textwidth]{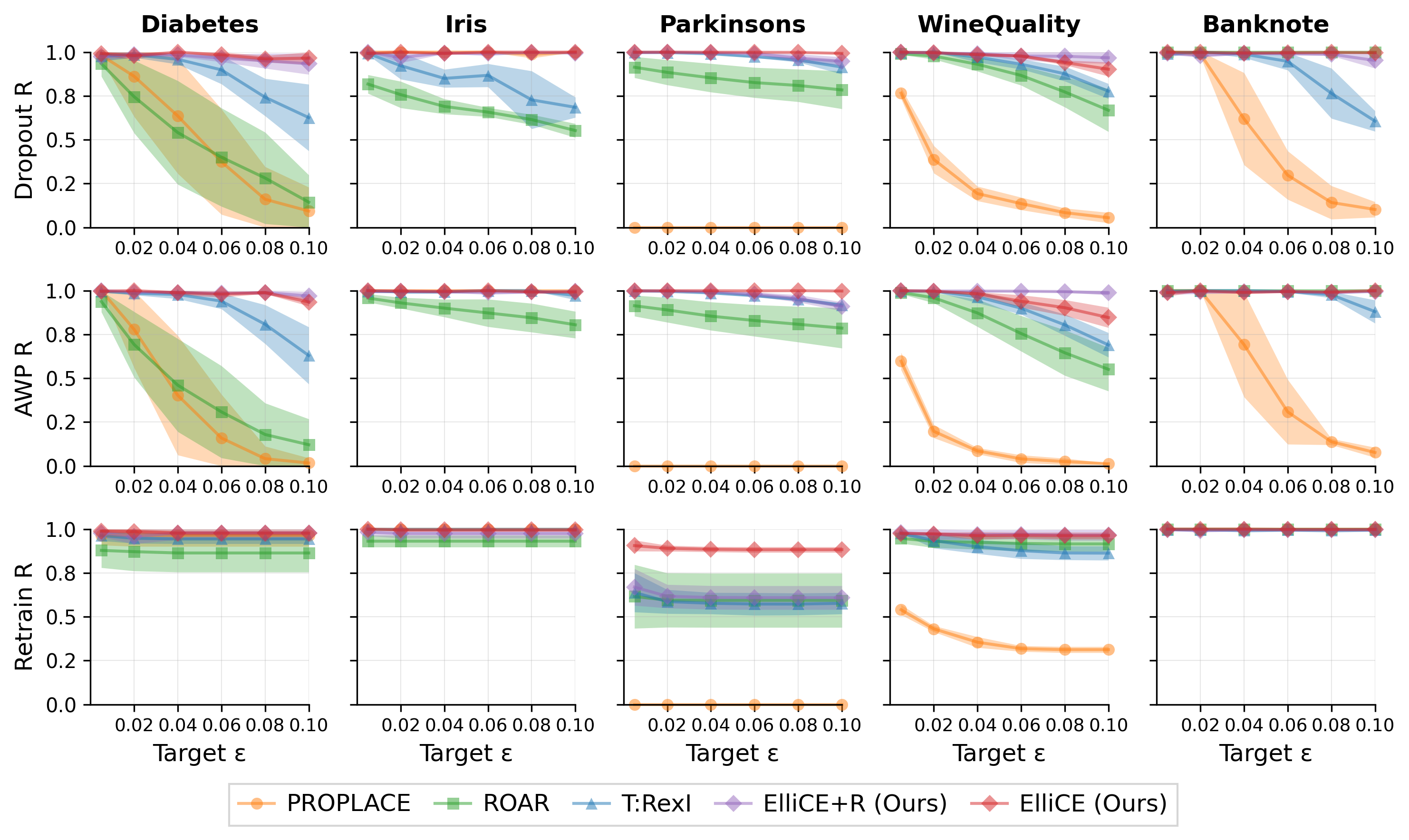}
    \caption{Robustness evaluation of Continuous ElliCE against baselines for NN models. On x-axis we have the target robustness level $\varepsilon_{\text{target}}$ and on y-axis the achieved robustness score.
}
    \label{fig:cont_rob_nn}
\end{figure}


\begin{figure}[h]
    \centering
\includegraphics[width=1\textwidth]{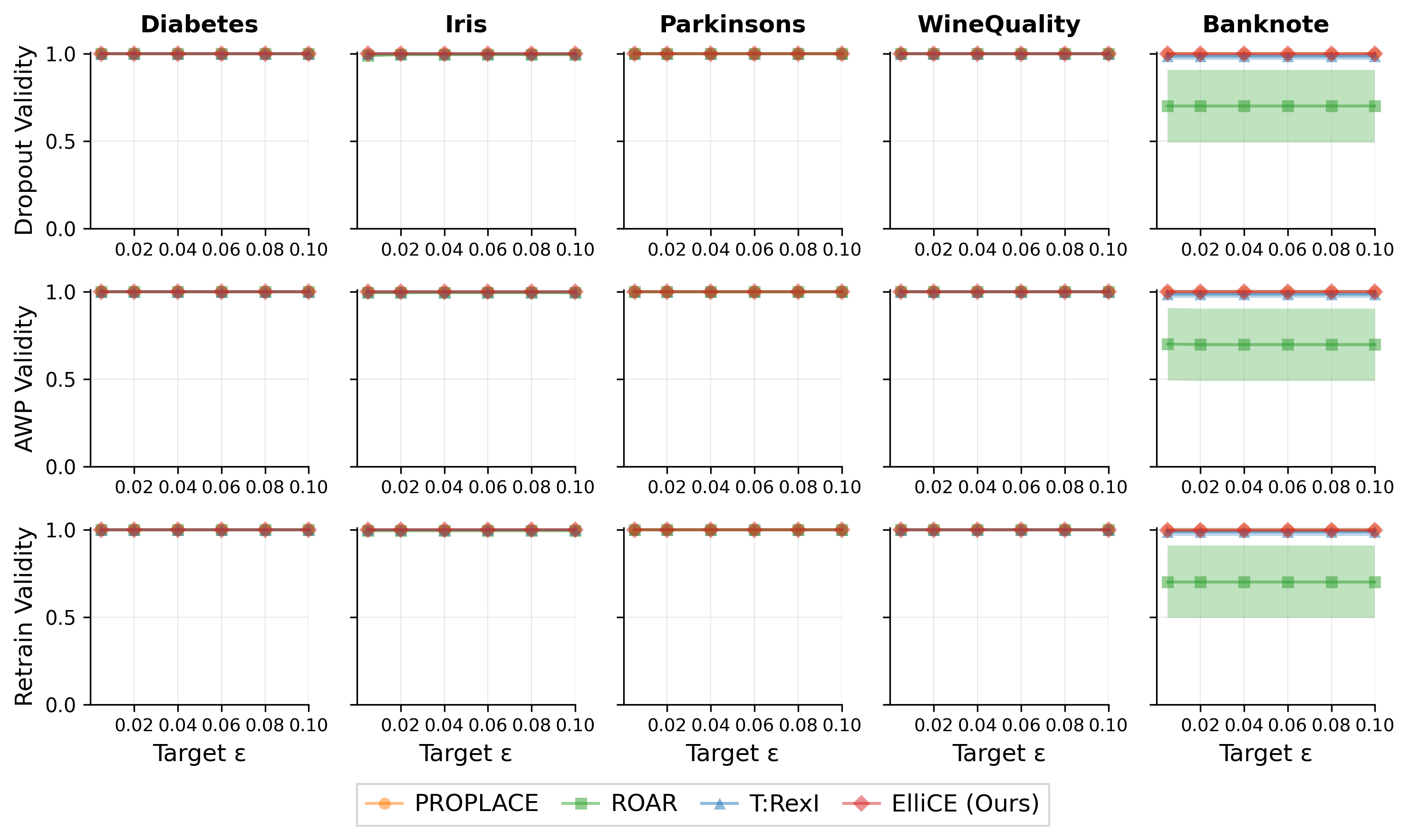}
    \caption{Validity evaluation of Continuous ElliCE against baselines for linear models. On x-axis we have the target robustness level $\varepsilon_{\text{target}}$ and on y-axis the achieved robustness score. 
}
    \label{fig:cont_val_linear}
\end{figure}

\begin{figure}[h]
    \centering
\includegraphics[width=1\textwidth]{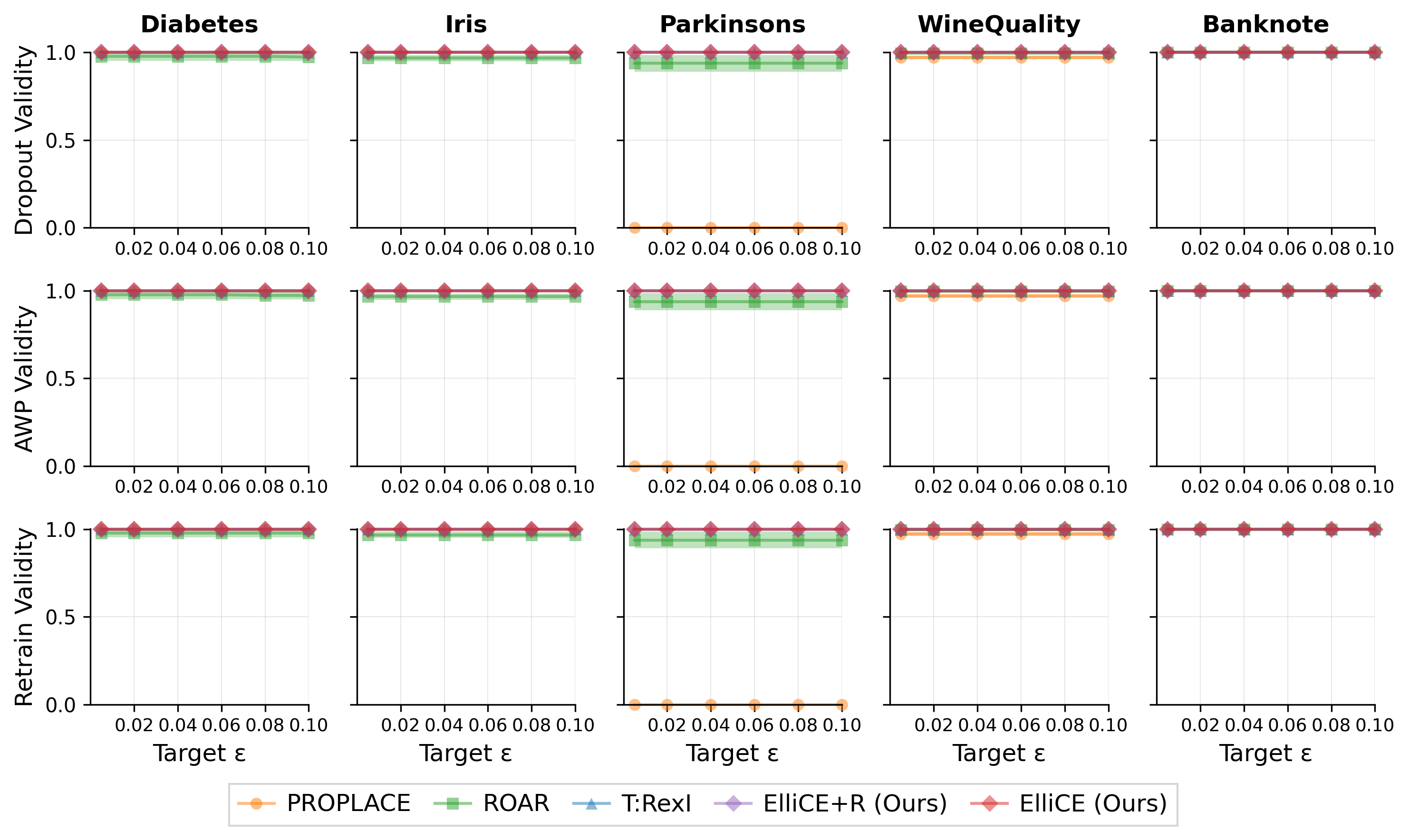}
    \caption{Validity evaluation of Continuous ElliCE against baselines for NN models. On x-axis we have the target robustness level $\varepsilon_{\text{target}}$ and on y-axis the achieved robustness score.
}
    \label{fig:cont_val_nn}
\end{figure}

\begin{figure}[h]
    \centering
\includegraphics[width=1\textwidth]{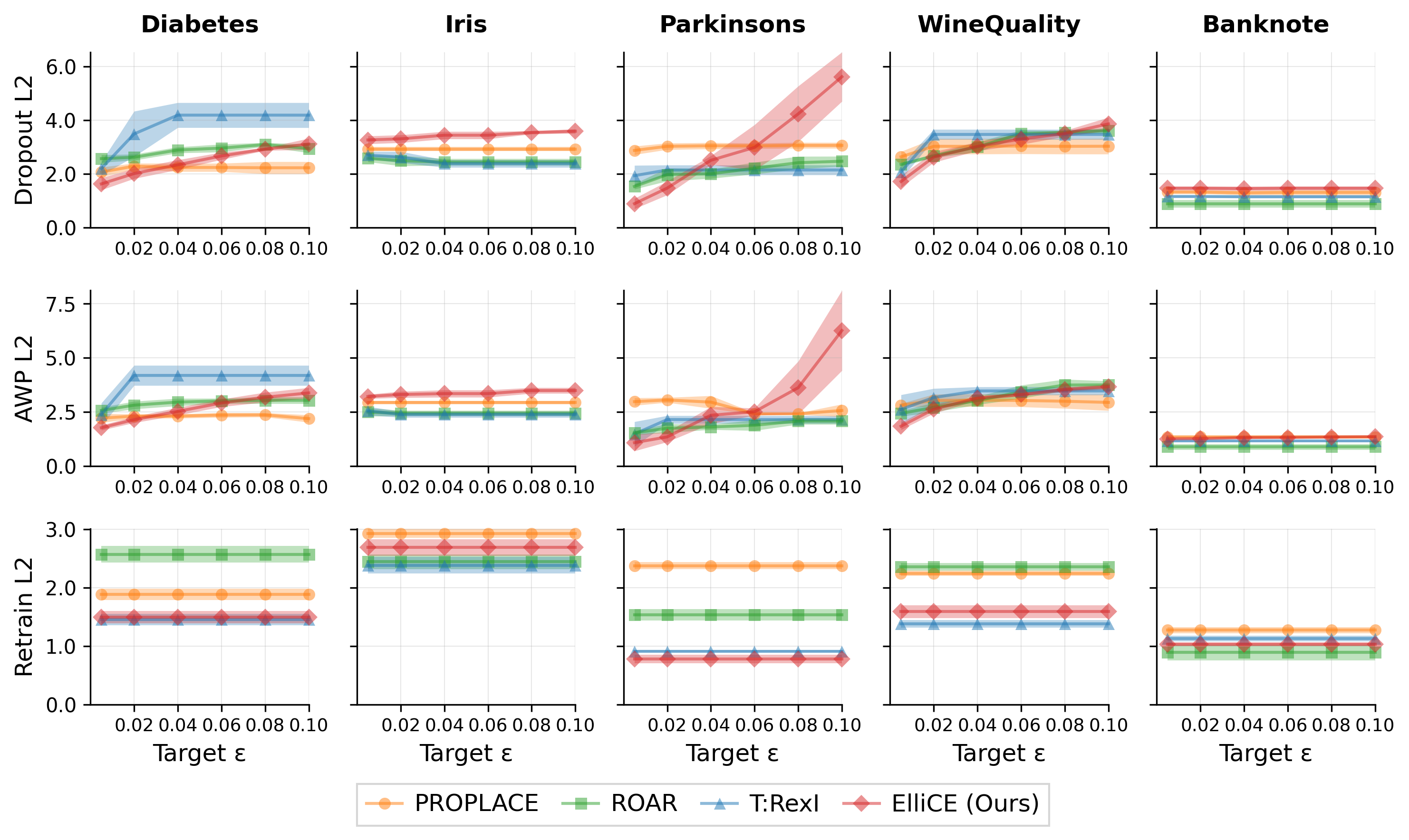}
    \caption{Length evaluation of Continuous ElliCE against baselines for linear models. 
}
    \label{fig:cont_len_linear}
\end{figure}

\begin{figure}[h]
    \centering
\includegraphics[width=1\textwidth]{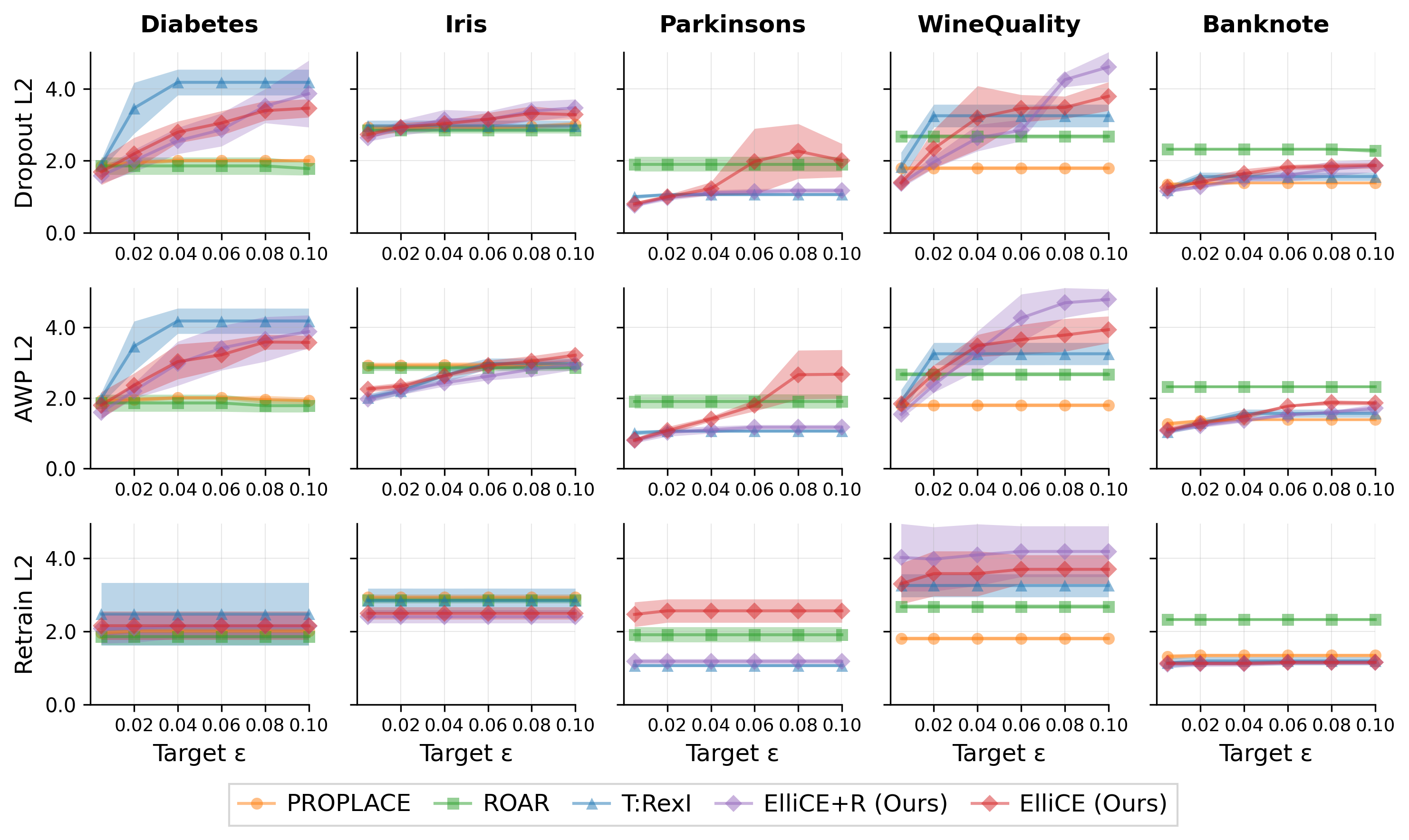}
    \caption{Length evaluation of Continuous ElliCE against baselines for NN models. On x-axis we have the target robustness level $\varepsilon_{\text{target}}$ and on y-axis the achieved $l_2$ length.
}
    \label{fig:cont_len_nn}
\end{figure}

\begin{figure}[h]
    \centering
\includegraphics[width=1\textwidth]{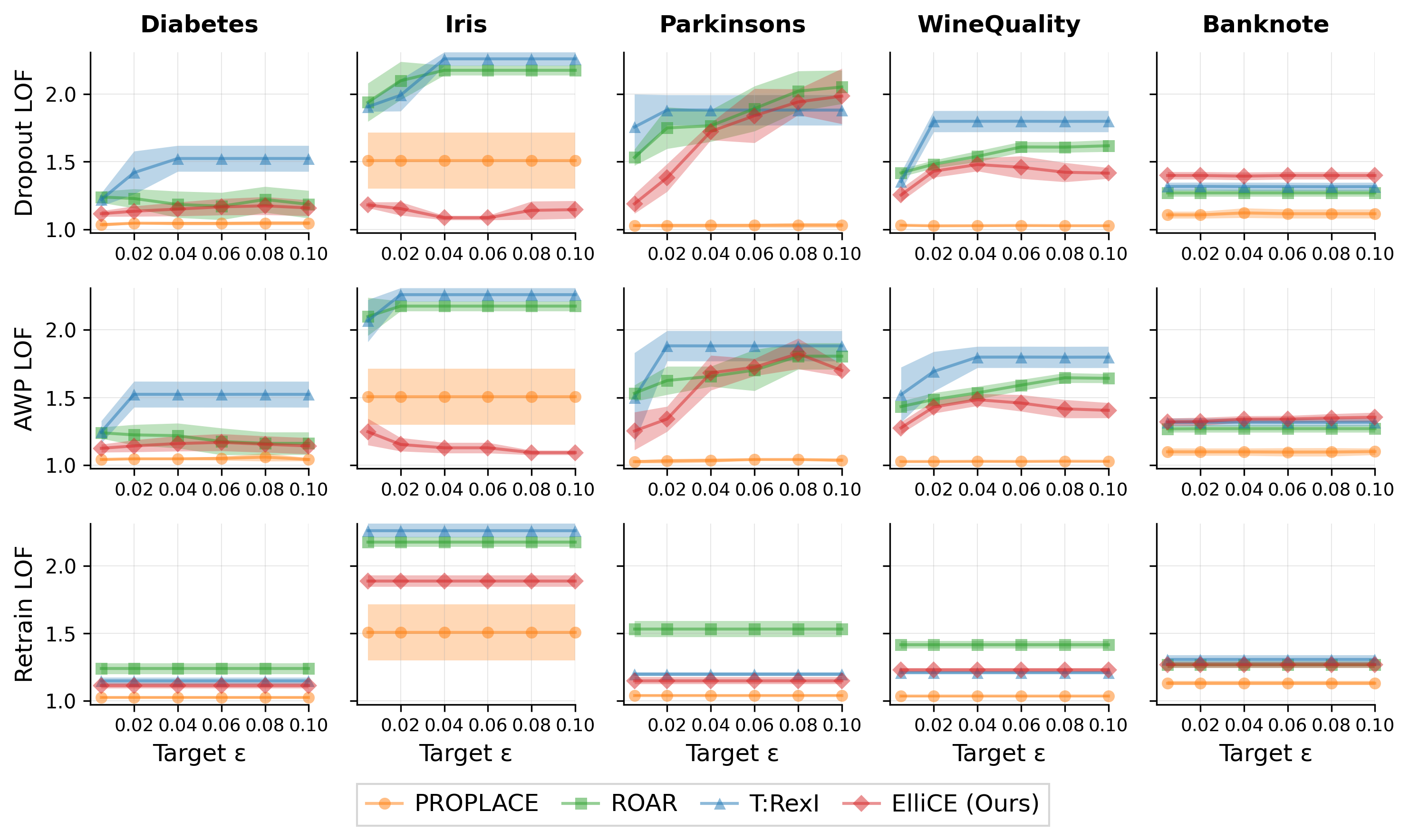}
    \caption{LOF evaluation of Continuous ElliCE against baselines for linear models. On x-axis we have the target robustness level $\varepsilon_{\text{target}}$ and on y-axis the achieved LOF score. 
}
    \label{fig:lof_cont_lin}
\end{figure}

\begin{figure}[h]
    \centering
\includegraphics[width=1\textwidth]{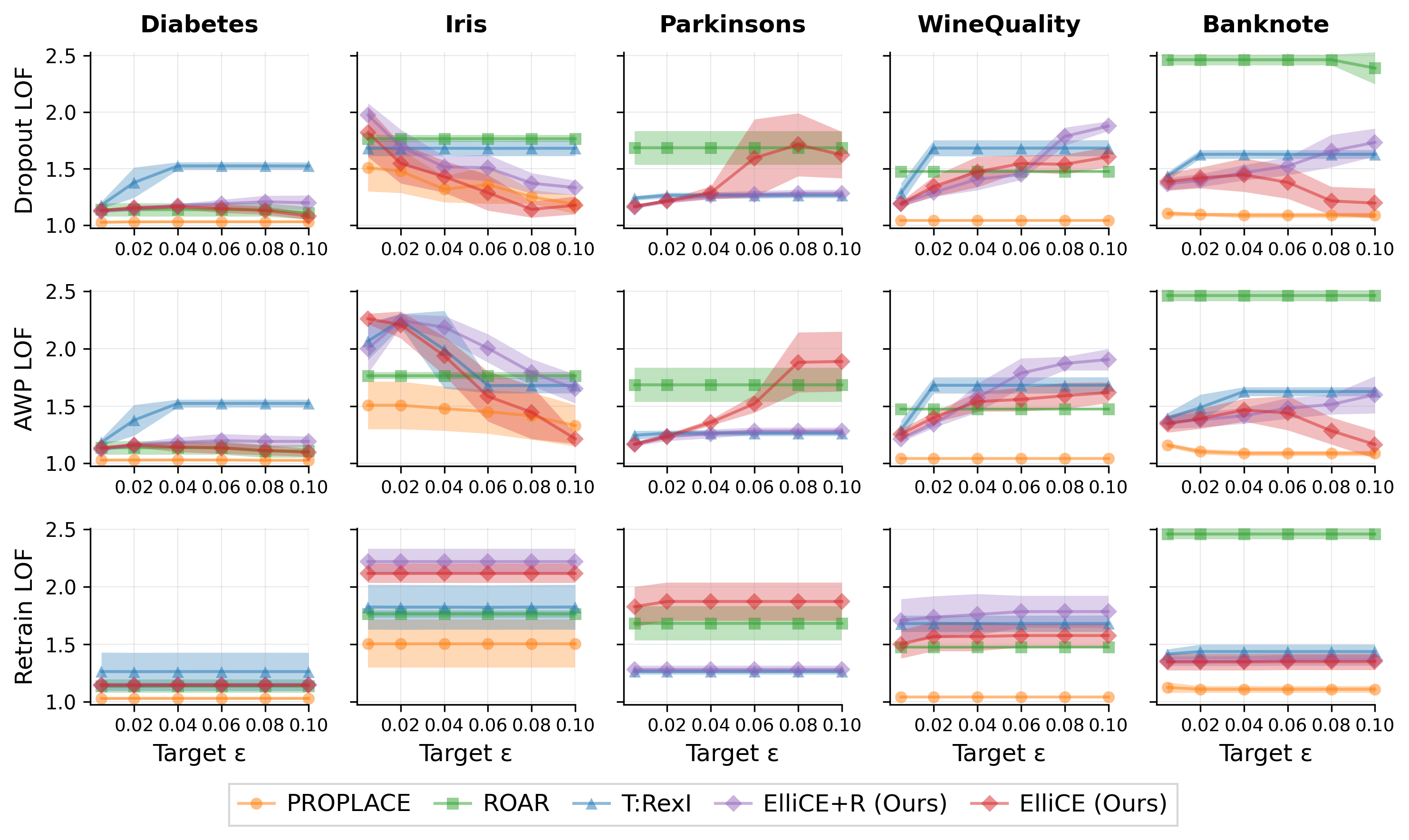}
    \caption{LOF evaluation of Continuous ElliCE against baselines for NN models. On x-axis we have the target robustness level $\varepsilon_{\text{target}}$ and on y-axis the achieved LOF score.
}
    \label{fig:lof_cont_mlp}
\end{figure}

\subsection{Experiments for Non-data Supported Counterfactual Generation}


Figures~\ref{fig:cont_rob_linear} and~\ref{fig:cont_rob_nn} illustrate robustness across varying multiplicity levels ($\epsilon$) for linear models on five benchmark tabular datasets: Diabetes, Iris, Parkinson’s, Wine Quality, and Banknotes. We compare ElliCE with PROPLACE, ROAR, and T:RexI, using Rashomon Dropout Robustness, AWP Robustness, and Retrain Robustness metrics. In Table~\ref{tab:performance_comparison_cnt_full} we present all robustness and length results for Non-data Supported Counterfactual Generation for LMs and MLPs (Retrain columns for linear models are filled with ``-'' due to the absence of randomness when training with a standard deterministic linear solver).

For Random Retrain evaluator, all methods achieve perfect scores. This is because linear models tend to not change upon retraining. Consequently, this also indicates that all methods achieve perfect validity (i.e., a score of 1) for the retrain method.

Figures~\ref{fig:cont_val_linear} and~\ref{fig:cont_val_nn} present validity results. Most methods generally find valid counterfactuals.


Figures~\ref{fig:cont_len_linear} and~\ref{fig:cont_len_nn} illustrate the relationship between counterfactual length and robustness. While all methods often start with similar counterfactual lengths for small multiplicity values, ElliCE distinctively increases the length of counterfactuals to enhance robustness, especially as $\varepsilon_\text{target}$ grows. Thus, when robust counterfactuals are  located nearby, other methods  also tend to find them. However, if achieving robustness at larger $\epsilon$ values necessitates exploring more distant points (i.e., increasing the length budget), ElliCE demonstrates a superior capability to do so effectively.

\subsection{Data shift}

To further support that the ElliCE effectively handles re-training scenarios (model shifts), we evaluate our methods and baselines under the data shift scenario.  Specifically, we split each dataset into two parts: the first part was used to train our method (ElliCE) and baseline models, while the second part was used exclusively to train the Random Retraining evaluator. 
This setup follows a setup similar to that used by T:Rex and Argumentative Ensembling and ensures that the methods under evaluation and the evaluator are trained on disjoint data, capturing a realistic data-shift scenario. 
For counterfactual evaluation, we use the same data split that was used to train evaluators model (the second data part described above). 
Under the data shift evaluation pipeline, ElliCE continues to outperform or perform on par with the baselines across both continuous and data-supported settings. Please see Figures~\ref{fig:shift_cnt_r} and~\ref{fig:shift_ds_r} for more details.

\begin{figure}[h]
    \centering
\includegraphics[width=1\textwidth]{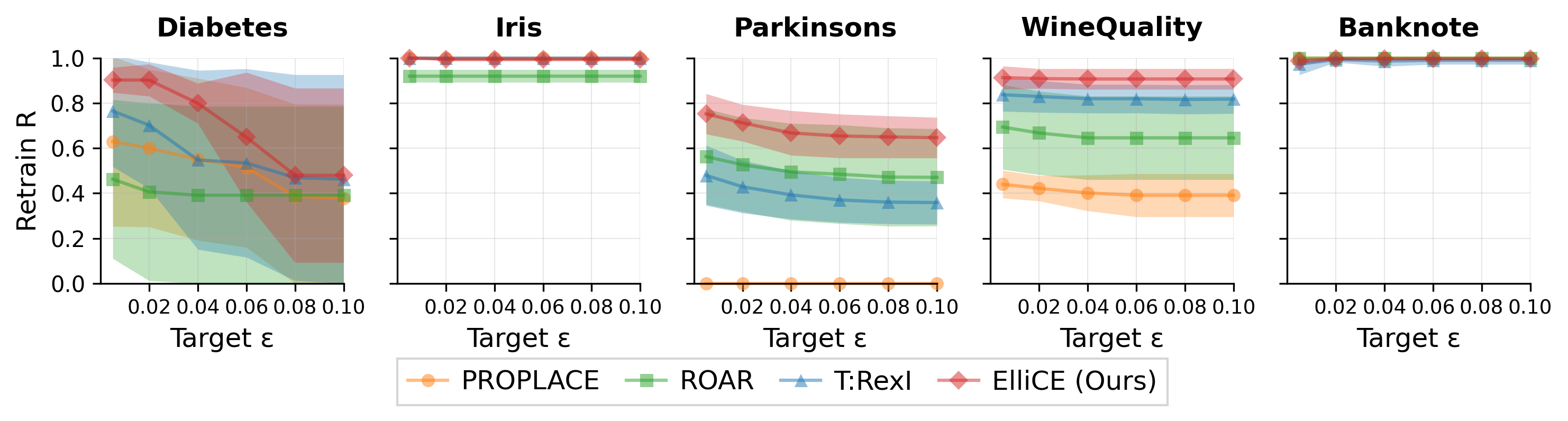}
    \caption{Robustness evaluation of Continuous ElliCE against baselines for MLP for Data Shift. On x-axis we have the target robustness level $\varepsilon_{\text{target}}$ and on y-axis the achieved Robustness score.
}
    \label{fig:shift_cnt_r}
\end{figure}

\begin{figure}[h]
    \centering
\includegraphics[width=1\textwidth]{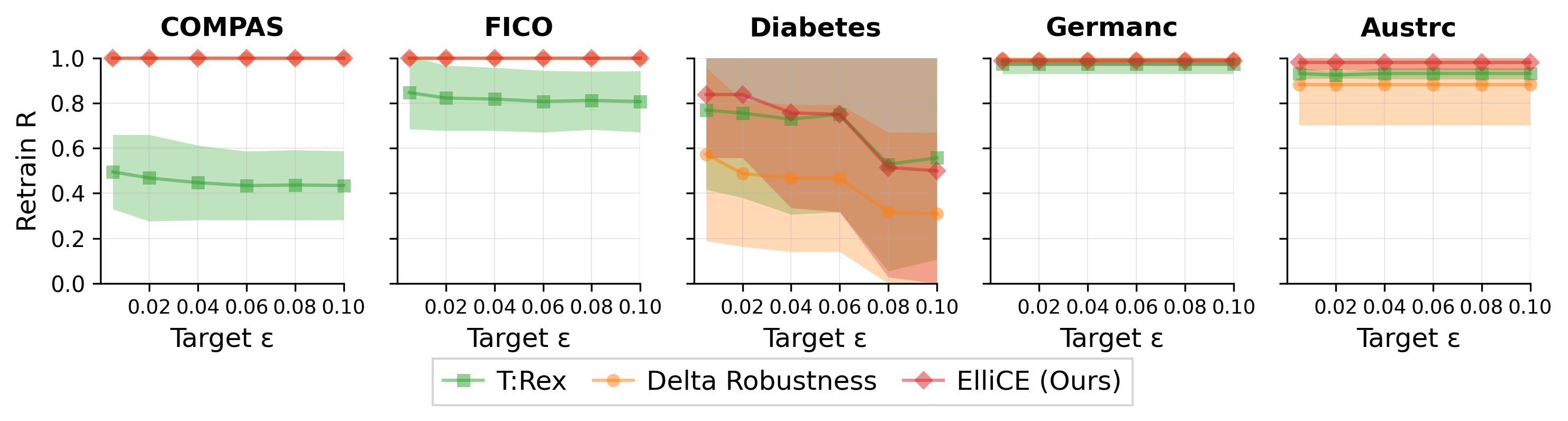}
    \caption{Robustness evaluation of Data supported ElliCE against baselines for MLP for Data Shift. On x-axis we have the target robustness level $\varepsilon_{\text{target}}$ and on y-axis the achieved Robustness score.
}
    \label{fig:shift_ds_r}
\end{figure}

\section{Delta-Robustness Works in the Last Layer under Reparameterization}

Delta-Robustness, a leading baseline for generating robust counterfactual explanations, perturbs all network parameters within an $\ell_\infty$-norm ball of a given radius $\delta$. Our analysis demonstrates a key finding: under a specific model reparameterization  Delta-Robustness works in the final layer of a MLP, an approach similar to that of ElliCE.

This reparameterization implies that Delta-Robustness then approximates the model uncertainty set in the last layer using an $\ell_\infty$-norm constraint (a hypercube), whereas ElliCE uses an $\ell_2$-norm constraint (an ellipsoid). While ellipsoidal bounds are often presumed to offer a tighter fit for convex losses, a hypercubic approximation might, in some cases, be more accurate. This geometric difference could potentially explain cases where Delta-Robustness performs better than ElliCE.

 In this section, we establish theoretical limitations for delta interval approaches to robust counterfactuals under model multiplicity. We prove that for certain model reparameterizations, establishing robust counterfactuals becomes either too restrictive or computationally inefficient.

\subsection{Problem Setting}
Consider a two-layer perceptron with input vector \(x \in \mathbb{R}^d\), weight matrices \(W_1 \in \mathbb{R}^{h \times d}\) and \(W_2 \in \mathbb{R}^{c \times h}\), and bias vectors \(b_1 \in \mathbb{R}^{h}\), \(b_2 \in \mathbb{R}^{c}\). The network output can be expressed as:
\[
y = W_2 \cdot \sigma(W_1 \mathbf{x} + b_1) + b_2,
\]
where \( \sigma \) is a 1-Lipschitz non-linear activation function (e.g., ReLU).
Throughout the following theorems we adopt a threshold of 0.

Let \(\Theta := \{(W_i, b_i)\}^{L}_{i=1}\) denote the network parameters of an L-layer network. One approach to ensure robustness under model multiplicity is to establish a \( \delta \)-interval around these parameters. For a fixed radius \(\delta > 0\) we define an \(\ell_\infty\)-ball:
\[
\mathcal M_\delta(\Theta)\;=\;
\Bigl\{\Theta' =  \{(W'_i, b'_i)\}^{L}_{i=1}\: \mid \;
\max\bigl(\|W_i-W'_i\|_\infty,\,
          \|b_i-b'_i\|_\infty\bigr)\le\delta\Bigr\}.
\]

\paragraph{Norm conventions.}
For the theoretical analysis, we measure weight perturbations using the induced matrix
\(\infty\)-norm
\[
\|A\|_{\infty}=\max_{i}\sum_{j}|A_{ij}|,
\]
which controls the worst-case row-sum change and yields dimension-stable bounds.
Our implementation uses entrywise (“box”) constraints
\[
\|A\|_{\max}=\max_{i,j}|A_{ij}|.
\]
If \(d_{A}\) is the number of columns of \(A\), these notions are related by
\[
\|A\|_{\max}\le \|A\|_{\infty}\le d_{A}\,\|A\|_{\max},
\]
hence \( \mathrm{Box}(\delta/d_{A})\subseteq \mathrm{Induced}(\delta)\subseteq \mathrm{Box}(\delta)\).
Consequently, results proven under the induced norm transfer to box perturbations up to a factor of \(d_{A}\) in the radius (i.e., with appropriately rescaled \(\delta\)).

(Above \(
\mathrm{Box}(\delta) := \{ A \mid \|A\|_{\max} \le \delta \}
\), and
\(
\mathrm{Induced}(\delta) := \{ A \mid \|A\|_{\infty} \le \delta \}.
\))

A counterfactual \(\mathbf{x}_c\) is considered \(\delta\)-robust if it maintains the desired prediction across all models in \(\mathcal{M}_\delta(\Theta)\).

\subsection{Necessary Conditions for Delta-Robustness}

We first establish necessary conditions for a meaningful delta-robustness analysis.

\begin{proposition}\label{prop:delta-upper-bound}
Consider an L-layer network whose last affine layer is \(h_L(\mathbf{x}) = W_Lh_{L-1}(\mathbf{x})+b_L\) with decision rule \(\hat{y}(\mathbf{x}) = \mathbf{1}[h_L(\mathbf{x})>0]\). Assume the final-layer bias satisfies \(b_L<0\).

Robust \(0\rightarrow 1\) counterfactuals (i.e., for a reference input \(\mathbf{x}_0\) classified as \(0\) under the reference parameters \(\Theta\), there exists a counterfactual \(\mathbf{x}_c\) that is classified as \(1\) under every \(\Theta' \in \mathcal{M}_\delta(\Theta)\)) exist within the \(\ell_\infty\)-ball \(\mathcal{M}_{\delta}(\Theta)\) only if the perturbation radius \(\delta>0\) satisfies:
\[
\delta < \|W_L\|_\infty.
\]
\end{proposition}

\begin{proof}
Assume that the perturbation radius satisfies \(\delta \geq \|W_L\|_\infty\). Consider a parameter set \(\Theta' =  \{(W'_i, b'_i)\}^{L}_{i=1}\), where
\[
W'_i=W_i,\quad b'_i=b_i,\quad \forall i \in \{1,\dots,L-1\},\quad W'_L=\mathbf{0},\quad b'_L = b_L.
\]

As all layers except the last one are unchanged, we can verify that \(\Theta' \in \mathcal{M}_\delta(\Theta)\):
\[
\|W_L - W_L'\|_\infty = \|W_L - \mathbf{0}\|_\infty = \|W_L\|_\infty \leq \delta,
\qquad
\|b_L - b_L'\|_\infty = 0.
\]

For any input \(\mathbf{x}\) in the last layer we have:
\[
h'_L(\mathbf{x}) = W'_Lh_{L-1}(\mathbf{x})+b'_L=\mathbf{0}\cdot h_{L-1}(\mathbf{x})+b_L=b_L < 0.
\]

The model \(\Theta'\) has a degenerate last layer that produces a constant output \(b_L < 0\) and predicts class \(0\) regardless of the input. Since this degenerate model belongs to \(\mathcal{M}_\delta(\Theta)\), no counterfactual \(\mathbf{x}_c\) can be robust \(0\to 1\) with respect to all models in the interval. Therefore, to avoid this case, we must have \(\delta < \|W_L\|_\infty\).
\end{proof}

This proposition establishes our first condition (A) and provides a clear upper bound on the permissible magnitude of \(\delta\).

\subsection{The Impact of Model Reparameterization}

Neural networks exhibit reparameterization invariance, where different parameter configurations can represent functionally equivalent models. We now demonstrate that this property creates fundamental challenges for the delta interval approach.

Let a two-layer perceptron be given by
\[
y = W_2 \sigma(W_1 \mathbf{x} + b_1) + b_2,
\]
where \(W_1, W_2\) are weight matrices, \(b_1, b_2\) are bias vectors and \(\sigma\) is the ReLU activation.

Reparameterization can be achieved by scaling. For any \(\alpha>0\) we define
\[
W_1^{(\alpha)} = \alpha W_1,\quad b_1^{(\alpha)} = \alpha b_1,\quad
W_2^{(\alpha)} = \frac{1}{\alpha} W_2,\quad b_2^{(\alpha)} =  b_2.
\]
Because ReLU is positively homogeneous, i.e., \(\sigma(\alpha z) = \alpha \sigma(z)\) for all \(\alpha > 0\), the output is unchanged:
\[
y' = W_2^{(\alpha)} \sigma(W_1^{(\alpha)} \mathbf{x} + b_1^{(\alpha)}) + b_2^{(\alpha)}
= \frac{W_2}{\alpha}\,\sigma(\alpha(W_1 \mathbf{x} + b_1)) + b_2
= y.
\]

\begin{theorem}
Let
\[
f_\Theta(\mathbf{x}) =  W_2 \sigma(W_1 \mathbf{x} + b_1) + b_2,\quad \Theta = (W_1, b_1, W_2, b_2)
\]
be a two-layer ReLU (1-Lipschitz) perceptron. Suppose the decision threshold is 0, input data is \(X = [-1, 1]^d\) and the negative final bias \(b_2<0\).

Fix $\delta > 0$ and \(\gamma >0\). Suppose the following two conditions for \(\Theta\) hold:

\begin{enumerate}[leftmargin=*]
    \item (A) Final-layer robustness bound:
    \[
    \delta <  \|W_2\|_\infty
    \]
    \hfill\textit{(see Proposition~\ref{prop:delta-upper-bound})}

    \item (B) Layer-wise tangibility:

    For each layer \(i \in \{1,2\}\), there exists a parameter perturbation in that layer that is no larger than \(\delta\) in \(\ell_\infty\) norm that changes the output by at least \(\gamma\) for some \(\mathbf{x} \in X\).

    Here, for a matrix \(A\), \(\|A\|_\infty\) denotes the induced matrix \(\infty\)-norm (maximum absolute row sum).

    For any pair \((W, b)\) and radius \(\delta>0\) set:
    \[
    \mathcal{B}_\infty((W,b),\delta) := \Bigl\{(W',b') \,\bigm|\, \|W'-W\|_\infty \leq \delta,\ \|b'-b\|_\infty \leq \delta\Bigr\}.
    \]

    Now layer-wise tangibility condition for a particular pair \(W, b\) can be defined as:
    \begin{align}
    &\sup_{\substack{
    \mathbf{x} \in X \\
    (W', b') \in \mathcal{B}_\infty((W,b),\delta)
    }} \bigl|f(W',b')-f(W'_{-i} \cup W_i, b'_{-i} \cup b_i)\bigr| > \gamma.
    \end{align}

    For binary classification, \(\gamma\) can be set to:
    \[
    \gamma = \min_{\mathbf{x} \in \mathcal{D}} |f(\mathbf{x})|.
    \]
\end{enumerate}

Then there exists \(\alpha>0\) such that for the functionally equivalent reparameterization
\(\Theta^{(\alpha)}=(W_1^{(\alpha)}, b_1^{(\alpha)}, W_2^{(\alpha)}, b_2^{(\alpha)})\),
no choice of radius \(\delta'>0\) can simultaneously satisfy the corresponding conditions (A) and (B) for the new model.
\end{theorem}

\begin{proof}
Fix \(\alpha>0\) and consider the reparameterized model \(\Theta^{(\alpha)}\).
Condition (A) for the new model requires
\[
\delta' < \|W_2^{(\alpha)}\|_\infty = \frac{\|W_2\|_\infty}{\alpha}.
\]

We now clarify how tangibility (B) is used in the proof. Since (B) must hold \emph{for each layer}, it is enough to show that, for sufficiently large \(\alpha\), \emph{no} perturbation of the first layer with size at most \(\delta'\) can change the output by \(\gamma\) for any input \(\mathbf{x}\in X\). This contradicts layer-wise tangibility for layer \(i=1\).

We now establish a key lemma.

\begin{lemma}
Consider a two-layer neural network with inputs \(\mathbf{x} \in [-1,1]^d\) and 1-Lipschitz activation function \(\sigma\). Fix a second-layer weight matrix \(\widetilde W_2\). Let \((\widetilde W_1,\widetilde b_1)\) be a perturbation of \((W_1,b_1)\) such that
\[
\|\widetilde W_1 - W_1\|_\infty \leq \delta
\quad\text{and}\quad
\|\widetilde b_1 - b_1\|_\infty \leq \delta.
\]
Then, for all \(\mathbf{x}\in[-1,1]^d\),
\[
\bigl\|\widetilde W_2 \sigma(W_1 \mathbf{x} + b_1) - \widetilde W_2 \sigma(\widetilde W_1 \mathbf{x} + \widetilde b_1)\bigr\|_\infty
\le 2\delta \cdot \|\widetilde W_2\|_\infty.
\]
\end{lemma}

\begin{proof}[Proof of Lemma]
Let \(\mathbf{x} \in [-1,1]^d\). The change in the hidden representation satisfies
\begin{align*}
\bigl\|\sigma(W_1\mathbf{x} + b_1) - \sigma(\widetilde W_1\mathbf{x} + \widetilde b_1)\bigr\|_\infty
&\le \bigl\|(W_1-\widetilde W_1)\mathbf{x} + (b_1-\widetilde b_1)\bigr\|_\infty \\
&\le \|W_1-\widetilde W_1\|_\infty\,\|\mathbf{x}\|_\infty + \|b_1-\widetilde b_1\|_\infty \\
&\le \delta\cdot 1 + \delta
= 2\delta,
\end{align*}
where we used the 1-Lipschitz property of \(\sigma\) and \(\|\mathbf{x}\|_\infty\le 1\). Therefore,
\[
\bigl\|\widetilde W_2 \sigma(W_1 \mathbf{x} + b_1) - \widetilde W_2 \sigma(\widetilde W_1 \mathbf{x} + \widetilde b_1)\bigr\|_\infty
\le \|\widetilde W_2\|_\infty \cdot \bigl\|\sigma(W_1\mathbf{x} + b_1) - \sigma(\widetilde W_1\mathbf{x} + \widetilde b_1)\bigr\|_\infty
\le 2\delta\,\|\widetilde W_2\|_\infty.
\]
\end{proof}

Returning to the main proof, consider any second-layer weight matrix \(\widetilde W_2\) satisfying
\(\|\widetilde W_2 - W_2^{(\alpha)}\|_\infty \le \delta'\). Then
\[
\|\widetilde W_2\|_\infty \le \|W_2^{(\alpha)}\|_\infty + \delta' = \frac{\|W_2\|_\infty}{\alpha} + \delta'.
\]
By the lemma, for any perturbation \((\widetilde W_1,\widetilde b_1)\) of the first layer with
\(\|W_1^{(\alpha)}-\widetilde W_1\|_\infty\le \delta'\) and \(\|b_1^{(\alpha)}-\widetilde b_1\|_\infty\le \delta'\), we have, for all \(\mathbf{x}\in X\),
\[
\bigl\|\widetilde W_2\sigma(W_1^{(\alpha)}\mathbf{x}+b_1^{(\alpha)})-\widetilde W_2\,\sigma(\widetilde W_1\mathbf{x}+\widetilde b_1)\bigr\|_\infty
\le 2\delta'\,\|\widetilde W_2\|_\infty
\le 2\delta'\Bigl(\frac{\|W_2\|_\infty}{\alpha}+\delta'\Bigr).
\]
Thus, in order for layer-wise tangibility (B) to hold for layer \(i=1\), it is necessary that the maximum possible change exceeds \(\gamma\), i.e.,
\[
2\delta'\Bigl(\frac{\|W_2\|_\infty}{\alpha}+\delta'\Bigr) > \gamma.
\]
Equivalently,
\begin{equation}\label{eq:necp1}
2(\delta')^2 + 2\delta'\,\frac{\|W_2\|_\infty}{\alpha} - \gamma > 0.
\end{equation}
Meanwhile (A) requires
\begin{equation}\label{eq:necp2}
\delta' < \frac{\|W_2\|_\infty}{\alpha}.
\end{equation}

For any fixed \(\gamma>0\), as \(\alpha\to\infty\) the linear term \(\tfrac{2\|W_2\|_\infty}{\alpha}\delta'\) vanishes, so~\eqref{eq:necp1} forces
\[
\delta' \;>\;\sqrt{\frac{\gamma}{2}} - \epsilon > 0
\]
for some small constant \(\epsilon>0\). But for sufficiently large \(\alpha\) we have
\[
\sqrt{\tfrac{\gamma}{2}} - \epsilon\;>\;\frac{\|W_2\|_\infty}{\alpha},
\]
so no \(\delta'>0\) can satisfy both \eqref{eq:necp1} and \eqref{eq:necp2}. This completes the proof.
\end{proof}

\subsection*{Implications}
Under the hypotheses of the theorem, and more broadly whenever a
\emph{single, layer-agnostic} radius~\(\delta\) is used, two tendencies emerge:
\begin{enumerate}
    \item If \(\delta\) is chosen comparatively large, the resulting
          \(\delta\)-interval may admit models in which some layers could shrink
          to (near-)zero. Such potentially degenerate members of the interval might
          predict the same undesired class for every input, so otherwise plausible
          counterfactuals could cease to be robust.
    \item Because neural networks admit
          reparameterizations that preserve their input--output
          behaviour, one may be able to construct an equivalent parameter
          set for which \emph{no} positive radius satisfies both the
          ``non-degeneracy'' and ``tangibility'' requirements.
          In other words, a single global~\(\delta\) may fail to
          accommodate all members of a functionally equivalent family.
\end{enumerate}
These observations suggest that---without additional machinery such as
layer-specific or parameter-specific radii---its ability to capture
meaningful model multiplicity may be limited.

\section{Enhanced Algorithmic Recourse: Actionability, Sparsity, and Extended Applications}\label{appendix:actionability}

This section presents additional contributions beyond the core ElliCE framework, including enhancements for actionability, sparsity control, plausibility evaluation, and extensions to broader model classes and feature types. These improvements address practical deployment challenges and extend the applicability of robust counterfactual generation to more diverse real-world scenarios.

\subsection{Plausibility Evaluation}

Plausibility is an essential part of counterfactual explanations, ensuring that generated recourse lies in realistic regions of the feature space.


To quantify plausibility, we adopt the \emph{Local Outlier Factor} (LOF) score, a widely used measure of local density deviation in outlier detection \cite{pmlr-v222-jiang24a}. Let \(S\) denote the training dataset, $\mathbf{x} \in \mathbb{R}^d$ a query point, and $L_k(\mathbf{x})$ the set of its $k$ nearest neighbors in $S$ under an $\ell_p$ distance $\delta(\cdot,\cdot)$. For any neighbor $\mathbf{x}_c \in L_k(\mathbf{x})$, we define the $k$-distance of $\mathbf{x}_c$ as $d_k(\mathbf{x}_c)$, and the \emph{reachability distance} of $\mathbf{x}$ with respect to $\mathbf{x}_c$ as
\[
r_{d_k}(\mathbf{x}, \mathbf{x}_c) := \max \{ \delta(\mathbf{x}, \mathbf{x}_c), \, d_k(\mathbf{x}_c) \}.
\]
The \emph{local reachability density} of $\mathbf{x}$ is then
\[
\mathrm{lrd}_k(\mathbf{x}) := \frac{|L_k(\mathbf{x})|}{\sum_{\mathbf{x}_c \in L_k(\mathbf{x})} r_{d_k}(\mathbf{x}, \mathbf{x}_c)}.
\]
Finally, the LOF score of $\mathbf{x}$ with respect to $S$ is defined as
\[
\mathrm{LOF}_{k,S}(\mathbf{x}) := \frac{1}{|L_k(\mathbf{x})|} \sum_{\mathbf{x}_c \in L_k(\mathbf{x})} \frac{\mathrm{lrd}_k(\mathbf{x}_c)}{\mathrm{lrd}_k(\mathbf{x})}.
\]

By construction, $\mathrm{LOF}_{k,S}(\mathbf{x}) \approx 1$ indicates that $\mathbf{x}$ lies in a region of comparable density to its neighbors, and is therefore considered plausible. Values $\mathrm{LOF}_{k,S}(\mathbf{x}) > 1$ suggest that $x$ lies in a relatively sparse region (i.e., more outlier-like), while values $<1$ indicate higher-than-average local density. For a set of counterfactuals, we report the mean LOF score across all points. This measure provides a principled way to compare the plausibility of counterfactuals generated by ElliCE and competing baselines.


For data-supported counterfactuals, ElliCE achieves plausibility comparable to baseline methods since explanations lie on the data manifold by construction. For non-data-supported counterfactuals, ElliCE performs competitively with methods not specifically designed for plausibility, outperforming TRexI and ROAR on the diabetes dataset while maintaining superior robustness guarantees. PROPLACE, which is explicitly designed for plausibility optimization, achieves better LOF scores but at the cost of reduced robustness as demonstrated in our main experimental results.

LOF metrics results are presented in Figures ~\ref{fig:lof_ds_lin}, ~\ref{fig:lof_ds_mlp}, ~\ref{fig:lof_cont_lin} and ~\ref{fig:lof_cont_mlp}.

The inherent robustness-proximity trade-off in ElliCE, explored in section~\ref{appendix:experiments}, naturally contributes to improved plausibility. By requiring counterfactuals to remain valid across multiple models in the Rashomon set, our method pushes explanations away from unstable regions near decision boundaries, often resulting in more realistic feature combinations that better align with the underlying data distribution.

\subsection{Actionability Enhancements}

In practical deployment, counterfactual explanations must respect application-specific constraints that restrict which features can be modified and how these modifications may occur. To incorporate such requirements, we extend ElliCE with actionability-aware mechanisms that ensure generated recourse recommendations remain feasible.

\subsubsection{Immutable Features and Directional Constraints}

We distinguish two primary classes of actionability constraints.  
(1) \emph{Immutable features}: attributes such as demographic characteristics or historical records that cannot be modified. These are enforced by gradient masking in continuous optimization and by pre-filtering in data-supported search.  
(2) \emph{Range-constrained features}: attributes restricted to remain within predefined feasible intervals. This general class includes monotonic one-directional changes as a special case, e.g., non-decreasing education level or non-increasing debt. Range constraints are enforced by projection operators during optimization and by pre-filtering infeasible candidates in search-based methods.

Table~\ref{tab:actionability_example} reports an illustration using the German Credit dataset. The example contrasts the closest counterfactual with an actionable alternative that respects immutability (e.g., age, foreign worker, gender). The results demonstrate the trade-off between minimal proximity and constraint satisfaction: actionable counterfactuals require larger modifications but remain valid and feasible. The example in the Table~\ref{tab:actionability_example} was generated using Continuous ElliCE method.


\subsection{Sparsity Control}\label{appendix:sparsity_control}

Sparse counterfactuals, which alter fewer features, are generally more interpretable and actionable. ElliCE incorporates possibility of sparsity control for both data-supported and continuous optimization settings.

\subsubsection{Data-Supported Sparsity Optimization}

For data-supported methods, sparsity is promoted either by criteria-based filtering or by a modified BallTree distance:
\[\nu(\mathbf{x}_1,\mathbf{x}_2) = C \cdot \|\mathbf{x}_1 - \mathbf{x}_2\|_0 + \|\mathbf{x}_1 - \mathbf{x}_2\|_1
\]
where $C > 0$ controls the relative importance of sparsity. Larger values of $C$ emphasize reducing the number of modified features, while smaller values prioritize proximity. This formulation enables efficient sublinear search while balancing sparsity and distance objectives. 

Moreover, ElliCE is method-agnostic with respect to nearest neighbor search and supports any available algorithm, including standard KDTree search with $\ell_1$ or $\ell_2$ distances.

\subsubsection{Continuous Sparsity Optimization}

For gradient-based methods, sparsity is enforced through an iterative procedure that integrates feature importance ranking, greedy selection, and coordinate masking. The procedure is summarized in Algorithm~\ref{alg:sparsity}. This procedure produces counterfactuals that are both robust and sparse by modifying only the most influential features while preventing unnecessary changes.

\begin{algorithm}[t]
\caption{Continuous Sparsity Optimization}
\label{alg:sparsity}
\begin{algorithmic}[1]
\Require original instance $\mathbf{x}$, model parameters $\bm{\theta}$, constraint set $\mathcal{C}$
\Ensure sparse robust counterfactual $\mathbf{x}_c$
\State $\mathbf{x}_c \gets \mathbf{x}$
\State Run full optimization to accumulate gradient magnitudes $g_i$ for each feature $i$
\State Rank features by decreasing $|g_i|$ to obtain list $\mathcal{L}$
\State $\mathcal{A} \gets \emptyset$
\For{feature $i \in \mathcal{L}$}
  \State $\mathcal{A} \gets \mathcal{A} \cup \{i\}$
  \State Optimize $\mathbf{x}_c$ with gradients masked outside $\mathcal{A}$ and $\mathbf{x}_c$ projected onto $\mathcal{C}$ (range and immutability constraints) during each step
  \If{robustness criterion is satisfied}
     \State \textbf{break}
  \EndIf
\EndFor
\State \Return $\mathbf{x}_c$
\end{algorithmic}
\end{algorithm}

\subsection{Multi-Class Extension}

Although our primary analysis considers binary classification, ElliCE naturally extends to multi-class settings while preserving convexity. For $K$ classes and target $y^*$, we impose robustness constraints against all competitors:
\[
\min_{\bm{\theta} \in \hat{\mathcal{R}}(\epsilon)} \; [\bm{\theta}_{y^*}^\top \mathbf{x}_c - \bm{\theta}_c^\top \mathbf{x}_c] \geq \tau, 
\quad \forall c \neq y^*,
\]
where $\tau \geq 0$ is a robustness margin. This produces $K-1$ second-order cone constraints, maintaining computational tractability. For neural architectures, we apply the ellipsoidal Rashomon approximation in embedding space, with optimization via gradient descent. All theoretical guarantees (uniqueness, stability, alignment) remain intact under this extension.


\subsection{Mixed Feature Types with Gumbel-Softmax}

Many real-world datasets contain a mixture of continuous and categorical variables, which complicates joint optimization. We adopt a unified relaxation strategy that treats the two types differently but within a single optimization framework. 

\paragraph{Continuous features.}  
Continuous variables are optimized directly in $\mathbb{R}^d$, with projection operators enforcing domain-specific constraints such as feature bounds and immutability.

\paragraph{Categorical features.}  
Each categorical variable represented by a one-hot group is relaxed into a probability vector using the Gumbel-Softmax distribution. Given a logit vector $z \in \mathbb{R}^K$ for a categorical group with $K$ possible categories, we draw Gumbel noise $g_i \sim \mathrm{Gumbel}(0,1)$ and compute
\[
y_i = \frac{\exp\!\bigl((z_i + g_i)/\tau\bigr)}{\sum_{j=1}^K \exp\!\bigl((z_j + g_j)/\tau\bigr)}, 
\quad i = 1, \dots, K,
\]
where $\tau > 0$ is a fixed temperature parameter. As $\tau \to 0$, $y$ approaches a one-hot vector, while for larger $\tau$ the distribution is smoother. In practice, we keep $\tau$ fixed during optimization, yielding a differentiable approximation that allows backpropagation through categorical choices while preserving the simplex constraint $\sum_i y_i = 1$.

\paragraph{Unified optimization.}  
The counterfactual candidate $\mathbf{x}_c$ is obtained by combining continuous variables with the relaxed categorical vectors $y$. Gradient-based updates are applied jointly to both parts, with projection steps ensuring immutability and range constraints. During optimization, we sample discrete counterfactuals by fixing continuous values and drawing one-hot categorical assignments. If at least one sampled counterfactual satisfies the robustness condition, gradient optimization is terminated. A fixed number of additional samples are then evaluated to explore whether a better robust counterfactual can be identified.

\begin{table}[t]
\centering
\footnotesize
\caption{Runtime performance. Data supported CE for MLP. Time required for method-dependent preprocessing and for computing counterfactuals on the validation set for each subfold.}
\setlength{\tabcolsep}{4pt}
\begin{tabular}{l|cccc}
\hline
\textbf{Dataset} & \textbf{ElliCE} & \textbf{ElliCE+R} & \textbf{T:Rex} & \textbf{Delta Rob} \\
\hline
\multicolumn{5}{l}{\emph{Linear}} \\
\hline
FICO & $1.521 \pm 0.019$ & --- & $5.447 \pm 0.062$ & $10.734 \pm 0.070$ \\
COMPAS & $0.816 \pm 0.016$ & --- & $3.037 \pm 0.056$ & $4.640 \pm 0.053$ \\
Australian & $0.055 \pm 0.007$ & --- & $0.288 \pm 0.010$ & $0.544 \pm 0.008$ \\
Diabetes & $0.051 \pm 0.000$ & --- & $0.266 \pm 0.004$ & $0.413 \pm 0.005$ \\
German & $0.089 \pm 0.008$ & --- & $0.433 \pm 0.008$ & $1.214 \pm 0.023$ \\
\hline
\multicolumn{5}{l}{\emph{MLP}} \\
\hline
FICO & $1.792 \pm 0.123$ & $1.283 \pm 0.076$ & $7.006 \pm 0.058$ & $242.035 \pm 1.161$ \\
COMPAS & $0.526 \pm 0.011$ & $0.443 \pm 0.014$ & $3.534 \pm 0.128$ & $360.480 \pm 6.701$ \\
Australian & $0.057 \pm 0.011$ & $0.036 \pm 0.003$ & $0.281 \pm 0.006$ & $2.783 \pm 0.032$ \\
Diabetes & $0.053 \pm 0.001$ & $0.033 \pm 0.001$ & $0.296 \pm 0.006$ & $1.922 \pm 0.032$ \\
German & $0.101 \pm 0.001$ & $0.037 \pm 0.001$ & $0.432 \pm 0.013$ & $9.905 \pm 0.068$ \\
\hline
\end{tabular}
\label{tab:runtime_dats_cm}
\end{table}

\begin{table}[t]
\centering
\footnotesize
\caption{Runtime performance. Continuous CE for MLP. Time required for method-dependent preprocessing and for computing counterfactuals on the validation set for each subfold.}
\setlength{\tabcolsep}{4pt}
\begin{tabular}{l|ccccc}
\hline
\textbf{Dataset} & \textbf{ElliCE} & \textbf{ElliCE+R} & \textbf{T:Rex} & \textbf{ROAR} & \textbf{PROPLACE} \\
\hline
\multicolumn{6}{l}{\emph{Linear}} \\
\hline
Diabetes & $0.250 \pm 0.031$ & --- & $19.476 \pm 1.116$ & $15.282 \pm 1.533$ & $0.836 \pm 0.038$ \\
Iris & $0.563 \pm 0.023$ & --- & $3.219 \pm 0.354$ & $15.603 \pm 0.242$ & $0.639 \pm 0.008$ \\
Parkinsons & $12.339 \pm 4.898$ & --- & $81.553 \pm 13.906$ & $73.409 \pm 11.057$ & $11.157 \pm 0.256$ \\
Wine Quality & $1.407 \pm 0.061$ & --- & $54.151 \pm 2.269$ & $58.821 \pm 1.432$ & $4.876 \pm 0.158$ \\
Banknote & $0.493 \pm 0.010$ & --- & $6.579 \pm 9.541$ & $19.914 \pm 11.567$ & $1.088 \pm 0.017$ \\
\hline
\multicolumn{6}{l}{\emph{MLP}} \\
\hline
Diabetes & $4.861 \pm 1.539$ & $1.122 \pm 0.578$ & $56.104 \pm 12.937$ & $9.448 \pm 0.579$ & $78.323 \pm 85.695$ \\
Iris & $0.562 \pm 0.025$ & $0.476 \pm 0.008$ & $20.984 \pm 3.801$ & $17.134 \pm 1.816$ & $3.703 \pm 0.383$ \\
Parkinsons & $112.323 \pm 6.253$ & $1.409 \pm 0.037$ & $25.771 \pm 10.461$ & $104.660 \pm 14.314$ & --- \\
Wine Quality & $26.961 \pm 4.745$ & $17.199 \pm 4.981$ & $160.531 \pm 37.765$ & $39.688 \pm 5.101$ & $869.119 \pm 794.885$ \\
Banknote & $0.631 \pm 0.030$ & $0.526 \pm 0.062$ & $14.770 \pm 3.422$ & $22.985 \pm 3.777$ & $10.622 \pm 2.580$ \\
\hline
\end{tabular}
\label{tab:runtime_continuous_cm}
\end{table}

\begin{table}[h]
\small
\centering
\setlength{\tabcolsep}{3pt}
\caption{Performance comparison of Data-Supported counterfactual methods at $\varepsilon_\text{target}=0.1 \hat{L}_{train}(f_{\text{baseline}})$.}
\label{tab:performance_comparison}
\footnotesize
\begin{tabular}{l|l|cc|cc|cc}
\hline
\multirow{3}{*}{\textbf{Dataset}} & \multirow{3}{*}{\textbf{Method}} & \multicolumn{6}{c}{\textbf{Evaluation Metrics}} \\[2pt]
\cline{3-8}
& & \multicolumn{2}{c|}{\textbf{Retrain}} & \multicolumn{2}{c|}{\textbf{Dropout Rashomon}} & \multicolumn{2}{c}{\textbf{AWP}} \\[2pt]
& & \textbf{R$\uparrow$} & \textbf{L2$\downarrow$} & \textbf{R$\uparrow$} & \textbf{L2$\downarrow$} & \textbf{R$\uparrow$} & \textbf{L2$\downarrow$} \\[2pt]
\hline
\multicolumn{8}{c}{\textbf{Linear models}} \\[2pt]
\hline
\multirow{3}{*}{Australian} 
& ElliCE   & --- & --- & 0.981 (0.03) & 2.46 (0.12) & \textbf{0.991 (0.01)} & 2.57 (0.15) \\[2pt]
& DeltaRob & --- & --- & 0.722 (0.28) & 2.64 (0.32) & 0.531 (0.30) & 2.51 (0.25) \\[2pt]
& T:Rex    & --- & --- & \textbf{0.988 (0.01)} & 2.62 (0.16) & 0.983 (0.01) & 2.73 (0.18) \\[2pt]
\hline
\multirow{3}{*}{COMPAS} 
& ElliCE   & --- & --- & \textbf{0.501 (0.50)} & 2.35 (0.91) & \textbf{0.500 (0.50)} & 2.13 (1.12) \\[2pt]
& DeltaRob & --- & --- & 0.001 (0.00) & 1.24 (0.11) & 0.000 (0.00) & 1.09 (0.04) \\[2pt]
& T:Rex    & --- & --- & 0.000 (0.00) & 0.77 (0.02) & 0.000 (0.00) & 0.77 (0.02) \\[2pt]
\hline
\multirow{3}{*}{Diabetes} 
& ElliCE   & --- & --- & \textbf{0.964 (0.06)} & 2.65 (0.06) & \textbf{0.989 (0.01)} & 3.10 (0.22) \\[2pt]
& DeltaRob & --- & --- & 0.670 (0.33) & 2.97 (0.44) & 0.033 (0.02) & 2.39 (0.20) \\[2pt]
& T:Rex    & --- & --- & 0.946 (0.04) & 3.04 (0.18) & 0.685 (0.27) & 3.43 (0.43) \\[2pt]
\hline
\multirow{3}{*}{FICO} 
& ElliCE   & --- & --- & \textbf{1.000 (0.00)} & 4.93 (0.17) & \textbf{1.000 (0.00)} & 5.46 (0.14) \\[2pt]
& DeltaRob & --- & --- & 0.068 (0.03) & 3.61 (0.19) & 0.008 (0.00) & 3.74 (0.06) \\[2pt]
& T:Rex    & --- & --- & 0.035 (0.02) & 3.02 (0.15) & 0.003 (0.00) & 2.89 (0.25) \\[2pt]
\hline
\multirow{3}{*}{German} 
& ElliCE   & --- & --- & \textbf{1.000 (0.00)} & 3.99 (0.20) & 0.997 (0.01) & 3.85 (0.11) \\[2pt]
& DeltaRob & --- & --- & 0.983 (0.02) & 4.34 (0.57) & \textbf{1.000 (0.00)} & 4.01 (0.14) \\[2pt]
& T:Rex    & --- & --- & 0.969 (0.04) & 3.91 (0.16) & \textbf{1.000 (0.00)} & 3.92 (0.16) \\[2pt]
\hline
\multicolumn{8}{c}{\textbf{Multi-layer perceptron}} \\[2pt]
\hline
\multirow{4}{*}{Australian} 
& ElliCE   & \textbf{1.000 (0.00)} & 2.30 (0.14) & \textbf{1.000 (0.00)} & 2.61 (0.14) & 0.954 (0.04) & 2.51 (0.15) \\[2pt]
& ElliCE+R & 0.996 (0.01) & 2.24 (0.13) & 0.993 (0.01) & 2.70 (0.23) & 0.954 (0.05) & 2.53 (0.13) \\[2pt]
& DeltaRob & \textbf{1.000 (0.00)} & 2.28 (0.15) & \textbf{1.000 (0.00)} & 3.16 (0.45) & \textbf{1.000 (0.00)} & 3.29 (0.45) \\[2pt]
& T:Rex    & 0.996 (0.01) & 2.25 (0.09) & 0.992 (0.01) & 2.69 (0.13) & 0.942 (0.08) & 2.82 (0.40) \\[2pt]
\hline
\multirow{4}{*}{COMPAS} 
& ElliCE   & \textbf{1.000 (0.00)} & 1.12 (0.12) & 0.814 (0.32) & 1.65 (0.17) & 0.430 (0.35) & 1.79 (0.15) \\[2pt]
& ElliCE+R & 0.999 (0.00) & 1.08 (0.19) & \textbf{1.000 (0.00)} & 1.86 (0.04) & \textbf{0.998 (0.00)} & 1.99 (0.05) \\[2pt]
& DeltaRob & \textbf{1.000 (0.00)} & 1.20 (0.10) & 0.999 (0.00) & 1.92 (0.16) & 0.775 (0.39) & 2.31 (0.42) \\[2pt]
& T:Rex    & 0.422 (0.34) & 0.72 (0.06) & 0.000 (0.00) & 0.63 (0.02) & 0.000 (0.00) & 0.60 (0.02) \\[2pt]
\hline
\multirow{4}{*}{Diabetes} 
& ElliCE   & \textbf{1.000 (0.00)} & 2.27 (0.13) & 0.993 (0.01) & 2.78 (0.26) & \textbf{1.000 (0.00)} & 3.11 (0.18) \\[2pt]
& ElliCE+R & 0.992 (0.01) & 2.24 (0.17) & \textbf{0.996 (0.01)} & 2.81 (0.21) & 0.980 (0.03) & 3.01 (0.34) \\[2pt]
& DeltaRob & 0.947 (0.06) & 2.39 (0.13) & 0.549 (0.23) & 2.50 (0.15) & 0.302 (0.17) & 2.50 (0.03) \\[2pt]
& T:Rex    & 0.988 (0.02) & 2.25 (0.16) & 0.962 (0.06) & 3.06 (0.34) & 0.869 (0.08) & 3.09 (0.26) \\[2pt]
\hline
\multirow{4}{*}{FICO} 
& ElliCE   & \textbf{1.000 (0.00)} & 3.53 (0.17) & \textbf{1.000 (0.00)} & 4.91 (0.22) & \textbf{1.000 (0.00)} & 5.06 (0.29) \\[2pt]
& ElliCE+R & \textbf{1.000 (0.00)} & 3.72 (0.23) & \textbf{1.000 (0.00)} & 4.75 (0.17) & 0.993 (0.01) & 5.58 (0.62) \\[2pt]
& DeltaRob & \textbf{1.000 (0.00)} & 4.00 (0.10) & \textbf{1.000 (0.00)} & 5.67 (0.58) & 0.957 (0.07) & 5.70 (0.72) \\[2pt]
& T:Rex    & 0.830 (0.08) & 3.12 (0.07) & 0.006 (0.00) & 3.07 (0.11) & 0.001 (0.00) & 2.77 (0.19) \\[2pt]
\hline
\multirow{4}{*}{German} 
& ElliCE   & \textbf{1.000 (0.00)} & 3.48 (0.10) & \textbf{0.996 (0.01)} & 4.32 (0.31) & \textbf{1.000 (0.00)} & 4.00 (0.24) \\[2pt]
& ElliCE+R & 0.990 (0.02) & 3.44 (0.08) & \textbf{0.996 (0.01)} & 4.00 (0.06) & 0.947 (0.06) & 3.94 (0.17) \\[2pt]
& DeltaRob & 0.982 (0.01) & 3.45 (0.06) & 0.988 (0.02) & 4.00 (0.15) & \textbf{1.000 (0.00)} & 3.99 (0.22) \\[2pt]
& T:Rex    & 0.989 (0.01) & 3.47 (0.04) & 0.967 (0.02) & 4.03 (0.20) & 0.989 (0.01) & 4.23 (0.24) \\[2pt]
\hline
\end{tabular}
\end{table}

\begin{table}[h]
\small
\centering
\setlength{\tabcolsep}{3pt}
\caption{Performance comparison of Continuous counterfactual methods at $\varepsilon_\text{target}=0.1 \hat{L}_{train}(f_{\text{baseline}})$ (continuous datasets), using L2 distance.}
\label{tab:performance_comparison_cnt_full} 
\footnotesize
\begin{tabular}{l|l|cc|cc|cc}
\hline
\multirow{3}{*}{\textbf{Dataset}} & \multirow{3}{*}{\textbf{Method}} & \multicolumn{6}{c}{\textbf{Evaluation Metrics}} \\[2pt]
\cline{3-8}
& & \multicolumn{2}{c|}{\textbf{Retrain}} & \multicolumn{2}{c|}{\textbf{Dropout Rashomon}} & \multicolumn{2}{c}{\textbf{AWP}} \\[2pt]
& & \textbf{R$\uparrow$} & \textbf{L2$\downarrow$} & \textbf{R$\uparrow$} & \textbf{L2$\downarrow$} & \textbf{R$\uparrow$} & \textbf{L2$\downarrow$} \\[2pt]
\hline
\multicolumn{8}{c}{\textbf{Linear models}} \\[2pt]
\hline
\multirow{4}{*}{Banknote} 
& ElliCE   & --- & --- & \textbf{0.623} (0.26) & 1.470 (0.07) & \textbf{0.993 (0.00)} & 1.259 (0.04) \\[2pt]
& PROPLACE & --- & --- & 0.069 (0.03) & 1.329 (0.07) & 0.668 (0.14) & 1.348 (0.07) \\[2pt]
& ROAR     & --- & --- & 0.002 (0.00) & 0.890 (0.13) & 0.399 (0.11) & 0.890 (0.13) \\[2pt]
& T:Rex    & --- & --- & 0.132 (0.08) & 1.163 (0.05) & 0.814 (0.05) & 1.163 (0.05) \\[2pt]
\hline
\multirow{4}{*}{Diabetes} 
& ElliCE   & --- & --- & \textbf{0.996 (0.01)} & 2.669 (0.13) & \textbf{0.996 (0.01)} & 2.914 (0.17) \\[2pt]
& PROPLACE & --- & --- & 0.499 (0.14) & 2.242 (0.15) & 0.118 (0.10) & 2.357 (0.11) \\[2pt]
& ROAR     & --- & --- & 0.970 (0.02) & 2.962 (0.15) & 0.973 (0.02) & 3.010 (0.12) \\[2pt]
& T:Rex    & --- & --- & 0.974 (0.02) & 4.188 (0.46) & 0.808 (0.10) & 4.188 (0.46) \\[2pt]
\hline
\multirow{4}{*}{Iris} 
& ElliCE   & --- & --- & \textbf{1.000 (0.00)} & 3.265 (0.15) & \textbf{1.000 (0.00)} & 3.213 (0.10) \\[2pt]
& PROPLACE & --- & --- & 0.745 (0.27) & 2.929 (0.08) & 0.633 (0.23) & 2.929 (0.08) \\[2pt]
& ROAR     & --- & --- & 0.105 (0.11) & 2.578 (0.13) & 0.005 (0.01) & 2.492 (0.19) \\[2pt]
& T:Rex    & --- & --- & 0.173 (0.11) & 2.701 (0.14) & 0.060 (0.06) & 2.552 (0.20) \\[2pt]
\hline
\multirow{4}{*}{Parkinsons} 
& ElliCE   & --- & --- & \textbf{0.997 (0.00)} & 4.236 (1.03) & 0.997 (0.00) & 3.621 (1.22) \\[2pt]
& PROPLACE & --- & --- & 0.269 (0.09) & 3.062 (0.11) & 0.046 (0.01) & 2.414 (0.04) \\[2pt]
& ROAR     & --- & --- & 0.923 (0.13) & 2.419 (0.23) & \textbf{1.000 (0.00)} & 2.080 (0.15) \\[2pt]
& T:Rex    & --- & --- & 0.919 (0.10) & 2.148 (0.18) & 0.986 (0.01) & 2.148 (0.18) \\[2pt]
\hline
\multirow{4}{*}{Wine Quality} 
& ElliCE   & --- & --- & \textbf{0.998 (0.00)} & 3.300 (0.20) & \textbf{0.998 (0.00)} & 3.306 (0.07) \\[2pt]
& PROPLACE & --- & --- & 0.470 (0.11) & 3.037 (0.29) & 0.154 (0.09) & 3.030 (0.29) \\[2pt]
& ROAR     & --- & --- & 0.973 (0.02) & 3.502 (0.13) & 0.994 (0.01) & 3.431 (0.31) \\[2pt]
& T:Rex    & --- & --- & 0.961 (0.02) & 3.470 (0.19) & 0.978 (0.01) & 3.470 (0.19) \\[2pt]
\hline
\multicolumn{8}{c}{\textbf{Multi-layer perceptron}} \\[2pt]
\hline
\multirow{5}{*}{Banknote} 
& ElliCE   & \textbf{1.000 (0.00)} & 1.125 (0.06) & 0.995 (0.01) & 1.644 (0.13) & 0.990 (0.01) & 1.475 (0.09) \\[2pt]
& ElliCE+R & 0.995 (0.01) & 1.116 (0.08) & 0.993 (0.01) & 1.496 (0.09) & 0.993 (0.01) & 1.365 (0.06) \\[2pt]
& PROPLACE & \textbf{1.000 (0.00)} & 1.339 (0.05) & 0.619 (0.26) & 1.386 (0.02) & 0.692 (0.30) & 1.386 (0.02) \\[2pt]
& ROAR     & \textbf{1.000 (0.00)} & 2.323 (0.03) & \textbf{1.000 (0.00)} & 2.323 (0.03) & \textbf{1.000 (0.00)} & 2.323 (0.03) \\[2pt]
& T:Rex    & 0.993 (0.01) & 1.188 (0.11) & 0.988 (0.02) & 1.563 (0.12) & \textbf{1.000 (0.00)} & 1.563 (0.12) \\[2pt]
\hline
\multirow{5}{*}{Diabetes} 
& ElliCE   & 0.977 (0.01) & 2.146 (0.39) & \textbf{0.985 (0.02)} & 3.050 (0.34) & 0.980 (0.02) & 3.221 (0.40) \\[2pt]
& ElliCE+R & \textbf{0.978 (0.04)} & 2.133 (0.33) & 0.965 (0.02) & 2.864 (0.46) & \textbf{0.985 (0.00)} & 3.414 (0.63) \\[2pt]
& PROPLACE & 0.482 (0.48) & 2.007 (0.05) & 0.187 (0.28) & 2.007 (0.05) & 0.079 (0.19) & 2.007 (0.05) \\[2pt]
& ROAR     & 0.864 (0.11) & 1.858 (0.24) & 0.400 (0.28) & 1.858 (0.24) & 0.308 (0.26) & 1.858 (0.24) \\[2pt]
& T:Rex    & 0.944 (0.03) & 2.471 (0.86) & 0.899 (0.08) & 4.181 (0.36) & 0.940 (0.04) & 4.181 (0.36) \\[2pt]
\hline
\multirow{5}{*}{Iris} 
& ElliCE   & 0.995 (0.01) & 2.492 (0.17) & \textbf{1.000 (0.00)} & 2.936 (0.12) & 0.998 (0.00) & 2.333 (0.10) \\[2pt]
& ElliCE+R & 0.975 (0.03) & 2.401 (0.18) & 0.970 (0.03) & 2.916 (0.21) & 0.995 (0.01) & 2.193 (0.10) \\[2pt]
& PROPLACE & \textbf{1.000 (0.00)} & 2.929 (0.08) & \textbf{1.000 (0.00)} & 2.933 (0.08) & \textbf{1.000 (0.00)} & 2.929 (0.08) \\[2pt]
& ROAR     & 0.930 (0.03) & 2.851 (0.09) & 0.758 (0.08) & 2.851 (0.09) & 0.930 (0.03) & 2.851 (0.09) \\[2pt]
& T:Rex    & \textbf{1.000 (0.00)} & 2.838 (0.34) & 0.922 (0.08) & 2.973 (0.15) & 0.992 (0.01) & 2.201 (0.11) \\[2pt]
\hline
\multirow{5}{*}{Parkinsons} 
& ElliCE   & \textbf{0.885 (0.02)} & 2.560 (0.32) & \textbf{0.999 (0.00)} & 1.226 (0.18) & \textbf{0.999 (0.00)} & 1.403 (0.07) \\[2pt]
& ElliCE+R & 0.610 (0.07) & 1.170 (0.07) & 0.996 (0.00) & 1.112 (0.07) & 0.993 (0.00) & 1.100 (0.10) \\[2pt]
& PROPLACE & --- & --- & --- & --- & --- & --- \\[2pt]
& ROAR     & 0.593 (0.15) & 1.902 (0.21) & 0.853 (0.08) & 1.902 (0.21) & 0.854 (0.08) & 1.902 (0.21) \\[2pt]
& T:Rex    & 0.577 (0.06) & 1.062 (0.04) & 0.990 (0.00) & 1.062 (0.04) & 0.987 (0.00) & 1.062 (0.04) \\[2pt]
\hline
\multirow{5}{*}{Wine Quality} 
& ElliCE   & 0.964 (0.02) & 3.695 (0.39) & 0.978 (0.01) & 3.457 (0.38) & 0.939 (0.04) & 3.653 (0.42) \\[2pt]
& ElliCE+R & \textbf{0.969 (0.03)} & 4.185 (0.70) & \textbf{0.978 (0.02)} & 2.848 (0.30) & \textbf{0.997 (0.00)} & 4.266 (0.67) \\[2pt]
& PROPLACE & 0.159 (0.16) & 1.796 (0.05) & 0.068 (0.07) & 1.796 (0.05) & 0.020 (0.03) & 1.796 (0.05) \\[2pt]
& ROAR     & 0.918 (0.04) & 2.674 (0.06) & 0.867 (0.06) & 2.674 (0.06) & 0.755 (0.10) & 2.674 (0.06) \\[2pt]
& T:Rex    & 0.877 (0.05) & 3.252 (0.32) & 0.931 (0.03) & 3.252 (0.32) & 0.900 (0.04) & 3.252 (0.32) \\[2pt]
\hline
\end{tabular}
\end{table}

\end{document}